\documentclass{article}

\usepackage[main,final,nonatbib]{neurips_2026}

\usepackage[utf8]{inputenc}
\usepackage[T1]{fontenc}
\usepackage[colorlinks=true,linkcolor={blue!60!black},citecolor={blue!60!black},urlcolor={blue!60!black}]{hyperref}
\usepackage{url}
\usepackage{booktabs}
\usepackage{amsfonts}
\usepackage{nicefrac}
\usepackage{microtype}
\usepackage{xcolor}
\usepackage{amssymb,amsthm}
\usepackage{amsmath}
\usepackage{enumitem}
\usepackage{caption}
\usepackage{subcaption}
\usepackage{tabularx}
\usepackage{array}
\usepackage{mathtools}
\usepackage{wrapfig}
\usepackage{thmtools}
\usepackage{thm-restate}
\usepackage{algorithm}
\usepackage{algorithmic}
\usepackage{graphicx}
\usepackage{placeins}
\usepackage[capitalize,noabbrev]{cleveref}
\usepackage{etoc}

\usepackage[
backref=true,
natbib=true,
style=numeric-comp,
giveninits=false, 
maxbibnames = 99, 
maxalphanames=8,
sorting=none
]{biblatex}
\theoremstyle{plain}
\newtheorem{theorem}{Theorem}[section]
\newtheorem{proposition}[theorem]{Proposition}
\newtheorem{lemma}[theorem]{Lemma}

\theoremstyle{definition}
\newtheorem{definition}[theorem]{Definition}
\newtheorem{assumption}[theorem]{Assumption}
\theoremstyle{remark}
\newtheorem{remark}[theorem]{Remark}

\crefname{assumption}{Assump.}{Assumps.}
\Crefname{assumption}{Assump.}{Assumps.}
\crefname{theorem}{Thm.}{Thms.}
\Crefname{theorem}{Thm.}{Thms.}
\crefname{section}{Sec.}{Secs.}
\Crefname{section}{Sec.}{Secs.}
\crefname{definition}{Def.}{Defs.}
\Crefname{definition}{Def.}{Defs.}
\crefname{proposition}{Prop.}{Props.}
\Crefname{proposition}{Prop.}{Props.}
\crefname{lemma}{Lem.}{Lems.}
\Crefname{lemma}{Lem.}{Lems.}
\crefname{remark}{Rem.}{Rems.}
\Crefname{remark}{Rem.}{Rems.}
\crefname{figure}{Fig.}{Figs.}
\Crefname{figure}{Fig.}{Figs.}
\crefname{table}{Tab.}{Tabs.}
\Crefname{table}{Tab.}{Tabs.}
\crefname{equation}{Eq.}{Eqs.}

\usepackage{listings}
\usepackage{xcolor}

\definecolor{codebg}{HTML}{F7F7F7}
\definecolor{codekw}{HTML}{1F4E79}
\definecolor{codestr}{HTML}{8B0000}
\definecolor{codecmt}{HTML}{6A737D}

\lstdefinestyle{pythonbox}{
  language=Python,
  basicstyle=\ttfamily\small,
  keywordstyle=\color{codekw}\bfseries,
  stringstyle=\color{codestr},
  commentstyle=\color{codecmt}\itshape,
  backgroundcolor=\color{codebg},
  frame=single,
  rulecolor=\color{black!30},
  framesep=4pt,
  showstringspaces=false,
  breaklines=true,
  aboveskip=8pt,
  belowskip=8pt,
}

\newcommand{\RR}{\mathbb{R}}

\def\eqref#1{equation~\ref{#1}}

\def\1{\bm{1}}

\DeclareMathAlphabet{\mathsfit}{\encodingdefault}{\sfdefault}{m}{sl}
\SetMathAlphabet{\mathsfit}{bold}{\encodingdefault}{\sfdefault}{bx}{n}

\def\gF{{\mathcal{F}}}

\def\gN{{\mathcal{N}}}
\def\gO{{\mathcal{O}}}

\newcommand{\E}{\mathbb{E}}

\newcommand{\R}{\mathbb{R}}

\renewcommand{\epsilon}{\varepsilon}

\DeclareMathOperator{\msign}{\text{msign}}

\DeclareMathOperator{\op}{op}

\DeclareMathOperator{\clip}{\text{clip}}

\DeclareMathOperator*{\textmin}{\textnormal{min}}
\renewcommand{\min}{\textmin}

\global\long\def\norm#1{\|#1\|}%
\global\long\def\inner#1#2{\langle#1,#2\rangle}%

\addtocontents{toc}{\protect\setcounter{tocdepth}{-1}}

\title{Can Entry-Wise Clipping Give Spectral Control of Stochastic Gradients?}

\author{%
  Zitao Song$^\clubsuit$ \quad
  Cedar Site Bai$^\clubsuit$ \quad
  Zhe Zhang$^\spadesuit$ \quad
  Brian Bullins$^\clubsuit$ \quad
  David F. Gleich$^\clubsuit$ \\[2pt]
  $^\clubsuit$Department of Computer Science, Purdue University \\
  $^\spadesuit$School of Industrial Engineering, Purdue University \\[2pt]
  \texttt{\{song903, bai123, zhan5111, bbullins, dgleich\}@purdue.edu}
}

\begin{document}

\maketitle
\begin{abstract}

Training instabilities such as loss spikes are frequently the result of stochastic gradient noise. Because of rare expressions in language training data, and multiple layer composition, the noise impact is heavy-tailed and survives mini-batch averaging. Existing remedies trade off structure against cost: vector-norm clipping ignores the matrix structure of weight updates, while spectral normalization (e.g., Muon \citep{jordan2024muon}) respects it at additional cost. We show that this trade-off can be balanced. Real gradient noise appears to be similar to entry-wise heavy-tailed contamination, and a first-order perturbation analysis reveals a \emph{localization} property of such noise, under which a simple entry-wise method achieves spectral control.  Exploiting this, we derive a tractable surrogate for the Bayes-optimal entry-wise estimator under a Gaussian signal prior. We establish $\gO(\epsilon^{-4})$ convergence guarantee under Cauchy-contaminated gradient noise. Empirically, we find that smooth shrinkage improves Adam on NanoGPT pretraining, saving ${\sim}7\%$ of training tokens. We further find that applying the entry-wise clipping before spectral normalization yields a ${\sim}2\%$ token saving on top of Muon.
\end{abstract}

\section{Introduction}
\label{sec:intro}

%
%

A growing body of evidence shows that stochastic gradient noise is markedly heavy-tailed, with extreme entries occurring far more often than any Gaussian model would predict \citep{simsekli2019tail,gurbuzbalaban2021heavy}. In Large Language Model (LLM) pretraining \citep{grattafiori2024llama,team2025kimi}, these heavy tails are further amplified by model depth and scale, producing loss spikes and training instabilities \citep{wortsman2023small,huang2025spam}. To address this, global \citep{pascanu2013difficulty} and coordinate-wise \citep{zhang2020adaptive} gradient clipping truncate updates under vector-norm constraints, but neither controls the spectrum of matrix-valued weight updates.

Spectral methods \citep{carlson2015stochastic,jordan2024muon} offer a powerful alternative for constraining matrix-valued weight norms \citep{miyato2018spectral} and have shown success in LLM pretraining \citep{wen2025fantastic}. The Muon optimizer \citep{jordan2024muon,bernstein2024old} can be viewed as uniformly setting all singular values to one, and SPECTRA \citep{jiang2026enhancing} generalizes this to spectral clipping under a random low-rank perturbation model, matching the performance of spectral normalization. However, spectral operations are computationally expensive, typically requiring SVD or Newton-Schulz iterations. This raises a natural question:
\begin{quote}
\itshape
Can entry-wise clipping function as a cheap surrogate for spectral clipping in controlling the spectrum of matrix updates?
\end{quote}

We show that the answer is yes under a heavy-tailed contamination model, but \emph{not} under Gaussian or random low-rank noise with bounded second moment. The distinguishing property is \emph{localization}: in real matrix-valued gradient noise, a small number of entries dominate the spectrum, whereas in Gaussian or random low-rank perturbations, no single entry can significantly perturb the spectrum. We formalize localization through the \emph{normalized localization ratio} $\widehat{R}(E)$ of a noise matrix $E$, a metric derived from a first-order perturbation analysis of the top singular value. This metric identifies two regimes in which entry-wise magnitudes directly affect the perturbed spectrum, making entry-wise clipping a viable instrument for spectral control. 

In particular, we observe that real stochastic noise from a transformer layer during large language model pretraining  (Panels~(a) and~(e) in \cref{fig:noise_models}) displays this localization phenomenon and resembles entry-wise heavy-tailed contamination (Panels~(d) and~(h)). Exploiting the entry-wise heavy-tailed structure, we derive \emph{smooth shrinkage}, $\varphi(x) = xe^{-|x|/\tau}$, as a tractable surrogate for the Bayes-optimal entry-wise estimator that minimizes the first-order spectral perturbation under a Gaussian prior on the signal.

\textbf{Contributions.} Our contributions are threefold:

\emph{Localization metric and noise model.} Through a first-order perturbation analysis of the top singular value (\Cref{thm:taylor}), we introduce the normalized localization ratio $\widehat{R}(E)$ (\Cref{def:localization}) that quantifies how the largest entries of a noise matrix $E$ align with the perturbation direction of the top singular value. We identify two regimes --- $\widehat{R}(E) \ll 1$ and $\widehat{R}(E) \gg 1$ --- in which entry-wise quantities can control the spectral perturbation. These two regimes match the empirical stochastic gradient noise of training large language models, at the early and late stages of training, respectively. We further verify that a heavy-tail contamination model (\Cref{def:contamination}), which reproduces the empirical heavy-tailed structure of gradient noise, falls into the $\widehat{R}(E) \gg 1$ localized regime.

\emph{Smooth shrinkage clipping.} Under a Gaussian signal prior, we derive smooth shrinkage as a tractable surrogate for the Bayes-optimal entry-wise estimator that minimizes spectral perturbation in the localized regime (\Cref{thm:main}). Empirically, smooth shrinkage serves as an efficient substitute for spectral clipping: applied directly to the optimizer update (\emph{post-clipping}, \Cref{sec:prelim_clipping_stages}), it improves Adam by saving ${\sim}8\%$ of training tokens; placed before spectral normalization (\emph{pre-clipping}, \Cref{sec:prelim_clipping_stages}), it also recovers subspace corrupted by localized spikes and improves Muon by saving ${\sim}2\%$ of tokens.

\emph{Convergence theory.} Under a known Cauchy-contaminated noise model, we establish convergence guarantees for stochastic gradient methods that apply either smooth shrinkage or hard coordinate-wise clipping at each of the two clipping stages. In both cases, we prove an $\gO(\epsilon^{-4})$ complexity bound for finding an $\epsilon$-stationary point, where the heavy-tail parameter enters only the complexity constants and not the exponent (\Cref{thm:post-hard-main,thm:post-smooth-main,thm:pre-hard-main,thm:pre-smooth-main}).

\textbf{Organization.}
\Cref{sec:prelim} reviews preliminaries and related work on clipping methods.
\Cref{sec:noise_model} introduces the entry-wise contamination model and analyzes its localization phenomenon.
\Cref{sec:operator} exploits this localization to derive a new clipping operator, whose theoretical and empirical properties are established in \Cref{sec:convergence,sec:experiments}.
Additional related work appears in \Cref{app:relate_work}.

\begin{figure}
    \centering
    \includegraphics[width=0.98\linewidth]{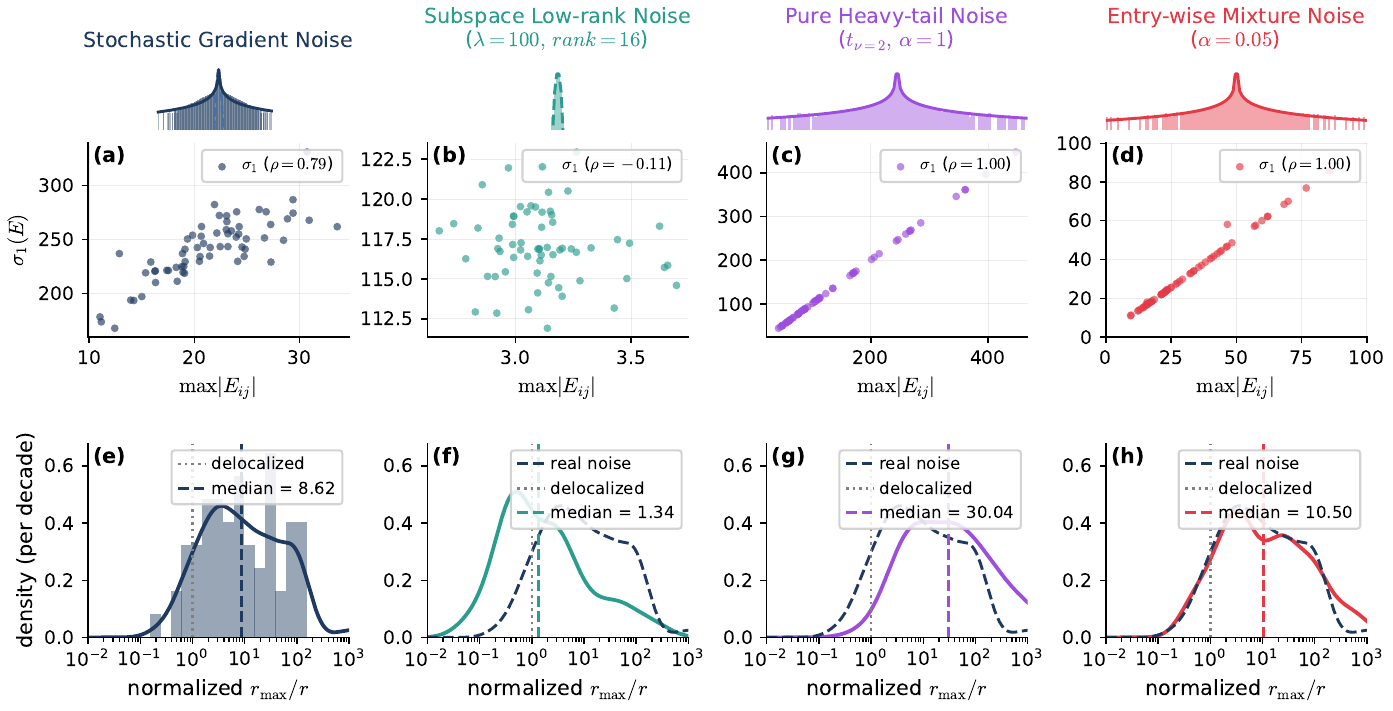}
\caption{%
\textbf{Entry-wise sparse heavy-tailed noise reproduces the spectral spikes observed in real stochastic gradients.}
The columns correspond to four noise models $E$, and the rows correspond to two diagnostics.
\textit{Real (a, e):} minibatch gradient noise from a GPT-2 layer (\texttt{blocks.10.attn.qkv\_w.v}).
\textit{Subspace low-rank (b, f):} $E = \lambda \sum_{r=1}^{K} u_r v_r^T$ with $\lambda{=}100$, $K{=}16$, and $u_r, v_r$ drawn uniformly from the unit sphere.
\textit{Pure heavy-tailed (c, g):} entry-wise noise from~\Cref{def:contamination} with $\alpha{=}1$ and $H \sim t_2$.
\textit{Entry-wise mixture (d, h):} entry-wise noise from~\Cref{def:contamination} with $\alpha{=}0.05$ and the same heavy-tailed component $H$ as the previous column.
All statistics are computed over $64$ independent realizations.
\textit{Top strip:} marginal distribution of entries, normalized by each column's empirical standard deviation.
\textit{Spectral--max scatter (a--d):} $\max_{i,j}|E_{ij}|$ versus $\sigma_1(E)$, with Spearman $\rho$ reported in each legend. A high $\rho$ means the top singular value is driven by the largest entry.
\textit{Localization ratio (e--h):} distribution of $\widehat{R} = R / \mathrm{median}[R_{\text{Gauss}}]$ from \cref{def:localization}; $\widehat{R} \approx 1$ indicates $E$ is indistinguishable from i.i.d.\ Gaussian noise, while $\widehat{R} \gg 1$ indicates that a single entry of $E$ carries a significant share of the leading-order spectral perturbation.
}
    \label{fig:noise_models}
\end{figure}

\section{Preliminaries}
\label{sec:prelim}

\textbf{Notation.} For $A \in \RR^{m \times n}$, $\sigma_i(A)$ denotes the $i$-th largest singular value with corresponding left/right singular vectors $u_i(A), v_i(A)$. We write $\|A\|_{\op}, \|A\|_F, \|A\|_*$ for the spectral, Frobenius, and nuclear norms, $\|A\|_{\infty} := \max_{ij}|A_{ij}|$ for the entry-wise max, and $\inner{A}{B} := \sum_{ij} A_{ij} B_{ij}$ for the Frobenius inner product. We write SVD as $A = U \Sigma V^T$ and $\msign(A) := U V^T$.

\subsection{Clipping Methods for Gradient Updates}
\label{sec:prelim_clipping_methods}
For a matrix $G \in \R^{m \times n}$ and a threshold $c > 0$, we call a nonlinear map $\text{clip}_c(\cdot) : \R^{m \times n} \to \R^{m \times n}$ a \emph{clipping map} if it controls the magnitude of gradient $G$ below $c$ under some norm. We briefly review three commonly used clipping maps.

\noindent\textbf{Coordinate-wise clipping} applies a scalar map independently to each entry of $G$:
\begin{equation}
\label{eq:coor_clip}
    \text{clip}_c^{\text{cw}}(G)_{ij} = \min\!\left(1,\, c/|G_{ij}|\right) G_{ij}.
\end{equation}
This guarantees $\|\text{clip}_c^{\text{cw}}(G)\|_\infty \le c$, but absent additional structural assumptions on $G$, it offers no direct control over the spectral norm $\|\text{clip}_c^{\text{cw}}(G)\|_{\text{op}}$.

\noindent\textbf{Global clipping} \citep{pascanu2013difficulty} rescales the entire matrix whenever its Frobenius norm exceeds the threshold:
\begin{equation}
    \text{clip}_c^{\text{g}}(G) = \min\!\left(1,\, c/\|G\|_F\right) G.
\end{equation}
This ensures $\|\text{clip}_c^{\text{g}}(G)\|_F \le c$, and hence $\|\text{clip}_c^{\text{g}}(G)\|_2 \le c$, but it scales all entries uniformly, suppressing informative signal components together with noise.

\noindent\textbf{Exact spectral clipping} truncates each singular value at the threshold $c$ given the SVD $G = U \Sigma V^T$:
\begin{equation}
    \text{clip}_c^{\text{sp}}(G) = \sum_i \min(\sigma_i, c)\, u_i v_i^{T}.
\end{equation}
This yields exact spectral-norm control, $\|\text{clip}_c^{\text{sp}}(G)\|_2 \le c$, while preserving the singular-vector structure of $G$. However, computing the full SVD costs $O(mn \min(m,n))$. Recent work \citep{newhouse2025training,jiang2026enhancing} avoids an explicit SVD in spectral clipping via GPU-friendly \emph{Newton--Schulz iterations} in BFloat16, but each iteration still incurs $O(mn \min(m,n))$ cost.

\subsection{Stages for Gradient Clipping}
\label{sec:prelim_clipping_stages}
Given iterates $X_k \in \R^{m \times n}$ and a clipping map $\text{clip}_c(\cdot)$ from \cref{sec:prelim_clipping_methods}, we categorize clipping methods in stochastic optimization into two groups based on the stage at which they are applied.

\paragraph{Post-clipping.} We call $\text{clip}_c(\cdot)$ a \emph{post-clipping} step when it is applied directly to a weight update $U_t$ computed by a specific optimizer using the stochastic gradient descent, yielding
\begin{equation}
\label{eq:post_clipping}
    X_{k+1} = X_{k} - \text{clip}_c(U_k).
\end{equation}
When $U_k$ is the stochastic gradient itself, this recovers Clipped SGD \citep{zhang2020adaptive}; when $U_k$ is the Adam update, it recovers Post-Spectral Clipping \citep{jiang2026enhancing}. In the post-clipping regime, Clipped SGD and its variants are known to converge both in expectation \citep{zhang2020adaptive} and with high probability \citep{cutkosky2021high,nguyen2023improved}, even when the gradient variance is unbounded. Without post-clipping, standard SGD analyses break down in the heavy-tailed gradient noise regime \citep{zhang2020adaptive}.

\paragraph{Pre-clipping.} We call $\text{clip}_c(\cdot)$ a \emph{pre-clipping} step when it is applied before the spectral normalization step, popularized by Muon \citep{jordan2024muon}, yielding
\begin{equation}
\label{eq:pre_clipping}
    X_{k+1} = X_{k} - \msign\!\left(\clip_c(M_k)\right),
\end{equation}
where $M_k$ can be a momentum buffer of the stochastic gradient $\tilde{G}$ or the stochastic gradient itself. Originally,  the pre-clipping map was introduced in \citep{si2025adamuon} and \citep{he2025root} to stabilize the gradient accumulation before $\msign$. While $\text{msign}$ bounds the final update norm, the raw mini-batch average fed into $\text{msign}$ can still be dominated by a single extreme heavy-tailed realization; as shown in \cref{sec:noise_model} and \cref{sec:experiments}, such outliers induce spectral spikes that corrupt and rotate the gradient matrix toward an outlier-driven subspace. Pre-clipping can mitigate this and help subspace identification by suppressing coordinate-induced spectral spikes prior to spectral normalization, yielding a signal-direction-preserving proxy for the true batch gradient.

\section{Heavy-tailed Entries Can Drive Spectral Inflation}
\label{sec:noise_model}
Let $G = \nabla f(X) \in \RR^{m \times n}$ denote the true matrix gradient and $\tilde{G} = G + E$ denote the stochastic counterpart. In this section we compare two candidate models for the noise term $E$ and ask under what conditions an entry-wise perturbation of $G$ can meaningfully alter its spectrum. Through a \emph{localization} metric (\cref{sec:localization}), we show that a heavy-tailed mixture noise model induces spectral spikes whose effect on the top singular value can be controlled by entry-wise error minimization (\cref{sec:loc_to_mse}). We empirically verify that real stochastic gradient noise exhibits the same heavy-tailed, localized structure (\cref{fig:noise_models}).

\subsection{Subspace Perturbation vs. Entry-wise Contamination}

A recent work \citep{jiang2026enhancing} models the gradient noise $E$ as a rank-$r$ subspace perturbation (\cref{def:subspace}) with random orthonormal vectors $U_r, V_r$. Such a perturbation alters the spectrum of $G$ by injecting a discrete cluster of spikes at the top of its singular value distribution.

\begin{definition}[Subspace perturbation \citep{peche2006largest}]
\label{def:subspace}
The noise is a rank-$r$ matrix $E = \lambda\, U_r V_r^T$, where $\lambda > 0$ controls the spike magnitude and $U_r \in \RR^{m \times r}$, $V_r \in \RR^{n \times r}$ have orthonormal columns.
\end{definition}
Under the bounded second-moment assumption $\E[U_rU_r^{T}] \preceq cI_n$ (\citep{jiang2026enhancing}), the isotropic low-rank model is poorly aligned with the empirical behavior of stochastic gradient noise: real gradients do not exhibit the Gaussian-like entries or bulk structures that this model implies. As shown in the top strip of \cref{fig:noise_models}, the entries of real gradients display heavy-tailed outliers that extend beyond the Gaussian bulk, a structure that low-rank perturbations with bounded second moment cannot reproduce. This motivates the entry-wise mixture model in \cref{def:contamination}, which combines a heavy-tailed component with the Gaussian bulk to capture this effect.

\begin{definition}[Entry-wise contamination]
\label{def:contamination}
Each entry of $E$ is drawn i.i.d.\ from Huber's contamination model \citep{huber1992robust}:
\begin{equation}
    E_{ij} \overset{\text{iid}}{\sim} (1-\alpha)\,\gN(0, \sigma^2) + \alpha\, H,
    \qquad \alpha \in [0,1],
    \label{eq:noise}
\end{equation}
where the contaminating distribution $H$ is restricted to a symmetric heavy-tailed family (e.g., $\text{Cauchy}(\sigma)$ or Student's $t_\nu$).
\end{definition}

Heavy-tailed stochastic gradient noise has been studied by \citet{simsekli2019tail} and \citet{zhang2020adaptive}. In the large-scale language setting, it is further reinforced by the heavy-tailed nature of natural language itself, as captured by Zipf's law \citep{piantadosi2014zipf,zipf2016human}. However, whether entry-wise noise of this form in the matrix updates can produce spectral spikes analogous to those of the subspace perturbation model remains unknown.

\subsection{From First-order Perturbation to a Localization Metric}
\label{sec:localization}

To understand when the entry-wise error inflates the spectrum, we first state a result based on first-order perturbation theory. (We were unable to find this specific result stated anywhere, but believe it to be fairly classical. It's similar to results from \citet{magnus1985differentiating,kato1966perturbation}.)

\begin{restatable}[First-order expansion of $\sigma_1$]{theorem}{thmtaylor}
\label{thm:taylor}
Let $G \in \RR^{m \times n}$ have a simple top singular value
$\sigma_1(G) > \sigma_2(G)$, with corresponding left and right singular
vectors $u_1$ and $v_1$, and let $\delta := \sigma_1(G) - \sigma_2(G)$
denote the spectral gap. There exists a universal constant $C > 0$ such
that for every perturbation $E \in \RR^{m \times n}$ with $\|E\|_{op} < \delta/4$,
\begin{equation}
  \sigma_1(G + E) \;=\; \sigma_1(G) \;+\; \sum_{ij} (u_1)_i (v_1)_j\, E_{ij}
  \;+\; \theta, \qquad |\theta| \le C\,\frac{\|E\|_{op}^2}{\delta}.
  \label{eq:taylor}
\end{equation}
\end{restatable}

\begin{table}[t]
\centering
\small
\caption{\small \textbf{Localization ratio under three noise models.}
The Gaussian noise has entries drawn i.i.d.\ from $\gN(0, \sigma^2)$;
the entry-wise spike noise has a single nonzero entry at location $(i', j')$;
the $uv$-aligned low-rank noise adopts the signal $G$'s top singular
directions $u, v$ as its rank-one component. Rates for the Gaussian column are stated up to constants in expectation, and the entry-wise spike and $uv$-aligned columns are deterministic.}
\begin{tabularx}{\textwidth}{l *{3}{>{\centering\arraybackslash}X}}
\toprule
 & Gaussian Noise & Entry-wise Spike Noise & $uv$-Aligned Low-rank Noise \\
 & $(E)_{ij}  \overset{\text{iid}}{\sim}\gN(0,\sigma^2)$ & $E = e_{i'}e_{j'}^{T}$ & $E=cuv^{T}$ \\
\midrule
$r_{\max}(E)$ & $\Theta(\log(mn)/(mn)^2)$ & $\Theta(1/(mn))$  & $\Theta(1/(mn)^2)$  \\
$r(E)$ & $\Theta(1/(mn))$ &  $\Theta(1/(mn))$& $\Theta(1)$ \\
$R(E)$           & $\Theta(\log(mn)/(mn))$ & $\Theta(1)$                    & $\Theta(1/(mn)^2)$ \\
$\widehat{R}(E)$ & $\Theta(1)$     & $\Theta(mn/\log(mn))$  & $\Theta(1/mn\log(mn))$ \\
\bottomrule
\end{tabularx}

\label{tab:localization-regimes}
\end{table}
The above theorem shows that the leading-order change in $\sigma_1$ is a $(u_1)_i(v_1)_j$-weighted sum of the entries of $E$, and under the unique singular value assumption, it further extends to the remaining singular values $\sigma_i$. A single heavily weighted entry of $E$ can therefore dominate the perturbation. This motivates a metric, named as localization ratio, that quantifies how much of $E$'s alignment with $G$'s top singular direction is concentrated in a single entry rather than spread across the matrix.
\begin{definition}[Localization ratio]
\label{def:localization}
Let $G, E \in \R^{m \times n}$, and $u \in \R^m$, $v \in \R^n$ be the top left and right singular directions of signal $G$. The \emph{localization ratio} of $E$ relative to $G$ is
\begin{equation}
\label{eq:localization_metric}
    R(E) \;:=\; \frac{r_{\max}(E)}{r(E)},
    \qquad
    r_{\max}(E) \;:=\; \frac{\max_{i,j}\,\bigl|u_i\, E_{ij}\, v_j\bigr|^{2}}{\|E\|_{F}^{2}},
    \qquad
    r(E) \;:=\; \frac{\bigl|\langle uv^{T},\,E\rangle\bigr|^{2}}{\|E\|_{F}^{2}}.
\end{equation}
The \emph{normalized localization ratio} is obtained by rescaling $R(E)$ against a standard Gaussian baseline $E_{\text{Gauss}}$ with i.i.d.\ entries:
\begin{equation}
\label{eq:normalized_localization}
\widehat{R}(E) \;:=\; R(E) \,\big/\, \text{median}\!\left[R(E_{\text{Gauss}})\right].
\end{equation}
\end{definition}

Both $r$ and $r_{\max}$ measure how $E$ projects onto $u v^T$ as fractions of $\norm{E}_{F}^2$: $r$ is the full bilinear projection, while $r_{\max}$ is its largest single-entry contribution. Their ratio $R(E)$ therefore captures the fraction of the projection concentrated in a single entry, and normalizing by the Gaussian baseline in $\widehat R(E)$ rescales $R(E)$ to a Gaussian-relative scale where $\widehat{R}(E)=1$. Under a delocalization assumption of the signal's singular direction, i.e., $|u_i v_j| = \Theta(1/\sqrt{mn})$, the behavior of $\widehat R(E)$ under three representative noise models is summarized in \Cref{tab:localization-regimes}.

The table makes $\widehat R(E)$ usable as a directional-structure test on $E$. When $\widehat R(E) \ll 1$, $E$'s projection onto $u v^T$ is more significant than the Gaussian baseline, indicating a coherent low-rank component aligned with $G$'s top singular direction. When $\widehat R(E) \gg 1$, a single entry contributes a critical share of the projection, indicating spike-like contamination. When $\widehat R(E) \approx 1$, $E$ is indistinguishable from isotropic Gaussian noise. We call $E$ \emph{localized} in either of the first two cases ($\widehat R(E) \ll 1$ or $\widehat R(E) \gg 1$) and \emph{delocalized} in the third ($\widehat R(E) \approx 1$). 

In practice, when measuring the true gradient noise during LLM training, we observe a clear phase transition: at the early stage of training, the noise is aligned with the signal direction ($\widehat R(E) \ll 1$), whereas as training progresses, it becomes Gaussian-dominated ($\widehat R(E) \approx 1$), with a few layers exhibiting spike-like contamination ($\widehat R(E) \gg 1$) (see \cref{fig:r_hat}).


 
\subsection{From Localization to Entry-wise Minimization}
\label{sec:loc_to_mse}
We now show how the above localization metric can be leveraged to control the top singular value of the noisy signal through minimizing the entry-wise error.  The proposition below characterizes such optimal entry-wise estimator under an i.i.d.\ Gaussian prior on the clean signal $G$.

\begin{proposition}[First-order spectral control via entrywise MSE]
\label{prop:spectral_to_mse}
Let $\tilde G_{ij} = G_{ij} + E_{ij}$, and $\Delta := \varphi(\tilde G) - G$ for an arbitrary entrywise estimator $\varphi$. When $\norm{\Delta}_{op}<\delta/4$, replacing $E \to \Delta$ in Theorem \ref{thm:taylor} and applying Cauchy--Schwarz to the summation term yields
\begin{equation}
  \E\!\bigl[\bigl|\sigma_1(\varphi(\tilde G)) - \sigma_1(G)\bigr|^2\bigr]
  \;\le\; \sum_{i,j} \E\!\bigl[(\varphi(\tilde G_{ij}) - G_{ij})^2\bigr]
        \;+\; o\!\bigl(\|\Delta\|_{op}^2/\delta\bigr).
\label{eq:spectral_mse}
\end{equation}
If we place an entrywise prior to signal $G$, i.e., $G_{ij} \overset{\text{iid}}{\sim} \gN(0, \sigma_x^2)$, the leading right-hand side is minimized coordinate-wise by $\varphi^\star(y) = \E[\,x \mid y\,]$, where $(x, y)$ denotes the entry pair $(G_{ij}, \tilde G_{ij})$.
\end{proposition}

\begin{remark}
\label{rem:spectra_to_entry}
The bound in Proposition \ref{prop:spectral_to_mse} is approximately saturated when $\widehat{R}(\Delta) \ll 1$ (i.e., $\Delta$ concentrates along $u_1 v_1^T$), so entry-wise MSE minimization is essentially equivalent to minimizing the top singular-value perturbation. When $\widehat{R}(\Delta) \gg 1$, the bound is loose, but the dominant spike entries of $\Delta$ are themselves sparse, and entry-wise minimization suppresses them directly. When $\widehat{R}(\Delta) \approx 1$, $\Delta$ is delocalized and entry-wise minimization cannot selectively control the spectrum.
\end{remark}

Following \Cref{rem:spectra_to_entry}, the spectral perturbation admits entry-wise control in either of the two localized regimes identified there, $\widehat{R}(\Delta) \ll 1$ and $\widehat{R}(\Delta) \gg 1$. In both cases the Bayes-optimal MMSE estimator $\E[G_{ij} \mid \tilde G_{ij}]$ is first-order optimal. The latter regime is the more attractive of the two: the underlying entry-wise spike noise can be modeled directly by a sparse heavy-tailed contamination without extra knowledge of the signal's true singular directions like $u_1$ and $v_1$.

\section{Smooth Shrinkage: A Bayes-Motivated Entry-Wise Spectrum Control}
\label{sec:operator}

\begin{wrapfigure}[14]{r}{0.42\linewidth}
  \vspace{-1.5\baselineskip}   
  \centering
  \includegraphics[width=\linewidth]{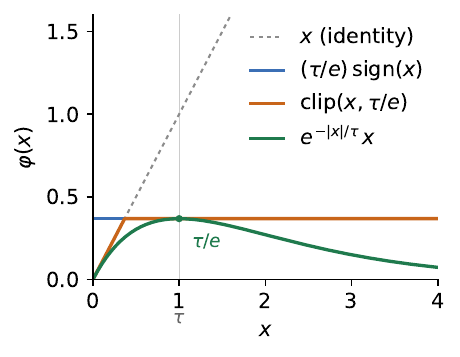}
  \vspace{-1.8\baselineskip}
  \caption{Three entry-wise operators on $x \geq 0$ with $\tau = 1$. We set $\beta=1$ for smooth shrinkage in \cref{eq:smooth_shrink}.}
  \label{fig:clipping_ops}
\end{wrapfigure}
 
Here we focus on the $\widehat{R}(\Delta) \gg 1$ regime, in which localization is induced by the entry-wise contamination model of \cref{def:contamination}. Our main result identifies a closed-form surrogate for the Bayes-optimal entry-wise estimator in \cref{prop:spectral_to_mse} that can control the spectral error and is asymptotically faithful in the large-$|\tilde G_{ij}|$ region, where heavy-tailed contamination dominates.

\begin{restatable}[Smooth shrinkage as a Bayes-motivated surrogate]{theorem}{thmmain}
\label{thm:main}
Let $y = x + e$, where $x \sim \gN(0,\sigma_x^2)$ and
$e \sim (1-\alpha)\gN(0,\sigma^2) + \alpha H$ with $\alpha \in [0,1)$,
and the heavy tailed density $h$ associated with $H$ satisfies the score and curvature conditions of Lemma \ref{lem:posterior_collapse}. Define the Wiener coefficient $\beta := \sigma_x^2/(\sigma_x^2 + \sigma^2)$ and the (Bayesian) retention probability $\pi(y) := \Pr(e \sim \gN \mid y)$.

\begin{enumerate}[label=\textnormal{(\roman*)},nosep,leftmargin=*]
\item \textbf{Asymptotic Bayes structure.}
For every $\varepsilon > 0$ there exists an explicit threshold
$y^\star = y^\star(\varepsilon, h, \sigma_x) \in (0,\infty)$
such that for all $|y| \ge y^\star$ the Bayes optimal entrywise estimator of $x$ given observation $y$ admits the factorization
\begin{equation}
  \E[x \mid y] \;=\; \pi(y)\,\beta\,y \;+\; \rho(y),
  \qquad |\rho(y)| \;\le\; \varepsilon.
\label{eq:thm_factor}
\end{equation}
Suppose $y$ is an entry of observed gradient $\tilde G$ and the entrywise MSE of $\tilde G$ is bounded (see \cref{prop:spectral_to_mse}), the above entrywise estimator also controls the spectral error
$\E\bigl[|\sigma_1(\varphi_\tau(\tilde G)) - \sigma_1(G)|^2\bigr]$.

\item \textbf{Unique closed form surrogate.}
Among continuous functions $\widehat\pi:[0,\infty)\to(0,1]$ satisfying
$\widehat\pi(y)\to 0$ as $y\to\infty$ \emph{(redescending)},
$\widehat\pi(y_1+y_2) = \widehat\pi(y_1)\,\widehat\pi(y_2)$ \emph{(multiplicative)}, and
$\widehat\pi(0) = 1$ \emph{(unit at origin)}, the unique one parameter family is
$\widehat\pi(y) = e^{-y/\tau}$ for some $\tau > 0$. Replacing $\pi$ in \cref{eq:thm_factor} with the even extension $y \mapsto e^{-|y|/\tau}$ yields the \emph{smooth shrinkage operator}
\begin{equation}
  \varphi_\tau(y) \;:=\; \beta\,e^{-|y|/\tau}\,y.
\label{eq:smooth_shrink}
\end{equation}
\end{enumerate}
\end{restatable}
 
By \Cref{thm:main}, smooth shrinkage damps each entry $y$ by a factor of $e^{-|y|/\tau}$, whereas hard coordinate-wise clipping (\cref{eq:coor_clip}) applies a damping factor of $\min\{1,\, c/|y|\}$. Setting $c = \tau/e$ aligns the two damping curves: they are tangent at $|y| = \tau$, where both equal $1/e$. For $|y| > \tau$, however, smooth shrinkage's exponential decay imposes a strictly stronger penalty than hard clipping's $1/|y|$ tail. \Cref{fig:clipping_ops} and \Cref{fig:clipping-comparison} illustrates this difference. The complete proof of \cref{thm:main} is provided in \cref{app:sec4_bayes}.




\section{Convergence Analysis}
\label{sec:convergence}


Consider a matrix-valued objective function $F: \mathbb{R}^{m \times n} \to \mathbb{R}$ with true gradient $G = \nabla F(X) \in \mathbb{R}^{m \times n}$. We evaluate the convergence result of the hard clipping map $C_\tau(x)$ and the newly introduced smooth shrinkage map $S_c(x)$ as two distinct entry-wise clipping methods at both \emph{post-clipping} and \emph{pre-clipping} stages. The definitions of  $C_\tau(x)$ and $S_c(x)$ follows:
\begin{equation}
\label{eq:two_clipping_op}
C_\tau(x) :=x\min\{1,\tau/|x|\}, \qquad S_c(x) := xe^{-|x|/c},    
\end{equation}
where we use threshold $c$ for smooth shrinkage in this section to distinguish it from $\tau$ used by the hard clipper, and its Wiener coefficient $\beta$ is absorbed into the learning rate. 

Under the proposed entry-wise contamination model in \cref{eq:noise} and \cref{ass:smooth-main,ass:bounded-main,ass:noise-main}, we establish the standard $\mathcal{O}(\varepsilon^{-4})$ iteration complexity result (\cref{thm:post-hard-main,thm:post-smooth-main,thm:pre-hard-main,thm:pre-smooth-main}) for reaching the $\varepsilon$-stationarity of a nonconvex function by stochastic first-order oracles. The overall complexity using the above two clipping methods is summarized in \cref{tab:convergence-comparison-main}. 

\begin{assumption}[Smoothness]
\label{ass:smooth-main}
For $X, Y \in \mathbb{R}^{m\times n}$, the objective $F: \mathbb{R}^{m\times n} \to \mathbb{R}$ is $L$-smooth in the Frobenius norm: $F(Y) \le F(X) + \langle\nabla F(X), Y-X\rangle + \frac{L}{2}\|Y-X\|_F^2.$
Moreover, we assume the objective is bounded below such that $\Delta := F(X_0) - \inf_X F(X) < \infty$. We denote $G_k := \nabla F(X_k)$ and let $d := mn$.
\end{assumption}

\begin{assumption}[Bounded True Gradients]
\label{ass:bounded-main}
There exists a constant $B < \infty$ such that the $\ell_\infty$ norm of the true gradient is bounded for all iterates:
$\|G_k\|_{\infty} = \max_{i,j} |(G_k)_{i,j}| \le B.$
\end{assumption}

\begin{assumption}[Entry-wise Noise Model]
\label{ass:noise-main}
Given $X_k$, a noisy gradient sample takes the form $G_k + \Xi_k$. $\Xi_k$ is independent of $X_k$ and across $k$. Each entry of $\Xi_k$ follows a symmetric noise mixture model: $\xi_k \sim (1-\alpha)\mathcal{N}(0,\sigma_k^2) + \alpha\,\text{Cauchy}(0,\gamma_k),$
where $0 \le \alpha \le 1$, $\sigma_k \ge 0$, and $\gamma_k > 0$. For all iterates, the local scales are uniformly bounded by $\sigma, \gamma>0$ such that $\sigma_k \le \sigma$ and $\gamma_k \le \gamma$.
\end{assumption}

The $L$-smoothness condition in \cref{ass:smooth-main} is standard in nonconvex optimization \citep{zhang2020adaptive}. We additionally assume boundedness of the true gradient in \cref{ass:bounded-main}, a condition also used in convergence analyses of adaptive methods \citep{ward2020adagrad,crawshaw2022robustness}. Rather than characterizing the noise solely through a finite $\mathfrak{p}$-moment condition \citep{zhang2020adaptive,liu2024nonconvex}, our mixture noise model in \cref{ass:noise-main} specifies additional distributional structure. By exploiting the explicit form of the noise density, we establish an $\gO(\varepsilon^{-4})$ complexity bound under the specified heavy-tailed noise model.


\begin{table}[t]
\small
\caption{\textbf{Main convergence results for \emph{post-clipping} (\cref{eq:post_clipping}) and \emph{pre-clipping} (\cref{eq:pre_clipping}).} To obtain $\varepsilon$-stationarity, under the entry-wise contamination model in \cref{eq:noise} and \cref{ass:smooth-main,ass:bounded-main,ass:noise-main}, pre-clipping requires $\mathcal{O}(L\Delta D^2 r^4 v_ \diamond \varepsilon^{-4})$ iterations whereas post-clipping requires $\mathcal{O}(L\Delta d\, \diamond^2 \varepsilon^{-4})$ with $ \diamond\in\{\tau,c\}$. In the table below, $S_{c}$ and $C_{\tau}$ are different clipping methods with thresholds $c$ and $\tau$ introduced in \cref{eq:two_clipping_op}. 
$B$ is the upper bound of the true gradient, $\sigma$ and $\gamma$ are the noise scale, $\alpha$ is the mixture variable, and $r$ is the smallest dimension of the gradient matrix. Exact constants for clipping threshold $\tau$ and $c$ are presented in \cref{thm:post-hard-main,thm:post-smooth-main,thm:pre-hard-main,thm:pre-smooth-main}.}
\begin{tabularx}{\linewidth}{@{} l l c l @{\quad} l @{}}
\toprule
Method & Clipping Threshold scale & $\varepsilon$-Stationarity &  \multicolumn{2}{l@{}}{Complexity} \\
\midrule
Post ($C_{\tau}$)  & $\tau\gtrsim B+ \sigma+\alpha\gamma$ & $\E\norm{G_k}_F$ & $\gO(L\Delta d\tau^2\varepsilon^{-4})$ &(\cref{thm:post-hard-main}) \\
Post ($S_{c}$)  & $c\gtrsim B+(1-\alpha)\sigma+\alpha\gamma$ & $\E\norm{G_k}_F$ & $\gO(L\Delta dc^2\varepsilon^{-4})$ & (\cref{thm:post-smooth-main})  \\
\addlinespace
Pre ($C_{\tau}$) & $\tau\gtrsim B+\sigma+\alpha\gamma\sqrt r$ & $\E\norm{G_k}_*$ & $\gO(L\Delta D^2r^4v_\tau\varepsilon^{-4})$ & (\cref{thm:pre-hard-main}) \\
Pre ($S_{c}$) & $c\gtrsim \sqrt r\{B+(1-\alpha)\sigma+\alpha\gamma\}$ & $\E\norm{G_k}_*$ & $\gO(L\Delta D^2r^4v_c\varepsilon^{-4})$ & (\cref{thm:pre-smooth-main}) \\
\bottomrule
\end{tabularx}
\label{tab:convergence-comparison-main}
\end{table}


\subsection{Convergence Result for Post-clipping }
\label{subsec:post-main}
Given an entry-wise mapping $\varphi : \mathbb{R}^{m \times n} \to \mathbb{R}^{m \times n}$, we consider the following single-sample clipped stochastic gradient update: 
\begin{equation}
    X_{k+1}=X_k-\eta\,\varphi(G_k+\Xi_k).
    \label{eq:post-update-main}
\end{equation}
The subsequent theorems establish the convergence guarantees when the mapping $\varphi$ is instantiated as either the hard-clipping function $C_\tau$ or the smooth-shrinkage function $S_c$.

\begin{restatable}[Post-clipping with hard clipping]{theorem}{thmposthard}
\label{thm:post-hard-main}
Under \Cref{ass:smooth-main,ass:bounded-main,ass:noise-main}, choose $\tau$ such that
$
    \tau\ge B+
    \max\left\{
    \sigma\sqrt{2\log 8},\,
    \frac{8\alpha\gamma}{\pi}
    \right\}.
$
Executing the update rule \cref{eq:post-update-main} with $\varphi=C_\tau$ and a step size of $\eta=\sqrt{2\Delta/(L\tau^2dK)}$ for $K$ iterations yields:
\begin{equation}
    \frac1K\sum_{k=0}^{K-1}\E\norm{G_k}_F^2
    \le
    2\tau\sqrt{\frac{2L\Delta d}{K}}.
    \label{eq:post-hard-main-rate}
\end{equation}
Consequently, when $K\ge 8L\Delta d\tau^2/\varepsilon^4$, we guarantee
$
    \frac1K\sum_{k=0}^{K-1}\E\norm{G_k}_F\le\varepsilon.
$
\end{restatable}

\begin{restatable}[Post-clipping with smooth shrinkage]{theorem}{thmpostsmooth}
\label{thm:post-smooth-main}
Under \cref{ass:smooth-main,ass:bounded-main,ass:noise-main}, choose
$
    c\ge
    \max\left\{
    4\left[B+\sqrt{\frac8\pi}(1-\alpha)\sigma\right],\,
    \frac{64\alpha\gamma}{\pi}
    \log\left(e+\frac{64\alpha}{\pi}\right)
    \right\}.
$
Executing the update rule \cref{eq:post-update-main} with $\varphi=S_c$ and a step size of $\eta=e\sqrt{2\Delta/(Lc^2dK)}$ for $K$ iterations yields:
\begin{equation}
    \frac1K\sum_{k=0}^{K-1}\E\norm{G_k}_F^2
    \le
    \frac{2c}{e}\sqrt{\frac{2L\Delta d}{K}}.
    \label{eq:post-smooth-main-rate}
\end{equation}
Consequently, when $K\ge 8L\Delta d c^2/(e^2\varepsilon^4)$, we guarantee
$
    \frac1K\sum_{k=0}^{K-1}\E\norm{G_k}_F\le\varepsilon.
$
\end{restatable}

Examining the threshold choices in \cref{thm:post-hard-main,thm:post-smooth-main} reveals an interesting phenomena: hard clipping $C_\tau$ leaves any coordinate within the range of the true gradient unaltered, while smooth shrinkage $S_c$ requires a larger threshold to preserve ordinary signals and control the Cauchy tail. In both cases, the threshold remains \emph{constant} once the noise scale and gradient bounds are established, complementary to existing clipping methods that require iteration-dependent thresholds \citep{zhang2020adaptive}.

\subsection{Convergence Result for Pre-clipping}
\label{subsec:pre-main}

Given an entry-wise mapping $\varphi : \mathbb{R}^{m \times n} \to \mathbb{R}^{m \times n}$, we consider the following mini-batched pre-clipping using $N$ \emph{i.i.d.}\ samples:
\begin{equation}
    \bar G_k^\varphi:=\frac1N\sum_{\ell=1}^N\varphi(G_k+\Xi_k^{(\ell)}),
    \qquad
    X_{k+1}=X_k-\eta\,\operatorname{msign}(\bar G_k^\varphi),
    \label{eq:pre-update-main}
\end{equation}
where $\operatorname{msign}(M) = U V^T$ for the singular value decomposition $M = U \Sigma V^T$. We define $r := \min\{m,n\}$, $q := \max\{m,n\}$, and $D := \sqrt{q/r}$. The subsequent theorems establish convergence guarantees when $\varphi$ is instantiated as either $C_\tau$ or $S_c$.

\begin{figure}
    \centering
    \includegraphics[width=0.95\linewidth]{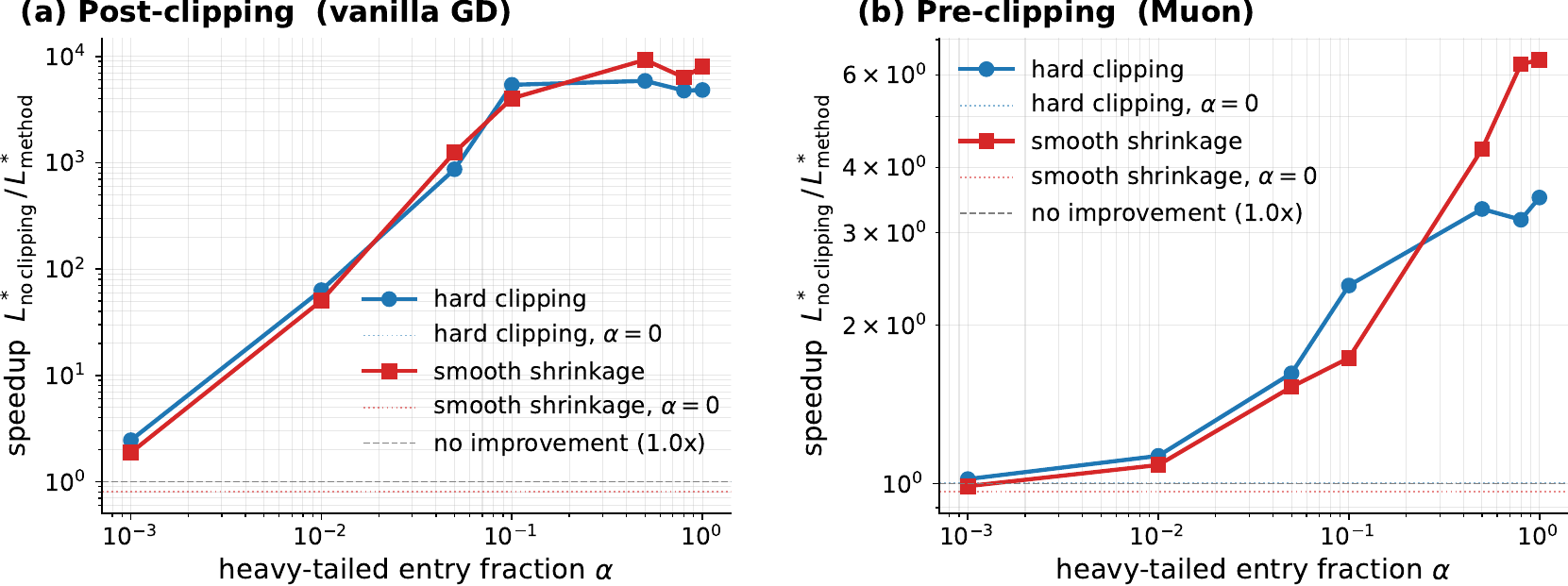}
    \caption{Random Gaussian feature regression ($d=32$, $n=128$) with a fraction $\alpha$ of feature entries corrupted by Student-$t$ noise ($\nu=1$, scale $3.0$). $y$-axis is the final-loss speedup over no clipping. Smooth shrinkage (red) in both \emph{post-clipping} (a) and \emph{pre-clipping} (b) tracks hard clipping (blue) for small $\alpha$ but pulls ahead as the noise is dominated by heavy-tailed entries.}
    \label{fig:synthetic_result}
\end{figure}

\begin{restatable}[Pre-clipped Muon with hard clipping]{theorem}{thmprehard}
\label{thm:pre-hard-main}
Under \Cref{ass:smooth-main,ass:bounded-main,ass:noise-main}, choose $\tau$ such that
$
    \tau\ge B+
    \max\left\{
    \sigma\sqrt{2\log(16\sqrt r)},\,
    \frac{16\alpha\gamma\sqrt r}{\pi}
    \right\},
    \label{eq:pre-hard-threshold-main}
$
and define
$v_\tau:=(1-\alpha)\sigma^2+\frac{4\alpha\gamma\tau}{\pi}.$
Executing the update rule \cref{eq:pre-update-main} with $\varphi=C_\tau$ and a step size of $\eta=\sqrt{2\Delta/(LrT)}$ for $T$ iterations yields:
\begin{equation}
    \frac1T\sum_{k=0}^{T-1}\E\norm{G_k}_*
    \le
    2\sqrt{\frac{2Lr\Delta}{T}}
    +4Dr^{3/2}\sqrt{\frac{v_\tau}{N}}.
    \label{eq:pre-hard-main-rate}
\end{equation}
Consequently, for $T \ge 32Lr\Delta/\varepsilon^2$ and $N \ge 64D^2r^3v_\tau/\varepsilon^2$, we guarantee
$
    \frac1T\sum_{k=0}^{T-1}\E\norm{G_k}_*\le\varepsilon.
$
\end{restatable}

\begin{restatable}[Pre-clipped Muon with smooth shrinkage]{theorem}{thmpresmooth}
\label{thm:pre-smooth-main}
Under \Cref{ass:smooth-main,ass:bounded-main,ass:noise-main}, choose
$
    c\ge
    \max\left\{
    8\sqrt r\left[B+\sqrt{\frac8\pi}(1-\alpha)\sigma\right],\,
    \frac{128\alpha\gamma\sqrt r}{\pi}
    \log\left(e+\frac{128\alpha\sqrt r}{\pi}\right)
    \right\},
    \label{eq:pre-smooth-threshold-main}
$
and define
$
    v_c:=(1-\alpha)\sigma^2+\frac{8\alpha\gamma c}{\pi e}.
$
Executing the update rule \cref{eq:pre-update-main} with $\varphi=S_c$ and a step size of $\eta=\sqrt{2\Delta/(LrT)}$ for $T$ iterations yields:
\begin{equation}
    \frac1T\sum_{k=0}^{T-1}\E\norm{G_k}_*
    \le
    2\sqrt{\frac{2Lr\Delta}{T}}
    +4Dr^{3/2}\sqrt{\frac{v_c}{N}}.
    \label{eq:pre-smooth-main-rate}
\end{equation}
Consequently, for $T \ge 32Lr\Delta/\varepsilon^2$ and $N \ge 64D^2r^3v_c/\varepsilon^2$, we guarantee
$
    \frac1T\sum_{k=0}^{T-1}\E\norm{G_k}_*\le\varepsilon.
$
\end{restatable}


\textbf{The heavy tail enters the constants, not the exponent.} Under the explicit noise model in \Cref{ass:noise-main}, despite the raw oracle $\Tilde{G}$ having no first absolute moment, all four rates retain the standard $\gO(\varepsilon^{-4})$ accuracy dependence: the Cauchy heavy tail is absorbed into the threshold and variance constants, not into the complexity exponent. This is possible because our noise is an entry-wise \emph{symmetric location} perturbation around the clean gradient, enabling us to derive the \textit{multiplicative} bias bound $|\E[\varphi(g+\xi)]-g|\le\rho|g|$. On the contrary, the asymmetric finite $\mathfrak{p}$-moment assumption for the stochastic gradient noise \citep{zhang2020adaptive,nguyen2023improved, madden2024high,liu2024nonconvex}, instead gives an \textit{additive} bias bound of the form $|\E[\varphi (g+\xi)-g]|\le C\sigma^{\mathfrak{p}}/\tau^{\mathfrak{p-1}}$ with the rate deteriorated with the heavy tail component.
We present detailed proofs of the above theorems in \Cref{app:convergence-proofs}.



\section{Experiments}
\label{sec:experiments}

\begin{figure}[t]
    \centering
    \includegraphics[width=0.99\linewidth]{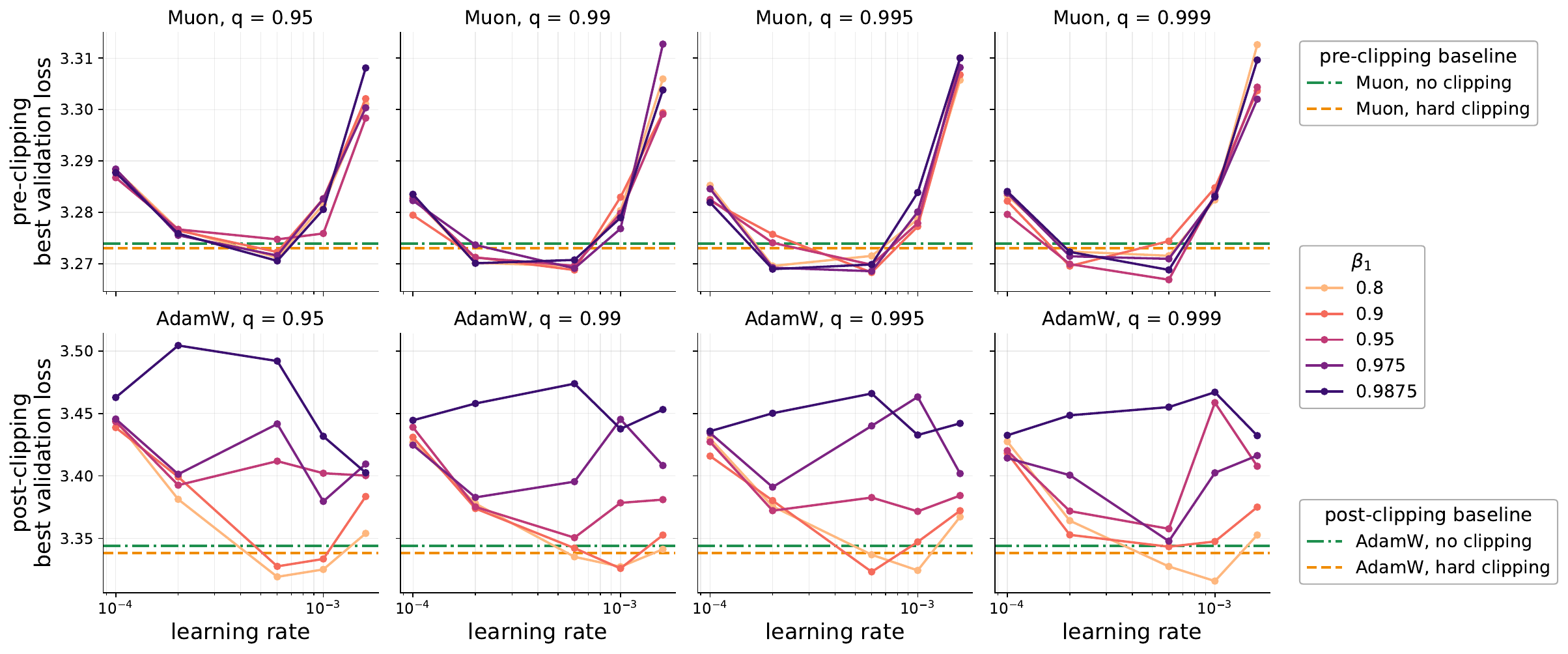}
   \caption{Smooth shrinkage for pre-clipping (\textit{Top}) and post-clipping (\textit{Bottom}) under varying clipping
thresholds in quantile number $q$ and first-order momentum coefficients $\beta_1$, evaluated by
final validation loss on a 274M-parameter Modded-NanoGPT. \textit{Takeaway 1}: Smooth shrinkage outperforms tuned Muon and AdamW baselines in both pre-clipping and post-clipping. \textit{Takeaway 2}: Pre-clipping is more robust to the choice of clipping
threshold $q$ and momentum coefficient.
\textit{Takeaway 3}: Clipping threshold $q \in [0.99, 0.999]$ is a useful search range for both
methods, and pre-clipping requires a longer gradient memory
($\beta_1^* = 0.95$) than post-clipping ($\beta_1^* = 0.8$) to achieve its best
performance.}
    \label{fig:ss_sweep}
\end{figure}

\begin{figure}
    \centering
    \includegraphics[width=0.95\linewidth]{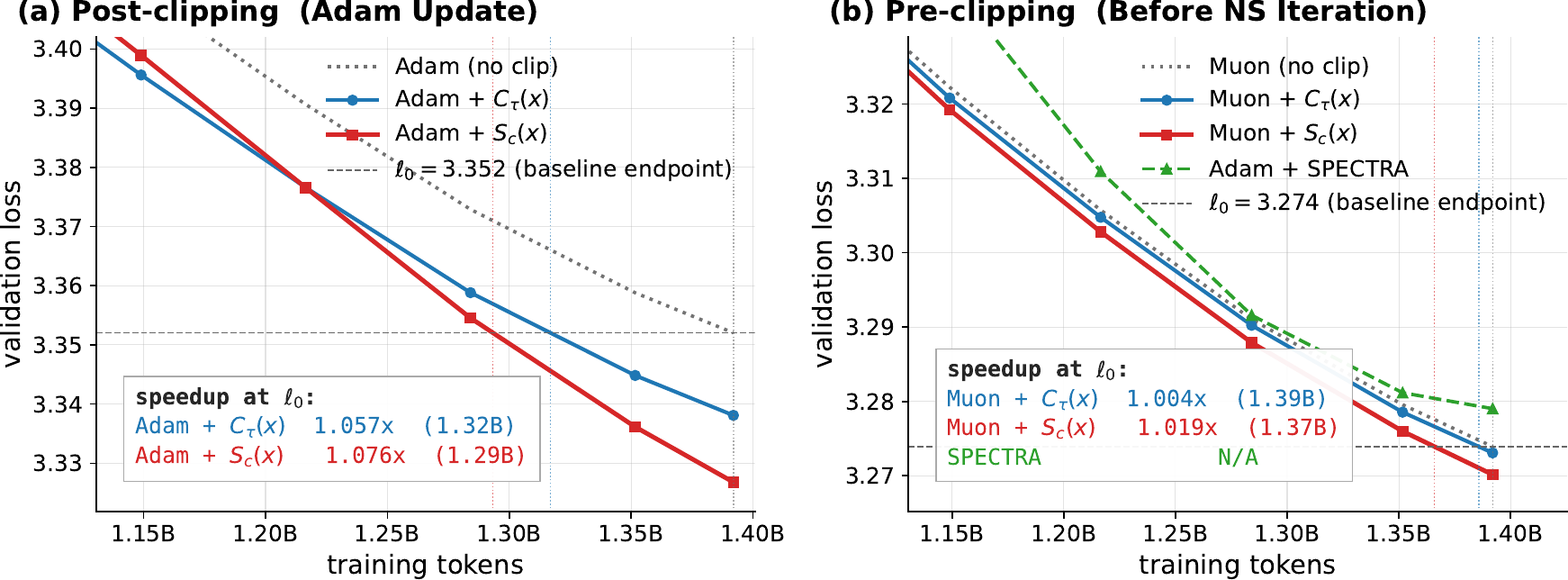}
    \caption{Validation loss vs. tokens on Modded-NanoGPT under hard clipping ($C_{\tau}(x)$) and smooth shrinkage ($S_{c}(x)$). $\ell_0$ is the no-clipping baseline's final loss;  $\text{speedup}:=\text{tokens}_{\text{baseline}}/\text{tokens}_{\text{methods}}$ with linear interpolation between eval points. Smooth shrinkage improves Adam and Muon on NanoGPT through saving $\sim{7.6}\%$ and $\sim{1.9}\%$ training tokens respectively.}
    \label{fig:nano_gpt_result}
\end{figure}

\paragraph{Random Gaussian feature regression.}
We consider minimizing $L(W) = \tfrac{1}{2n}\|WA - Y\|_F^2$  with i.i.d.\ Gaussian features $A$, where gradients are corrupted by the contamination model in~\cref{eq:noise} with Cauchy heavy tails ($H = t_1$). We compare vanilla GD against spectral GD, each with no clipping, hard clipping, and smooth shrinkage. We report the median of the minimized loss in \cref{fig:synthetic_result} after a grid search over learning rate and clipping threshold under a fixed number of iteration steps. We observe that both clipping methods can help recover the rotated subspace corrupted by heavy-tailed noise \cref{fig:post_d32,fig:post_d64,fig:pre_d32,fig:pre_d64} (a). We provide the full setup and complete results in \cref{app:synthetic}. 

To test robustness beyond the assumed Cauchy tail, we repeat the experiment
under two forms of noise misspecification: (i) a misspecified Student-$t$
tail index and (ii) a low-rank contamination model inspired by the BBP phase transition \citep{baik2005phase,benaych2012singular}. For a misspecified tail index, both clipping maps
retain their speedup as the tail becomes lighter, and smooth shrinkage
matches or exceeds hard clipping at high contamination (large $\alpha$ in \cref{eq:noise}). Under the structured
low-rank noise, both maps remain beneficial in the post-clipping setting. The detailed construction of
the misspecified noise models and the complete results are deferred
to \cref{sec:noise_misspec}.

\paragraph{Language model pretraining.}
\begin{wrapfigure}{r}{0.5\textwidth}

  \begin{center}
    \includegraphics[width=0.5\textwidth]{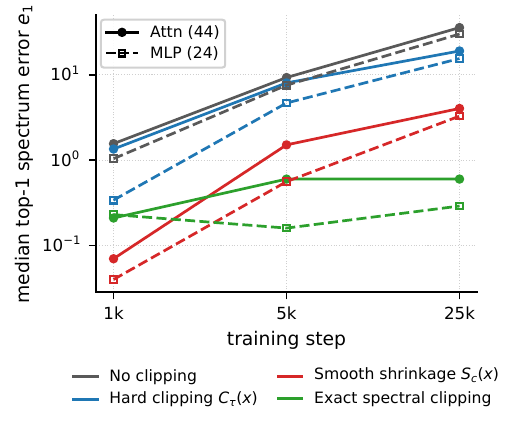}
  \end{center}
  \vspace{-10pt}
  \caption{Median top-1 spectrum error $e_1$ of processed mini-batch
  gradients on NanoGPT, over 44 attention (solid) and 24 MLP (dashed)
  weight matrices at checkpoints of steps 1k, 5k, and 25k. Unclipped
  gradients increasingly overestimate the true leading singular value as
  training progresses. $S_{c}(x)$ is markedly stronger than $C_{\tau}(x)$ throughout and attains the lowest error at step 1k, while exact
  spectral clipping dominates at later checkpoints. }
\label{fig:spectrum_control}

\end{wrapfigure}

We use the modded-nanoGPT codebase \citep{modded_nanogpt_2024}, where it is equipped with rotary embeddings, RMSNorm, linear decay learning-rate schedule, ReLU\textsuperscript{2} activations, three value-embedding layers, and an untied output head, totalling 274M parameters. We pretrain it on FineWeb-Edu \citep{penedo2024fineweb}. We present the sensitivity results of the clipping threshold $q$ and gradient momentum coefficient $\beta_1$ for both pre-clipping and post-clipping in \cref{fig:ss_sweep}. We report validation loss curves under each method's best learning rate and clipping threshold in \cref{fig:nano_gpt_result}. We observe that smooth shrinkage improves Adam and Muon by saving $\sim 7.6\%$ and $\sim 1.9\%$ of training tokens, respectively. We provide the detailed experiment settings and hyperparameter sweeping result in \cref{app:lm}.

\paragraph{Spectrum control on real gradients}
We further test whether entry-wise clipping improves spectral fidelity on real
gradients of NanoGPT by investigating checkpoints at different training steps. At each
checkpoint, we approximate the true gradient $G$ of each weight matrix
using a large batch. We measure the top-1 spectrum error
$e_1 = |\hat{\sigma}_1 - \sigma_1|^2 / \sigma_1^2$ of each processed
gradient against the large-batch reference, where $\hat{\sigma}_1$ is the singular value of the clipped stochastic gradient and $\sigma_1$ is the true gradient singular value approximated by a large batch. As shown in
\cref{fig:spectrum_control}, unprocessed gradients (gray line) increasingly
overestimate the leading singular value as training progresses; both
entry-wise maps reduce this error, with smooth shrinkage (red line) markedly stronger
than hard clipping (blue line). Smooth shrinkage even attains the lowest error at step
1k, whereas exact spectral clipping (green) dominates later. This suggests smooth shrinkage is a low-cost spectral proxy for exact spectral clipping, especially early in training when noise is aligned to the signal direction. Details and additional results are deferred to
\cref{sec:spectrum_control_details}.


\section*{Acknowledgement}
ZS and DFG would like to acknowledge DOE DE-SC0023162 Sparsitute MMICC center for partial support. CSB and BB would like to acknowledge support from NSF Award \#2420724.
\printbibliography

\newpage
\appendix

\begin{center}
    {\large \bf Appendix}
    \vspace{1em}
\end{center}

\addtocontents{toc}{\protect\setcounter{tocdepth}{2}}
\tableofcontents
\bigskip

\newpage
\begin{center}
\Large\textbf{Appendix}
\end{center}
\section{Related Work}
\label{app:relate_work}
\paragraph{Gradient Clipping Methods.}
Global gradient clipping is widely used to stabilize training in deep learning
\citep{pascanu2013difficulty}. When the stochastic gradient noise is heavy-tailed
\citep{simsekli2019tail}, under the assumption of an unbiased gradient oracle with
finite $p$-th central moment for some $p \in (1,2]$, \citet{zhang2020adaptive} first
established the $\Omega\!\bigl(T^{\frac{1-p}{3p-2}}\bigr)$ lower bound for nonconvex
optimization; \citet{zhang2020adaptive} and \citet{nguyen2023improved} respectively
provided matching in-expectation and high-probability upper bounds for clipped SGD,
up to extra logarithmic factors. These analyses, however, rely on iteration-dependent
clipping thresholds, e.g.\ $c_t = c\cdot t^{\frac{1}{3p-2}}$. 

Beyond the global clipping
variants, under coordinate-wise heavy-tailed noise, coordinate-wise clipping has been
shown to converge faster than global clipping for $\mu$-strongly convex objectives
\citep{zhang2020adaptive}. More recently, global clipping has been extended to
non-Euclidean space in the context of large-scale language model training
\citep{pethick2025generalized,jiang2026enhancing}: under a standard bounded-variance
noise model, an $\gO(T^{-1/4})$ rate was established within a Frank-Wolfe analysis
framework \citep{pethick2025generalized}, and \citet{sfyraki2025lions} obtained a high probability
$\tilde\gO\bigl(T^{\frac{1-p}{3p-2}}\bigr)$ rate under the bounded $p$-th moment
assumption. Nonetheless, little is known about how clipping behaves in the
unbounded-noise regime ($p=1$), even when the noise density distribution is known \citep{armacki2023high}.

\paragraph{Normalized Methods.}
Normalized gradient methods \citep{nesterov1984minimization,nesterov2018lectures}
are another line of research that controls gradient noise without requiring an
iteration-dependent clipping threshold. Under the Euclidean norm and a finite
$p$-th moment assumption on the noise, \citet{cutkosky2021high} and
\citet{liu2023breaking} study variants of normalized SGD that incorporate gradient
clipping as an additional ingredient. More recently, \citet{hubler2024gradient}
and \citet{liu2024nonconvex} have shown that normalized SGD attains the optimal
complexity $\gO\!\bigl(T^{\frac{1-p}{3p-2}}\bigr)$ in the nonconvex setting
\emph{without} any clipping. Beyond the Euclidean geometry, under a standard
bounded-variance gradient oracle, SignSGD \citep{bernstein2018signsgd} and Stacey
\citep{luo2025stacey} have been shown to achieve an $\gO(T^{-1/4})$ rate with
stationarity measured in the $\ell_1$ and $\ell_{p^*}$ (dual) norms,
respectively, in the nonconvex setting. Most recently, spectral-norm normalization
(i.e., Muon-type updates) has likewise been shown to attain the $\gO(T^{-1/4})$
complexity under bounded variance \citep{shen2025convergence}.

\paragraph{Spectral Methods.}
Among normalized methods, spectral-norm methods
\citep{carlson2015stochastic,jordan2024muon,bernstein2024old} respect the matrix structure of
large-scale neural networks and have emerged as an alternative to optimizers
that flatten weights into long vectors. Within this spectral geometry, one can
develop adaptive spectral methods \citep{song2026decoupling} and spectral
clipping methods \citep{newhouse2025training,jiang2026enhancing}, in direct
analogy to Adam \citep{kingma2014adam} and global gradient clipping in the
adaptive vector optimizer or vector clipping setting. Nevertheless, how coordinate-wise gradient clipping affects
the spectrum of the weight matrices remains largely unexplored.

\newpage
\section{Additional Information for Smooth Shrinkage and Limitations}
\label{app:ss_limitation}
In this section, we provide additional details and a reference implementation of smooth shrinkage. \Cref{fig:clipping-comparison} compares its damping factor $e^{-|y|/\tau}$ against the hard clipping factor $\min\{1, c/|y|\}$.
\label{app:ss}
\begin{figure}[!h]
  \centering
  \includegraphics[width=0.95\linewidth]{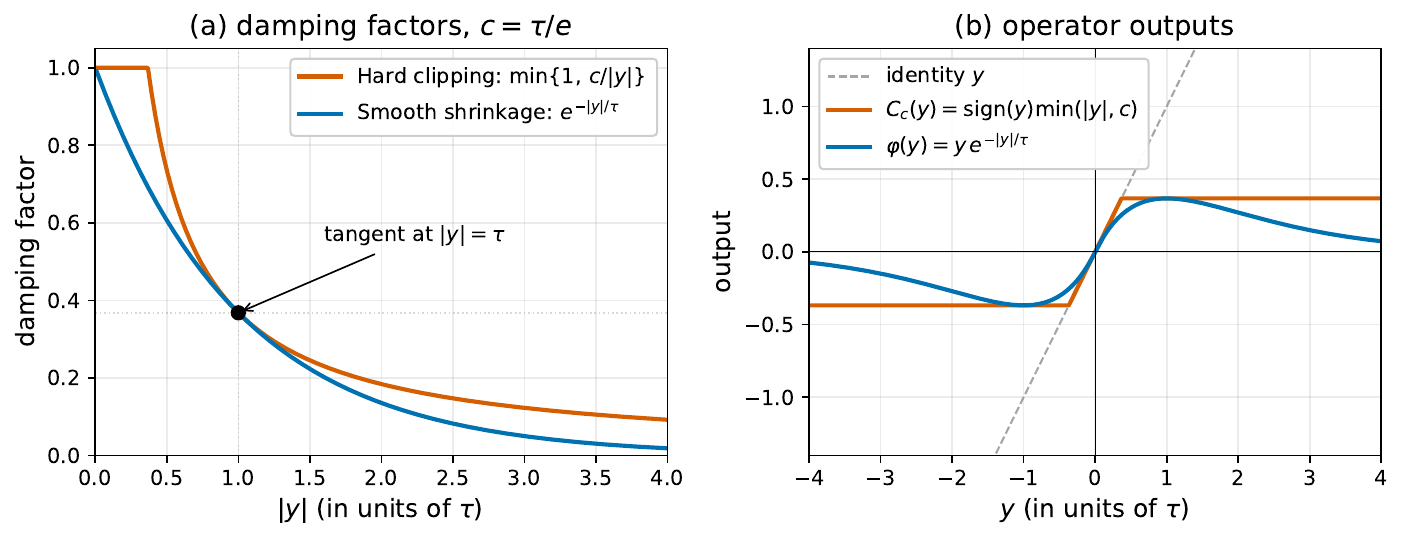}
  \caption{Comparison of smooth shrinkage and hard coordinate-wise clipping at the tangent calibration $c = \tau/e$. \textbf{(a)} Damping factors $e^{-|y|/\tau}$ and $\min\{1, c/|y|\}$ meet tangentially at $|y|=\tau$ with value $1/e$. \textbf{(b)} Operator outputs: hard clipping plateaus at magnitude $c$, while smooth shrinkage peaks near $|y|=\tau$ and decays exponentially, imposing a strictly stronger penalty on large entries.}
  \label{fig:clipping-comparison}
\end{figure}

We present the implementation below, which adds only two lines to the original optimizer code.
\begin{lstlisting}[style=pythonbox]
def smooth_shrinkage(x, q):
    c = torch.quantile(x.abs(), q)
    return x * torch.exp(-x.abs() / max(c, 1e-12))

# ---- Post-clipping: apply phi to any optimizer update direction U ----
U = ...                       # e.g., AdamW: m / (v.sqrt() + eps)
U = smooth_shrinkage(U, q)
W -= lr * U

# ---- Pre-clipping: apply phi to any signal S before matrix-sign ----
S = ...                       # e.g., momentum m, raw gradient g, ...
S = smooth_shrinkage(S, q)
W -= lr * scale * msign(S)    # scale = 0.2 * sqrt(max(n, m))
\end{lstlisting}

\paragraph{Broader Impacts} This work contributes algorithmic improvements to the optimization of matrix-valued weights in large-scale neural network training. The direct positive consequence is improved compute and energy efficiency: token savings of $\sim$$7\%$ for Adam and $\sim$$2\%$ for Muon during NanoGPT pretraining translate to proportional reductions in resource cost and carbon footprint at scale, with potential to broaden access for compute-constrained research groups.

\paragraph{Limitations} The derivation of the above smooth shrinkage in \cref{thm:main} places a Gaussian prior $G_{ij} \overset{\text{iid}}{\sim} \mathcal{N}(0, \sigma_x^2)$ on the clean gradient signal. While the entrywise contamination model (\cref{def:contamination}) captures heavy-tailed fluctuations, the assumption that the underlying full-batch gradient has a Gaussian-like bulk is a modeling choice rather than an empirical fact. Structured signals aligned with task-specific low-rank subspaces may violate it and shift the form of the optimal entrywise estimator. In addition, the dichotomy between entrywise heavy-tailed contamination and bounded-second-moment subspace perturbation does not exhaust the relevant alternatives: an anisotropic low-rank perturbation without bounded second moment can itself induce heavy-tailed spectral structure, and a unified analysis of this is more involved and thus left to future work.
\newpage
\section{Proofs for \texorpdfstring{\Cref{sec:noise_model}}{Section 3}}
\label{app:sec3}

\subsection{Proof of~\cref{thm:taylor}}
\label{app:taylor}

\thmtaylor*

\begin{remark}
The bilinear first-order coefficient $u_1^T E v_1$ is classical and goes back to \citet{magnus1985differentiating}; analytic perturbation theory for simple eigenvalues of symmetric matrices is developed in \citet[Ch.~II]{kato1966perturbation}. The Taylor expansion in the form of \cref{eq:taylor}, together with an explicit quadratic remainder, is not stated in those references, so we provide a self-contained proof in the following section.
\end{remark}

\begin{proof}[Proof of \Cref{thm:taylor}] \

\noindent\textbf{Symmetric reduction}
Define the symmetric block matrices
\[
  A \;:=\; \begin{pmatrix} 0 & G \\ G^T & 0 \end{pmatrix},
  \qquad
  B \;:=\; \begin{pmatrix} 0 & E \\ E^T & 0 \end{pmatrix}
  \;\in\; \RR^{(m+n) \times (m+n)},
\]
and set $N := m + n$. The spectrum of $A$ is
$\{\pm \sigma_k(G)\}_{k=1}^{\min(m, n)} \cup \{0, \dots, 0\}$, with
$2|m - n|$ zero eigenvalues. The hypothesis $\sigma_1(G) > \sigma_2(G) \ge 0$
implies that $\lambda_* := \lambda_1(A) = \sigma_1(G)$ is a simple eigenvalue
with unit eigenvector $w_* := \tfrac{1}{\sqrt 2}\binom{u_1}{v_1}$, and the
gap between the top two eigenvalues of $A$ equals $\delta$. A direct
computation gives $\|B\|_{op} = \|E\|_{op}$ and
$\sigma_1(G + E) = \lambda_1(A + B)$, so it suffices to establish the
expansion for $\lambda_1$.

\medskip
\noindent\textbf{Local differentiability of $\lambda(\cdot)$ and $w(\cdot)$} Here we show the differentiability of $\lambda(\cdot)$ and $w(\cdot)$ along the perturbation direction via the Implicit Function Theorem.

Let $S\in \mathrm{Sym}(N)$ be a $N$ by $N$ symmetric perturbation to the matrix $A$, the eigenvalue problem of on perturbed matrix $A+S$ defines a $S$ parameterized smooth map
$F : \RR \times \RR^N \times \mathrm{Sym}(N) \to \RR \times \RR^N$:
\[
  F(\lambda, w; S) \;=\;
  \begin{pmatrix}
    (A + S - \lambda I)\, w \\[2pt]
    \tfrac{1}{2}\bigl(\|w\|^2 - 1\bigr)
  \end{pmatrix}, \text{ where } \lambda \in \R, w \in \R^{N}.
\]
Recall $\lambda_*$ and $w_*$ are the eigenvalue and eigenvectors of $A$, when $S=0$, we have $F(\lambda_*, w_*, 0) = 0$. We denote $(a,v) \in \RR \times \RR^{N}$ as the local perturbation at $(\lambda_*, w_*)$. The Jacobian of $F(\lambda,w)$ at this point with can be written as
\[
  L(a, v) \;=\;
  \begin{pmatrix}
   -w_* & A-\lambda_*I \\
   0 & w_*
  \end{pmatrix} 
  \begin{pmatrix}
      a \\ v
  \end{pmatrix}.
\]
We claim $L(a, v)=0$ only has trivial solution. If $L(a, v) = 0$, then $v \perp w_*$ and
$a w_* = (A - \lambda_* I)\, v$. Taking the inner product with $w_*$ to the second equation and
using symmetry,
\[
  a \;=\; w_*^T (A - \lambda_* I)\, v
  \;=\; \bigl((A - \lambda_* I)\, w_*\bigr)^T v
  \;=\; 0.
\]
Hence $(A - \lambda_* I)\, v = 0$, so $v \in \ker(A - \lambda_* I) = \mathrm{span}(w_*)$.
Since $\lambda_*$ is simple and there is no repetitive eigenvalue, combined with $v \perp w_*$, this forces
$v = 0$. Then the Jacobian matrix $\begin{pmatrix}
   -w_* & A-\lambda_*I \\
   0 & w_*
  \end{pmatrix}$ is invertible. 
  
Since $F$ is $C^{\infty}$, by the implicit
function theorem there exists a neighborhood $U \subset \mathrm{Sym}(N)$ around 
$0$ and a unique $C^\infty$  map $D_{(\lambda,w)}F: S \mapsto (\lambda(S), w(S))$ on $U$ with $D_{(\lambda,w)}F(0)=(\lambda_*, w_*)$ and
$F(\lambda(S), w(S); S) = 0$.

\medskip
\noindent\textbf{Local differentiability extension.} Now, we want to extend the local differentiability from the neighborhood $U \in \text{Sym}(N)$ around $0$ to all perturbations $B$ that satisfy $\norm{B}_{op} <\delta/4$. 

Parametrize the segment $t \mapsto A + t B$ for $t \in [0, 1]$. By Weyl's
inequality, $|\lambda_k(A + t B) - \lambda_k(A)| \le t\, \|B\|_{op}$ for
every $k$, so
\[
  \delta(t) \;:=\; \lambda_1(A + t B) - \lambda_2(A + t B)
  \;\ge\; \delta - 2 t\, \|B\|_{op}
  \;\ge\; \delta - 2\, \|B\|_{op}
  \;>\; 0
\]
for all $t \in [0, 1]$. The top eigenvalue of $A + t B$ stays simple along
the segment, so the IFT applies at each $t \in [0, 1]$. A connectedness argument (the set of $t \in [0, 1]$ for which a smooth branch ($(\lambda(t),w(t))$)
exists is open, closed, and nonempty) can promote our previous local statement from the IFT to `global' one along the line segment, this yields a $C^\infty$ function
\[
  f : [0, 1] \to \RR, \qquad f(t) \;:=\; \lambda_1(A + t B),
\]
together with a $C^\infty$ unit eigenvector path $w : [0, 1] \to \RR^N$
satisfying $(A + t B)\, w(t) = f(t)\, w(t)$ and $\|w(t)\| = 1$.

\medskip
\noindent\textbf{First derivative derivation.} Now with the differentiability along the line segment,
Differentiating the following perturbed eigenvalue problem with respect to $t$
$$(A + t B)\, w(t) = f(t)\, w(t),$$
this yields
\begin{equation}
  B\, w(t) + (A + t B - f(t) I)\, w'(t) \;=\; f'(t)\, w(t).
  \label{eq:diff_eigen}
\end{equation}
Taking the inner product with $w(t)$ and using both the symmetry of
$A + t B$ and the eigenvector identity
$(A + t B - f(t) I)\, w(t) = 0$,
\begin{equation}
  f'(t) \;=\; w(t)^T B\, w(t).
  \label{eq:fprime}
\end{equation}
At $t = 0$ this gives $f'(0) = w_*^T B\, w_*$. Substituting
$w_* = \tfrac{1}{\sqrt 2}\binom{u_1}{v_1}$ and the block form of $B$,
\[
  w_*^T B\, w_*
  \;=\; \tfrac{1}{2}\,(u_1^T, v_1^T)
        \begin{pmatrix} 0 & E \\ E^T & 0 \end{pmatrix}
        \begin{pmatrix} u_1 \\ v_1 \end{pmatrix}
  \;=\; \tfrac{1}{2}\bigl(u_1^T E v_1 + v_1^T E^T u_1\bigr)
  \;=\; u_1^T E v_1,
\]
which is the claimed first order coefficient
$\sum_{i, j} (u_1)_i (v_1)_j\, E_{ij}$.

\medskip
\noindent\textbf{Second derivative bound.}
Now we want to bound the second-order derivative of $f(t)$. Differentiating \cref{eq:fprime} and using the symmetry of $B$ we have
\begin{align}
  f''(t) &= 2\, w(t)^T B\, w'(t) \\
  &\le  2 \norm{B}_{op} \norm{ w'(t)}.
  \label{eq:fpp}
\end{align}
The second inequality follows by Cauchy and $\norm{w(t)}=1$.
Now we want to bound the term $w'(t)$.

Differentiating $\|w(t)\|^2 = 1$ gives $w(t)^T w'(t) = 0$, so
$w'(t) \in w(t)^\perp$. Rearranging \cref{eq:diff_eigen} yields
\[
  (A + t B - f(t) I)\, w'(t) \;=\; \bigl(f'(t) I - B\bigr)\, w(t).
\]
The right hand side lies in $w(t)^\perp$, since
$w(t)^T (f'(t) I - B)\, w(t) = f'(t) - w(t)^T B\, w(t) = 0$ by
\cref{eq:fprime}. Since $\delta(t) >0$, $A + t B - f(t) I$ restricted to the orthogonal complement $w(t)^{\perp}$
is invertible, and its inverse has operator norm bounded by $1 / \delta(t)$. Since $w'(t) \in w(t)^{\perp}$, we have
\begin{align*}
  \|w'(t)\| &\le \norm{(A+tB-f(t)I)^{-1}}_{op} \|(f'(t) I - B)\, w(t)\|\\
  &\le \frac{\|(f'(t) I - B)\, w(t)\|}{\delta(t)} \\
  &= \frac{\bigl\|B\, w(t) - \bigl(w(t)^T B\, w(t)\bigr) w(t)\bigr\|}{\delta(t)} \\
  &\le \frac{\|B\, w(t)\|}{\delta(t)} \\
  &\le \frac{\|B\|_{op}}{\delta - 2\, \|B\|_{op}},  
\end{align*}

where the third inequality follows from from the orthogonal decomposition of $Bw(t)$ along $w(t)$ and $w(t)^{\perp}$, which gives $\norm{Bw -(w^{T}Bw)w}^2 =\norm{Bw}^2 - |w^{T}Bw|^2 \le \norm{Bw}^2$. When $\norm{B}_{op} \le \delta/4$, combining with
\cref{eq:fpp},
\begin{equation}
  |f''(t)|
  \;\le\; 2\, \|B\|_{op}\, \|w'(t)\|
  \;\le\; \frac{2\, \|B\|_{op}^2}{\delta - 2\, \|B\|_{op}} 
  \;\le\; \frac{4 \norm{B}_{op}^2}{\delta}
  \qquad \text{for all } t \in [0, 1].
  \label{eq:fpp_bound}
\end{equation}

\medskip
\noindent\textbf{Taylor with remainder.} 
By Taylor's theorem in one real variable applied to $f$ on $[0, 1]$, there
exists $\xi \in (0, 1)$ such that
\[
  f(1) \;=\; f(0) + f'(0) + \tfrac{1}{2}\, f''(\xi).
\]
Set $\theta(E) := \tfrac{1}{2}\, f''(\xi)$. The bound \cref{eq:fpp_bound}
gives $|\theta(E)| \le \frac{4\norm{B}_{op}^2}{\delta}$.
Translating back via $f(1) = \lambda_1(A + B) = \sigma_1(G + E)$,
$f(0) = \sigma_1(G)$, $f'(0) = u_1^T E v_1$, and
$\|B\|_{op} = \|E\|_{op}$ delivers \cref{eq:taylor}.
\end{proof}

\newpage
\section{Proofs for \Cref{sec:operator}}
\label{app:sec4_bayes}

Throughout this section, we use $\phi_{\sigma_x}$ to denote the density of
$\gN(0,\sigma_x^2)$, $h$ as the density of $H$, and $p_Y$ as the marginal
density of $Y=X+E$ (Gaussian or heavy-tailed component as required).

 
 
 
 
 

\subsection{Proof of \cref{thm:main}}
\thmmain*

We separate the detailed proof of \cref{thm:main} into two parts:
\paragraph{Proof of Part~(i): the Bayes-optimal estimator}
\label{ssec:part_i} \
 
The proof of part~(i) combines two short propositions with a posterior collapse lemma. Here, we let $Z \in \{\gN, H\}$ denote the latent mixture indicator, with $\Pr(Z = \gN) = 1 - \alpha$. In \cref{prop:mixture}, we expand the conditional expectation using the latent indicator variable, and in \cref{prop:gaussian_posterior}, we present the result of the conditional expectation using a conjugate prior. 
 
\begin{proposition}[Mixture decomposition of the posterior mean]
\label{prop:mixture}
Conditioning on the latent indicator $Z$ and applying the tower property, we have $
  \E[x \mid y] \;=\; \pi(y)\,\E[x \mid y, Z=\gN] \;+\; (1-\pi(y))\,\E[x \mid y, Z=H].
$
\end{proposition}

\begin{proposition}[Conjugate Gaussian component]
\label{prop:gaussian_posterior}
For $x \sim \gN(0, \sigma_x^2)$ and $e \mid Z{=}\gN \sim \gN(0, \sigma^2)$ independent of $x$, the standard Gaussian conjugate calculation gives $\E[x \mid y, Z=\gN] = \beta\,y$ with $\beta = \sigma_x^2/(\sigma_x^2 + \sigma^2)$.
\end{proposition}
  
\begin{restatable}[Posterior collapse of the heavy tailed component]{lemma}{lemposterior}
\label{lem:posterior_collapse}
Let $x \sim \gN(0, \sigma_x^2)$ and let $e \sim H$ be independent of $x$, where the density $h$ is strictly positive and lies in $C^2(\RR)$. Suppose there exist constants $C_1, C_2, T_0 > 0$ such that for all $|t| \ge T_0$,
\begin{equation}
  \Bigl|\tfrac{d}{dt}\log h(t)\Bigr| \;\le\; \tfrac{C_1}{|t|},
  \qquad
  \Bigl|\tfrac{d^2}{dt^2}\log h(t)\Bigr| \;\le\; \tfrac{C_2}{t^2}.
\label{eq:score_curvature}
\end{equation}
Then there exist explicit constants $\widetilde C = 3 C_1$ and $y_1 = y_1(C_1, C_2, T_0, h, \sigma_x) < \infty$ such that
$
  \bigl|\E[x \mid y, Z = H]\bigr| \;\le\; \frac{\widetilde C\,\sigma_x^2}{|y|}
  \  \text{for all } |y| \ge y_1.
$
In particular, for any $\varepsilon > 0$ and $|y| \ge \max(y_1, \widetilde C \sigma_x^2/\varepsilon)$, $|\E[x \mid y, Z = H]| \le \varepsilon$.
\end{restatable}
 
We want to note that the conditions in \cref{eq:score_curvature} from the above lemma are mild. The Student $t_\nu$ family for any $\nu > 0$ (including the Cauchy case $\nu = 1$) satisfies them with $C_1 = C_2 = \nu + 1$ and $T_0 = \sqrt\nu$, and symmetric stable laws satisfy them analogously~\citep{polson2010shrink}. The intuition for the above lemma is that for observation $|y| \gg \sigma_x$, the heavy tailed likelihood $h(y - x)$ is nearly flat over the effective support of the prior, so the posterior reverts toward $\gN(0, \sigma_x^2)$ and concentrates near the origin. This is the classical outlier rejection phenomenon for heavy tailed likelihoods studied by~\citet{o1979outlier} and~\citet{pericchi1992exact}. For completeness, we provide complete proof of \Cref{lem:posterior_collapse} in \Cref{app:posterior_collapse_proof}.

\paragraph{Complete Proof of Part~(i).}
Combining \Cref{prop:mixture}, \Cref{prop:gaussian_posterior}, and \Cref{lem:posterior_collapse}, for $|y| \ge y^\star := \max\!\bigl(y_1,\, \widetilde C \sigma_x^2/\varepsilon\bigr)$ we obtain
\(
  \E[x \mid y] = \pi(y)\,\beta\,y + \rho(y),
\)
where $\rho(y) = (1 - \pi(y))\,\E[x \mid y, Z = H]$. Since $1 - \pi(y) \in [0, 1]$, the lemma gives $|\rho(y)| \le |\E[x \mid y, Z = H]| \le \varepsilon$. The bound is uniform in the contamination parameters $(\alpha, \sigma)$. \qed
 
\paragraph{Proof of Part~(ii): the Unique Closed Form Surrogate}
\label{ssec:part_ii}

By Bayes' rule, the exact retention probability $\pi(y) = \Pr[e\sim \gN | y]$ is
\begin{align}
  \pi(y)
  &= \frac{\Pr[y|e\sim \gN] \cdot\Pr[e \sim \gN ]}{\Pr[y|e\sim \gN] \cdot\Pr[e \sim \gN] + \Pr[y|e\sim H] \cdot\Pr[e \sim H]} \\ 
  &=\Bigl[\,1 + \tfrac{\alpha}{1 - \alpha} \cdot
        \frac{f_{Y \mid Z = H}(y)}{f_{Y \mid Z = \gN}(y)}\,\Bigr]^{-1},
\label{eq:retention_bayes}
\end{align}
with $f_{Y \mid Z = \gN} = \gN(0, \sigma_x^2 + \sigma^2)$ and
$f_{Y \mid Z = H} = \phi_{\sigma_x} \ast h$. The right hand side admits
no closed form, depends on $(\alpha, \sigma, h)$, and decays faster than
any exponential in $|y|$ because $f_{Y \mid Z = \gN}$ is Gaussian while
$f_{Y \mid Z = H}$ has heavy tails. We therefore replace $\pi$ by a
closed form surrogate $\widehat\pi$ obeying the three axioms stated in
\Cref{thm:main}.

\paragraph{Justification of the axioms.}
\emph{Redescending.} The surrogate $\widehat\pi$ must satisfy
$\widehat\pi(y) \to 0$ as $|y| \to \infty$, mirroring $\pi$ itself: in
\cref{eq:retention_bayes}, $f_{Y \mid Z = H}(y) / f_{Y \mid Z = \gN}(y) \to \infty$, so $\pi(y) \to 0$.

\emph{Multiplicative accumulation.} We require that evidence of
contamination accumulate at a constant rate per unit of magnitude,
independently of the magnitude already accumulated. Equivalently,
$\widehat\pi(y_1 + y_2) = \widehat\pi(y_1)\,\widehat\pi(y_2)$.

\emph{Unit at origin.} An observation of zero magnitude carries no
evidence of contamination and is treated as Gaussian, i.e.,
$\widehat\pi(0) = 1$.

\paragraph{Complete Proof of Part~(ii).}
The multiplicative axiom is the Cauchy exponential functional equation on
$[0, \infty)$. Under continuity, its only solutions $\widehat\pi : [0, \infty) \to (0, 1]$
are the exponentials $\widehat\pi(y) = e^{-c y}$ for some $c \ge 0$. The redescending axiom forces
$c > 0$. Writing $c = 1 / \tau$ for $\tau > 0$ and extending evenly to
$\RR$ via $y \mapsto |y|$ yields $\widehat\pi(y) = e^{-|y| / \tau}$.
Substituting into \cref{eq:thm_factor} gives the smooth shrinkage operator
$\varphi_\tau(y) = \beta\,e^{-|y| / \tau}\,y$, which completes the proof of
\Cref{thm:main}. \qed

\subsection{Proof of \Cref{lem:posterior_collapse}}
\label{app:posterior_collapse_proof}
 
\lemposterior*
\begin{proof}
Here we give a non-asymptotic proof for the posterior collapse, under the assumption of score and curvature condition in 
\cref{eq:score_curvature}.  We write $X\sim\gN(0,\sigma_x^2)$, $E\sim H$
independent of $X$, $Y=X+E$, marginal
$p_Y(y)=\int\phi_{\sigma_x}(x)h(y-x)\,dx$, and posterior mean
$m(y):=\mathbb E[X\mid Y=y]$.  We set $\psi(u):=-(\log h)'(u)=-h'(u)/h(u)$;
by \cref{eq:score_curvature}, $|\psi(u)|\le C_1/|u|$ for $|u|\ge T_0$.
 
\paragraph{Tweedie-type identity.} Since $m(y)=E[X|Y=y]= \int x \frac{p_{X,Y}(x,y)}{p_{Y}(y)}dx$, and $p_{X,Y}(x,y)= p_{X}(x)p_{Y|X}(y|x)$ we have
\begin{align}
  m(y)\,p_Y(y)
  &=\int x\,\phi_{\sigma_x}(x)h(y-x)\,dx \notag \\
   &\stackrel{(a)}{=}-\sigma_x^2\!\int\phi_{\sigma_x}'(x)h(y-x)\,dx \notag \\
  &\stackrel{(b)}{=}-\sigma_x^2\!\int\phi_{\sigma_x}(x)\,h'(y-x)\cdot(-1)\,dx \notag \\
   &\stackrel{(c)}{=}\sigma_x^2\!\int\phi_{\sigma_x}(x)\,\psi(y-x)\,h(y-x)\,dx, \notag \\
   &\stackrel{(d)}{=}\sigma_x^2\,\mathbb E[\psi(Y-X)\mid Y=y] p_{Y}(y)  \label{eq:tweedie}.
\end{align}
where $(a)$ uses $x\,\phi_{\sigma_x}(x)=-\sigma_x^2\,\phi_{\sigma_x}'(x)$, $(b)$ uses integration
by parts (boundary terms vanish since $\phi_{\sigma_x}\to 0$
super-exponentially and $h$ is bounded),
, $(c)$ uses $h'(u)=-\psi(u)h(u)$, and $(d)$ uses the definition of conditional expectation of $\psi(Y-X)$. 
 
\paragraph{Global bound on $\psi$.}
We claim $\|\psi\|_\infty\le M:=\max\!\bigl(\sup_{|u|\le T_0}|\psi(u)|,\,
(C_1+C_2)/T_0\bigr)<\infty$.  On $[-T_0,T_0]$, $\psi$ is continuous (since
$h\in C^1$ and $h>0$) hence bounded.  For $|u|\ge T_0$,
$|\psi'(u)|=|(\log h)''(u)|\le C_2/u^2$ by~\cref{eq:score_curvature}, so
\(
|\psi(u)|\le|\psi(T_0)|+\!\int_{T_0}^{|u|}\!|\psi'(s)|\,ds
\le C_1/T_0+\!\int_{T_0}^\infty\!C_2/s^2\,ds=(C_1+C_2)/T_0.
\)
 
\paragraph{Bound $m(y)$ via event splitting.} We create two event $A$ and $A^c$ that splits $Y-X$ into a good and bad region. In particular,
fix $|y|\ge 2T_0$ and split on $A:=\{|Y-X|>|y|/2\}$.  On $A$, $|Y-X|\ge T_0$,
and the score bound gives $|\psi(Y-X)|\le C_1/|Y-X|\le 2C_1/|y|$.  On $A^c$, $|Y-X|\le T_0$, the global bound gives $|\psi(Y-X)|\le M$.  Using the Tweedie-type identity \cref{eq:tweedie} we have
\begin{equation}
  |m(y)|
  \le\sigma_x^2\!\left[\tfrac{2C_1}{|y|}\,\mathbb P(A\mid Y=y)+M\,\mathbb P(A^c\mid Y=y)\right]
  \le\frac{2C_1\sigma_x^2}{|y|}+M\sigma_x^2\,\mathbb P(A^c\mid Y=y).
\label{eq:master}
\end{equation}
 
\paragraph{Bound $\mathbb P(A^c\mid Y=y)$.}
WLOG, we assume $y>0$ (the case $y<0$ is symmetric).  Then $A^c=\{X\in[y/2,3y/2]\}$
and
\(
\mathbb P(A^c\mid Y=y)=\!\int_{y/2}^{3y/2}\phi_{\sigma_x}(x)h(y-x)\,dx\,/\,p_Y(y).
\) We want to show $\mathbb P(A^c\mid Y=y)$ will become expoentially small when $y$ closes to $x$. In particular,
We bound the numerator and denominator using the power-law envelopes from
\Cref{lem:tail_envelope} below.
 
\emph{Numerator.}
On $A^c$, let $u:=y-x$, we know $|u|\le y/2$ and $\phi_{\sigma_x}$ is maximized at $x=y/2$:
$\phi_{\sigma_x}(x)\le\frac{1}{\sqrt{2\pi}\sigma_x}e^{-y^2/(8\sigma_x^2)}$.
Setting $H_0:=\sup_{|u|\le T_0}h(u)<\infty$ and using \Cref{lem:tail_envelope}
on $T_0\le|u|\le y/2$,
$\sup_{|u|\le 2|y|}h(u)\le\max(H_0,h(T_0)(2|y|/T_0)^{C_1})=:K_1(|y|)$, hence
\begin{equation}
  \int_{y/2}^{3y/2}\phi_{\sigma_x}(x)h(y-x)\,dx
  \;\le\;y\cdot\tfrac{e^{-y^2/(8\sigma_x^2)}}{\sqrt{2\pi}\sigma_x}\cdot K_1(y).
\label{eq:numerator_ub}
\end{equation}
 
\emph{Denominator.}
To derive a lower bound for $p_{Y}$ on $\R$, we first restrict $p_Y$ to $|x|\le 1$.  Then
$\phi_{\sigma_x}(x)\ge\frac{1}{\sqrt{2\pi}\sigma_x}e^{-1/(2\sigma_x^2)}$.
For $|y|\ge T_0+1$, \Cref{lem:tail_envelope} gives the lower envelope
$h(y-x)\ge h(T_0)\,(T_0/(|y|+1))^{C_1}$, hence
\begin{equation}
  p_Y(y)\;\ge\;2\cdot\tfrac{e^{-1/(2\sigma_x^2)}}{\sqrt{2\pi}\sigma_x}
  \cdot h(T_0)\,(T_0/(|y|+1))^{C_1}.
\label{eq:denominator_lb}
\end{equation}
 
\emph{Ratio.}
Dividing \cref{eq:numerator_ub} by \cref{eq:denominator_lb}, the
$1/(\sqrt{2\pi}\sigma_x)$ factors cancel:
\begin{equation}
  \mathbb P(A^c\mid Y=y)\;\le\;
  \underbrace{\tfrac{|y|}{2}\cdot\tfrac{K_1(|y|)}{h(T_0)}\cdot
              \bigl(\tfrac{|y|+1}{T_0}\bigr)^{C_1}\cdot
              e^{-y^2/(8\sigma_x^2)+1/(2\sigma_x^2)}}_{=:R(y)}.
\label{eq:R_def}
\end{equation}
$R(y)$ is the product of an at-most-polynomial factor (degree
$C_1+\max(1,C_1)$ in $|y|$) and the Gaussian factor $e^{-y^2/(8\sigma_x^2)}$,
hence decays super-exponentially as $|y|\to\infty$.
 
\paragraph{Put every together.}
Substituting \cref{eq:R_def} into \cref{eq:master},
$|m(y)|\le 2C_1\sigma_x^2/|y|+M\sigma_x^2\,R(y)$.  Define $y_1\ge 2T_0$ as
the smallest value such that $M\,R(y)\le C_1/|y|$ for all $|y|\ge y_1$.  For
$|y|\ge y_1$,
\(
|m(y)|\le 2C_1\sigma_x^2/|y|+\sigma_x^2\,C_1/|y|=3C_1\sigma_x^2/|y|,
\)
proving \cref{lem:posterior_collapse} with $\widetilde C=3C_1$.
The $\varepsilon$-form follows by taking
$|y|\ge\max(y_1,\widetilde C\sigma_x^2/\varepsilon)$.
\end{proof}

\medskip
 
\begin{lemma}[Power-law envelope from score control]
\label{lem:tail_envelope}
Under the score condition $|(\log h)'(t)|\le C_1/|t|$ on $|t|\ge T_0$, for
every $|t|\ge T_0$,
\(
h(T_0)\,(T_0/|t|)^{C_1}\le h(t)\le h(T_0)\,(|t|/T_0)^{C_1}.
\)
\end{lemma}
 
\begin{proof}
Since $h(s)$ is symmetric, Without loss of generality, for $t\ge T_0$, $\log h(t)-\log h(T_0)=\!\int_{T_0}^{t}(\log h)'(s)\,ds$, so
$|\log h(t)-\log h(T_0)|\le C_1\!\int_{T_0}^{t}\!ds/s=C_1\log(t/T_0)$.
Taking the exponentiation on both sides yields the two-sided bound.
\end{proof}
\newpage

\section{Proofs for \texorpdfstring{\Cref{sec:convergence}}{Section 5}}
\label{app:convergence-proofs}

This section proves the four theorems in \Cref{sec:convergence}.  We start the convergence proof with an entry-wise property on the gradient matrix that can be proved under a known Cauchy distribution.  The same scalar property are used twice: first for post-clipping of one noisy gradient, and then for pre-clipped Muon.  Throughout the appendix, $\E_k[\cdot]$ denotes conditional expectation over minibatch $\xi_k$ given filtration$\gF_{k-1}=\sigma(\xi_1,\dots,\xi_{k-1})$. $G \in \R^{m \times n}$ is the true gradient matrix and $g := (G)_{ij}$ is the gradient entry inside the true gradient matrix $G$.

\paragraph{The relative-bias principle for matrix scalars}
\label{app:relative-principle}

For an entry-wise pre-processing map $\varphi$, define the scalar processed mean
\begin{equation}
    m_\varphi(g):=\E_\xi\varphi(g+\xi),
    \qquad
    \xi\sim (1-\alpha)N(0,\sigma^2)+\alpha\,\mathrm{Cauchy}(0,\gamma).
    \label{eq:scalar-processed-mean-app}
\end{equation}
For a matrix $G$, $m_\varphi(G)$ means that \cref{eq:scalar-processed-mean-app} is applied entry by entry. We isolate the central scalar estimate as a named property that all four convergence theorems will rely on.

\begin{definition}[Relative-Bias Property]
\label{def:relative-bias-app}
A coordinate-wise map $\varphi:\R\to\R$ satisfies the \emph{relative-bias property} with constant $\rho\in[0,1)$ on $[-B,B]$ under the noise model \cref{eq:scalar-processed-mean-app} if, for every $|g|\le B$,
\begin{equation}
    |m_\varphi(g)-g|\le \rho |g|.
    \label{eq:relative-bias-app}
\end{equation}
\end{definition}
\cref{eq:relative-bias-app} is not an unbiasedness statement: it says that the processed conditional mean is a \emph{multiplicative} perturbation of the clean signal. This is a stronger result than a uniform additive-bias bound for optimization, because the bias vanishes whenever the true gradient entry is zero. In general,
\[
    m_\varphi(g)=\E\varphi(g+\xi)
    \neq \varphi(g)
    \qquad\text{and}\qquad
    m_\varphi(g)\neq g.
\]
Lemmas~\ref{lem:hard-relative-bias-app} and~\ref{lem:smooth-relative-bias-app} establish this relative-bias property for $C_\tau$ and $S_c$ respectively, with explicit constants $\rho_\tau$ and $\rho_c$ that depend on the threshold and the noise scales. Next, we show how we use the relative bias alignment result to prove the convergence result for both post-clipping and pre-clipping.

\paragraph{Relative bias alignment in \emph{post-clipping}}
For the post-clipped update \cref{eq:post-update-main}, define
\begin{equation}
    Y_k^\varphi:=\varphi(G_k+\Xi_k),
    \qquad
    M_{k,\mathrm{post}}^\varphi:=\E_kY_k^\varphi=m_\varphi(G_k).
    \label{eq:post-mean-app}
\end{equation}
If \cref{eq:relative-bias-app} holds entrywise for all $k$, then
\begin{equation}
    \norm{M_{k,\mathrm{post}}^\varphi-G_k}_F\le \rho\norm{G_k}_F.
    \label{eq:post-frob-rel-app}
\end{equation}
Consequently,
\begin{equation}
\begin{aligned}
    \inner{G_k}{M_{k,\mathrm{post}}^\varphi}
    &=\norm{G_k}_F^2+\inner{G_k}{M_{k,\mathrm{post}}^\varphi-G_k} \\
    &\ge \norm{G_k}_F^2-\norm{G_k}_F\norm{M_{k,\mathrm{post}}^\varphi-G_k}_F \\
    &\ge (1-\rho)\norm{G_k}_F^2.
\end{aligned}
\label{eq:post-alignment-principle-app}
\end{equation}
This is the post-clipping use of the relative-bias principle: it converts a biased processed update into a direction that remains positively aligned with the true gradient, provided $\rho<1$.

\paragraph{Relative bias absorption in \emph{pre-clipping}}
For pre-clipping, the processed mini-batch mean is
\begin{equation}
    \bar G_k^\varphi=\frac1N\sum_{\ell=1}^N\varphi(G_k+\Xi_k^{(\ell)}),
    \qquad
    M_{k,\mathrm{pre}}^\varphi:=\E_k\bar G_k^\varphi=m_\varphi(G_k).
    \label{eq:pre-mean-app}
\end{equation}
The exact decomposition is
\begin{equation}
    \bar G_k^\varphi-G_k
    =\underbrace{(\bar G_k^\varphi-M_{k,\mathrm{pre}}^\varphi)}_{\text{zero-mean processed sampling fluctuation}}
    +\underbrace{(M_{k,\mathrm{pre}}^\varphi-G_k)}_{\text{deterministic conditional bias}}.
    \label{eq:correct-decomp-app}
\end{equation}
The first term is controlled by a variance estimate for $\varphi(g+\xi)$ and by independence across the sample index.  The second term is deterministic conditional on $X_k$.  If \cref{eq:relative-bias-app} holds for all $k$,
\begin{equation}
    \norm{M_{k,\mathrm{pre}}^\varphi-G_k}_*
    \le \sqrt r\norm{M_{k,\mathrm{pre}}^\varphi-G_k}_F
    \le \sqrt r\rho\norm{G_k}_F
    \le \sqrt r\rho\norm{G_k}_*.
    \label{eq:pre-bias-principle-app}
\end{equation}
When this is inserted into the matrix-sign descent lemma, it appears on the same side as the stationarity measure $\norm{G_k}_*$ and can be moved to the left as long as $2\sqrt r\rho<1$.  This is the pre-clipped Muon use of the same principle.  The difference between \cref{eq:post-alignment-principle-app} and \cref{eq:pre-bias-principle-app} is purely algorithmic: post-clipping needs alignment of the processed update with $G_k$, while Muon needs a nuclear-norm control of the error in the matrix-sign direction. 

\paragraph{Organization.} We first present preliminary Lemmas in \cref{app:scalar-inequalities}. In \cref{app:hard-scalar} and \cref{app:smooth-scalar}, we prove both hard clipping operator $C_{\tau}(x)$ and smooth shrinkage $S_{c}(x)$ satisfies the Relative-Bias property in \cref{def:relative-bias-app}. In \cref{app:post-proofs}, we apply this property to the descent lemma given by $L$-smoothness and prove the convergence result of \cref{thm:post-hard-main} and \cref{thm:post-smooth-main}. Similarly, we use the same trick to prove \cref{thm:pre-hard-main} and \cref{thm:pre-smooth-main} in \cref{app:pre-proofs}.

\subsection{Elementary Scalar Inequalities}
\label{app:scalar-inequalities}

\begin{lemma}[Gaussian and Cauchy tails]
\label{lem:scalar-tails-app}
Let $Z\sim N(0,\sigma^2)$ and $H\sim\mathrm{Cauchy}(0,\gamma)$.  For $t\ge0$,
\begin{equation}
    \mathbb P(|Z|\ge t)\le 2\exp\left(-\frac{t^2}{2\sigma^2}\right),
    \label{eq:gauss-tail-app}
\end{equation}
with the convention that this probability is zero when $\sigma=0<t$.  For $a>0$,
\begin{equation}
    \mathbb P(H>a)=\mathbb P(H<-a)
    =\frac1\pi\arctan\left(\frac\gamma a\right)
    \le \frac{\gamma}{\pi a}.
    \label{eq:cauchy-tail-app}
\end{equation}
Moreover,
\begin{equation}
    \E\min\{H^2,a^2\}\le \frac{4\gamma a}{\pi},
    \label{eq:cauchy-trunc-two-sided-app}
\end{equation}
and the one-sided version obeys
\begin{equation}
    \int_0^\infty \min\{h^2,a^2\}\frac{\gamma}{\pi(\gamma^2+h^2)}\,dh
    \le \frac{2\gamma a}{\pi}.
    \label{eq:cauchy-trunc-one-sided-app}
\end{equation}
\end{lemma}

\begin{proof}
The Gaussian bound is the standard Chernoff tail bound.  For the Cauchy tail, integrate the density:
\[
    \mathbb P(H>a)=\int_a^\infty \frac{\gamma}{\pi(\gamma^2+h^2)}\,dh
    =\frac1\pi\arctan\left(\frac\gamma a\right)
    \le \frac{\gamma}{\pi a},
\]
where the last step uses $\arctan x\le x$ for $x\ge0$.  For the one-sided truncation bound, split the integral at $a$:
\[
\int_0^a h^2\frac{\gamma}{\pi(\gamma^2+h^2)}\,dh
\le \int_0^a \frac{\gamma}{\pi}\,dh
=\frac{\gamma a}{\pi},
\]
and
\[
\int_a^\infty a^2\frac{\gamma}{\pi(\gamma^2+h^2)}\,dh
\le a^2\int_a^\infty \frac{\gamma}{\pi h^2}\,dh
=\frac{\gamma a}{\pi}.
\]
Adding these two estimates proves \cref{eq:cauchy-trunc-one-sided-app}; symmetry gives \cref{eq:cauchy-trunc-two-sided-app}.
\end{proof}

\begin{lemma}[A logarithmic threshold device]
\label{lem:log-device-app}
Let $a\ge0$ and $s\ge1$.  If
\[
    x\ge 2sa\log(e+2sa),
\]
then
\[
    \frac{a\log(e+x)}{x}\le \frac1s.
\]
\end{lemma}

\begin{proof}
If $a=0$, the claim is immediate.  Suppose $a>0$ and set $A:=2sa$.  The function $x/\log(e+x)$ is increasing on $(0,\infty)$ because
\[
    \frac{d}{dx}\frac{x}{\log(e+x)}
    =\frac{\log(e+x)-x/(e+x)}{\log^2(e+x)}>0.
\]
Therefore it is enough to verify the claim at $x_0=A\log(e+A)$.  Since
\[
    e+x_0=e+A\log(e+A)\le (e+A)^2,
\]
we have $\log(e+x_0)\le 2\log(e+A)$.  Hence
\[
    x_0=A\log(e+A)\ge \frac A2\log(e+x_0)=sa\log(e+x_0),
\]
which is the desired inequality at $x_0$.
\end{proof}

\subsection{Proofs of Scalar Relative bias and Variance under $C_{\tau}$ }
\label{app:hard-scalar}

Recall $C_\tau(x)=\operatorname{sign}(x)\min\{|x|,\tau\}$.

\begin{lemma}[Relative bias of hard clipping]
\label{lem:hard-relative-bias-app}
Let $|g|\le B<\tau$, and let $\xi_{0} \sim (1-\alpha) N(0,\sigma_0^2) + \alpha \mathrm{Cauchy}(0,\gamma_0)$ with $0\le \sigma_0 \le \sigma$ and $0 < \gamma_0\le \gamma $.  Then
\begin{equation} 
    |\E C_\tau(g+\xi_0)-g|
    \le \rho_\tau |g|,
    \label{eq:hard-relative-bias-app}
\end{equation}
where
\begin{equation}
    \rho_\tau:=2(1-\alpha)\bar\Phi\left(\frac{\tau-B}{\sigma}\right)
    +\frac{\alpha\gamma}{\pi}\left(\frac1{\tau-B}+\frac1{\tau+B}\right),
    \label{eq:rho-hard-app}
\end{equation}
$\bar\Phi(t):=\mathbb P(N(0,1)\ge t)$, and the Gaussian term is interpreted as zero when $\sigma=0$.
\end{lemma}

\begin{proof}
We first prove the claim for a generic symmetric scalar noise $W$ with a continuous density.  Define
\[
    m_W(s):=\E C_\tau(s+W).
\]
Because $C_\tau$ is odd and $W$ is symmetric, $m_W(0)=0$.  The map $C_\tau$ is Lipschitz and is differentiable except at $\pm\tau$, with derivative $1\{|x|<\tau\}$ at differentiability points.  Differentiating under the integral gives
\[
    m_W'(s)=\mathbb P(|s+W|<\tau).
\]
Thus
\[
    m_W(g)-g
    =\int_0^g (m_W'(s)-1)\,ds,
\]
and hence
\[
    |m_W(g)-g|
    \le |g|\sup_{|s|\le B}\mathbb P(|s+W|\ge\tau).
\]

For $W=Z\sim N(0,\sigma_0^2)$, if $|s|\le B$ and $|s+Z|\ge\tau$, then $|Z|\ge\tau-B$, hence $\mathbb P(|s+Z|\ge\tau)\le \mathbb P(|Z|\ge\tau-B)$. When $\sigma_0>0$ this is at most $2\bar\Phi((\tau-B)/\sigma_0)$ by the standard Gaussian tail bound; when $\sigma_0=0$, $Z\equiv 0$ and the probability is exactly zero, since $\tau>B$. Combining the two cases under the convention that $\bar\Phi((\tau-B)/0):=0$ (consistent with \cref{eq:gauss-tail-app}) and using monotonicity of $\bar\Phi$,
\[
    |m_Z(g)-g|
    \le 2\bar\Phi\left(\frac{\tau-B}{\sigma_0}\right)|g|
    \le 2\bar\Phi\left(\frac{\tau-B}{\sigma}\right)|g|.
\]

For $W=H\sim\mathrm{Cauchy}(0,\gamma_0)$, \cref{eq:cauchy-tail-app} gives, uniformly over $|s|\le B$,
\[
\begin{aligned}
    \mathbb P(|s+H|\ge\tau)
    &=\mathbb P(H\ge\tau-s)+\mathbb P(H\le-\tau-s) \\
    &\le \frac{\gamma_0}{\pi}\left(\frac1{\tau-s}+\frac1{\tau+s}\right)
    \le \frac{\gamma_0}{\pi}\left(\frac1{\tau-B}+\frac1{\tau+B}\right) \le \frac{\gamma}{\pi}\left(\frac1{\tau-B}+\frac1{\tau+B}\right),
\end{aligned}
\]
where the supremum over $|s|\le B$ is achieved at the endpoints $s=\pm B$. Taking the mixture of the Gaussian and Cauchy bounds proves \cref{eq:hard-relative-bias-app}.
\end{proof}

\begin{lemma}[Variance of hard clipping]
\label{lem:hard-variance-app}
Let $|g|\le B<\tau$, let $\xi_{0} \sim (1-\alpha) N(0,\sigma_0^2) + \alpha \mathrm{Cauchy}(0,\gamma_0)$ with $0\le \sigma_0 \le \sigma$ and $0 < \gamma_0\le \gamma $.  Then
\begin{equation}
    \operatorname{Var}(C_\tau(g+\xi_0))
    \le v_\tau,
    \qquad
    v_\tau:=(1-\alpha)\sigma^2+\frac{4\alpha\gamma\tau}{\pi}.
    \label{eq:v-hard-app}
\end{equation}
\end{lemma}

\begin{proof}
For any random variable $Y$ and deterministic number $a$, $\operatorname{Var}(Y)\le\E(Y-a)^2$.  Here we choose the center $a=g$.  This is natural for hard clipping because $C_\tau(g)=g$ whenever $|g|\le B<\tau$.

For the Gaussian branch, $C_\tau$ is $1$-Lipschitz, so
\[
    |C_\tau(g+Z)-g|=|C_\tau(g+Z)-C_\tau(g)|\le |Z|.
\]
Thus the Gaussian contribution is at most $\E Z^2=\sigma_0^2 \le \sigma^2$.

For the Cauchy branch, use the positive and negative halves separately.  If $h\ge0$, then moving from $g$ to $g+h$ can increase the clipped value by at most $\tau-g$, so
\[
    |C_\tau(g+h)-g|=\min\{h,\tau-g\}.
\]
If $h\le0$, then moving from $g$ to $g+h$ can decrease the clipped value by at most $\tau+g$, so
\[
    |C_\tau(g+h)-g|=\min\{|h|,\tau+g\}.
\]
Applying the one-sided truncation bound \cref{eq:cauchy-trunc-one-sided-app} to the two halves gives
\[
    \E|C_\tau(g+H)-g|^2
    \le \frac{2\gamma_0(\tau-g)}{\pi}+\frac{2\gamma_0(\tau+g)}{\pi}
    =\frac{4\gamma_0\tau}{\pi} \le \frac{4\gamma\tau}{\pi}.
\]
Mixing the two branches proves the lemma.
\end{proof}

\subsection{Proofs of Scalar Relative bias and Variance under $S_{c}$ }
\label{app:smooth-scalar}

Recall $S_c(x)=xe^{-|x|/c}$.

\begin{lemma}[Elementary properties of smooth shrinkage]
\label{lem:smooth-properties-app}
The map $S_c$ is odd, $1$-Lipschitz, and satisfies $|S_c(x)|\le c/e$ for all $x$.  Its derivative is
\[
    S_c'(x)=e^{-|x|/c}\left(1-\frac{|x|}{c}\right),
\]
and $S_c'$ is $2/c$-Lipschitz.
\end{lemma}

\begin{proof}
Oddness is immediate.  For $x\ne0$, the displayed derivative follows by differentiating on the two half-lines; the same formula gives the derivative at $0$.  With $t=|x|/c$,
\[
    |S_c'(x)|=e^{-t}|1-t|\le1,
\]
so $S_c$ is $1$-Lipschitz.  Also
\[
    |S_c(x)|=|x|e^{-|x|/c}=c\,t e^{-t}\le c/e.
\]
Finally, write $S_c'(x)=h(|x|/c)$ with $h(t)=e^{-t}(1-t)$.  Since $|h'(t)|=e^{-t}|t-2|\le2$ for $t\ge0$,
\[
    |S_c'(x)-S_c'(y)|\le \frac2c\,||x|-|y||\le\frac2c|x-y|.
\]
\end{proof}

\begin{lemma}[Derivative deficits for smooth shrinkage]
\label{lem:smooth-derivative-deficits-app}
Let $Z\sim N(0,\sigma^2)$ and $H\sim\mathrm{Cauchy}(0,\gamma)$.  Define
\[
    \mu_c:=\E S_c'(Z),
    \qquad
    \kappa_c:=\E S_c'(H).
\]
Then
\begin{equation}
    1-\mu_c\le \sqrt{\frac8\pi}\frac{\sigma}{c},
    \qquad
    1-\kappa_c\le \frac{8\gamma}{\pi c}\log\left(e+\frac c\gamma\right).
    \label{eq:smooth-derivative-deficits-app}
\end{equation}
\end{lemma}

\begin{proof}
For $t\ge0$, the elementary inequality
\[
    1-e^{-t}(1-t)\le2t
\]
holds.  Therefore
\[
    1-\mu_c
    =\E\left[1-e^{-|Z|/c}\left(1-\frac{|Z|}{c}\right)\right]
    \le \frac2c\E|Z|
    =\sqrt{\frac8\pi}\frac{\sigma}{c}.
\]

For the Cauchy branch, symmetry gives
\[
    1-\kappa_c
    =\frac{2\gamma}{\pi}\int_0^\infty
    \frac{1-e^{-u/c}(1-u/c)}{\gamma^2+u^2}\,du.
\]
Split the integral into $[0,c]$ and $[c,\infty)$.  On $[0,c]$, use $1-e^{-u/c}(1-u/c)\le 2u/c$.  On $[c,\infty)$, use the crude bound $1-e^{-u/c}(1-u/c)\le2$.  Then
\[
\begin{aligned}
    1-\kappa_c
    &\le \frac{2\gamma}{\pi}\left[
        \frac2c\int_0^c\frac{u}{\gamma^2+u^2}\,du
        +2\int_c^\infty\frac{1}{\gamma^2+u^2}\,du
    \right] \\
    &\le \frac{2\gamma}{\pi c}\log\left(1+\frac{c^2}{\gamma^2}\right)
        +\frac{4\gamma}{\pi c}.
\end{aligned}
\]
Since $\log(1+x^2)\le2\log(e+x)$ and $1\le\log(e+x)$ for $x\ge0$, this is at most
$8\gamma(\pi c)^{-1}\log(e+c/\gamma)$.
\end{proof}

\begin{lemma}[Relative bias of smooth shrinkage]
\label{lem:smooth-relative-bias-app}
For $|g|\le B$, let $\xi_{0} \sim (1-\alpha) N(0,\sigma_0^2) + \alpha \mathrm{Cauchy}(0,\gamma_0)$ with $0\le \sigma_0 \le \sigma$ and $0 < \gamma_0\le \gamma $. Then
\begin{equation}
    |\E S_c(g+\xi_0)-g|
    \le \rho_c |g|,
    \label{eq:smooth-relative-bias-app}
\end{equation}
where
\begin{equation}
    \rho_c:=\frac{B+\sqrt{8/\pi}(1-\alpha)\sigma}{c}
    +\frac{8\alpha\gamma}{\pi c}\log\left(e+\frac c\gamma\right).
    \label{eq:rho-smooth-app}
\end{equation}
\end{lemma}

\begin{proof}
Because $S_c$ is odd and $\xi$ is symmetric, $\E S_c(\xi)=0$.  Since $S_c'$ is $2/c$-Lipschitz, Taylor's theorem with Lipschitz derivative gives, for every $x$ and $g$,
\[
    |S_c(x+g)-S_c(x)-gS_c'(x)|\le\frac{g^2}{c}.
\]
Set $x=\xi_0$ and take expectations.  Then
\[
\begin{aligned}
    |\E S_c(g+\xi_0)-g|
    &=|\E[S_c(g+\xi_0))-S_c(\xi_0)]-g| \\
    &\le |g|\,|\E S_c'(\xi_0)-1|+\frac{g^2}{c}.
\end{aligned}
\]
The derivative deficit of the mixture is
\[
    1-\E S_c'(\xi_0)
    =(1-\alpha)(1-\mu_c)+\alpha(1-\kappa_c).
\]
Using \Cref{lem:smooth-derivative-deficits-app}, $g^2/c\le (B/c)|g|$, and monotonicity in $\gamma$ and $\sigma$ proves the displayed bound. The term $B/c$ in $\rho_c$ arises from this Taylor remainder: unlike hard clipping, which exactly satisfies $C_\tau(g)=g$ for $|g|\le B<\tau$ and contributes no remainder, smooth shrinkage is only \emph{approximately} the identity near zero, with a quadratic deviation that the threshold $c$ must absorb.
\end{proof}

\begin{lemma}[Variance of smooth shrinkage]
\label{lem:smooth-variance-app}
For $|g|\le B$, let $\xi_{0} \sim (1-\alpha) N(0,\sigma_0^2) + \alpha \mathrm{Cauchy}(0,\gamma_0)$ with $0\le \sigma_0 \le \sigma$ and $0 < \gamma_0\le \gamma $. Then
\begin{equation}
    \operatorname{Var}(S_c(g+\xi_0))
    \le v_c,
    \qquad
    v_c:=(1-\alpha)\sigma^2+\frac{8\alpha\gamma c}{\pi e}.
    \label{eq:v-smooth-app}
\end{equation}
\end{lemma}

\begin{proof}
Use $\operatorname{Var}(Y)\le\E(Y-a)^2$ with the deterministic center $a=S_c(g)$.  By \Cref{lem:smooth-properties-app}, $S_c$ is $1$-Lipschitz and $|S_c|\le c/e$, so
\[
    |S_c(g+\xi_0)-S_c(g)|\le\min\{|\xi|,2c/e\}.
\]
The Gaussian contribution is at most $\E Z^2=\sigma^2$.  The Cauchy contribution is bounded by
\[
    \E\min\{H^2,(2c/e)^2\}
    \le \frac{8\gamma_0 c}{\pi e}
    \le \frac{8\gamma c}{\pi e}
\]
from \cref{eq:cauchy-trunc-two-sided-app}.  Mixing the two branches proves the variance bound.
\end{proof}

\subsection{Proofs of Clipping Threshold $\tau$ and $c$}
\label{app:threshold-verification}
Recall in \cref{eq:rho-hard-app} and \cref{eq:rho-smooth-app}
\begin{equation*}
    \rho_\tau=2(1-\alpha)\bar\Phi\left(\frac{\tau-B}{\sigma}\right)
    +\frac{\alpha\gamma}{\pi}\left(\frac1{\tau-B}+\frac1{\tau+B}\right),
\end{equation*}
\begin{equation*}
    \rho_c=\frac{B+\sqrt{8/\pi}(1-\alpha)\sigma}{c}
    +\frac{8\alpha\gamma}{\pi c}\log\left(e+\frac c\gamma\right).
\end{equation*}
We now verify that the thresholds in the main theorems make the relative-bias constants small enough.

\begin{lemma}[Post-clipping thresholds]
\label{lem:post-thresholds-app}
The hard-clipping threshold $$
    \tau\ge B+
    \max\left\{
    \sigma\sqrt{2\log 8},\,
    \frac{8\alpha\gamma}{\pi}
    \right\},
$$ implies $\rho_\tau\le1/2$.  The smooth-shrinkage threshold $$
    c\ge
    \max\left\{
    4\left[B+\sqrt{\frac8\pi}(1-\alpha)\sigma\right],\,
    \frac{64\alpha\gamma}{\pi}
    \log\left(e+\frac{64\alpha}{\pi}\right)
    \right\},
$$ implies $\rho_c\le1/2$.
\end{lemma}

\begin{proof}
For hard clipping, $\tau  \ge B+ \sigma\sqrt{2\log 8}$ gives
\[
    2(1-\alpha)\bar\Phi\left(\frac{\tau-B}{\sigma}\right)
    \le 2\exp\{-\log 8\}=\frac14,
\]
and, since $\tau+B>\tau-B$,
\[
    \frac{\alpha\gamma}{\pi}\left(\frac1{\tau-B}+\frac1{\tau+B}\right)
    \le \frac{2\alpha\gamma}{\pi(\tau-B)}
    \le \frac14.
\]
Thus $\rho_\tau\le1/2$.

For smooth shrinkage, the first part of the maximum in $c$ gives
\[
    \frac{B+\sqrt{8/\pi}(1-\alpha)\sigma}{c}\le\frac14.
\]
For the logarithmic term, apply \Cref{lem:log-device-app} to the second part of the maximum in $c$ with
\[
    x=\frac c\gamma,
    \qquad
    a=\frac{8\alpha}{\pi},
    \qquad
    s=4.
\]
The second lower bound in $c$ is $x\ge2sa\log(e+2sa)$, so
\[
    \frac{8\alpha\gamma}{\pi c}\log\left(e+\frac c\gamma\right)
    =\frac{a\log(e+x)}{x}
    \le\frac14.
\]
Hence $\rho_c\le1/2$.
\end{proof}

\begin{lemma}[Pre-clipped Muon thresholds]
\label{lem:pre-thresholds-app}
The hard-clipping threshold 
$$
    \tau\ge B+
    \max\left\{
    \sigma\sqrt{2\log(16\sqrt r)},\,
    \frac{16\alpha\gamma\sqrt r}{\pi}
    \right\},
$$
implies $\rho_\tau\le1/(4\sqrt r)$.  The smooth-shrinkage threshold 
$$
    c\ge
    \max\left\{
    8\sqrt r\left[B+\sqrt{\frac8\pi}(1-\alpha)\sigma\right],\,
    \frac{128\alpha\gamma\sqrt r}{\pi}
    \log\left(e+\frac{128\alpha\sqrt r}{\pi}\right)
    \right\},
$$

implies $\rho_c\le1/(4\sqrt r)$.
\end{lemma}

\begin{proof}
For hard clipping, the Gaussian term satisfies
\[
    2(1-\alpha)\bar\Phi\left(\frac{\tau-B}{\sigma}\right)
    \le 2\exp\{-\log(16\sqrt r)\}
    =\frac1{8\sqrt r}.
\]
The Cauchy term satisfies
\[
    \frac{\alpha\gamma}{\pi}\left(\frac1{\tau-B}+\frac1{\tau+B}\right)
    \le \frac{2\alpha\gamma}{\pi(\tau-B)}
    \le \frac1{8\sqrt r}.
\]
Therefore $\rho_\tau\le1/(4\sqrt r)$.

For smooth shrinkage, the first lower bound in gives
\[
    \frac{B+\sqrt{8/\pi}(1-\alpha)\sigma}{c}\le\frac1{8\sqrt r}.
\]
For the logarithmic term, apply \Cref{lem:log-device-app} with
\[
    x=\frac c\gamma,
    \qquad
    a=\frac{8\alpha}{\pi},
    \qquad
    s=8\sqrt r.
\]
The second lower bound in $c$ becomes $x\ge2sa\log(e+2sa)$, hence
\[
    \frac{8\alpha\gamma}{\pi c}\log\left(e+\frac c\gamma\right)
    \le \frac1{8\sqrt r}.
\]
Thus $\rho_c\le1/(4\sqrt r)$.
\end{proof}

\subsection{Proofs of \cref{thm:post-hard-main,thm:post-smooth-main}}
\label{app:post-proofs}

\begin{lemma}[Post-clipping descent with relative bias]
\label{lem:post-relative-descent-app}
Assume \Cref{ass:smooth-main,ass:bounded-main,ass:noise-main}.  Let \(\varphi\) be applied coordinatewise and suppose that, for every \(|g|\le B\), the relative-bias property \cref{eq:relative-bias-app} holds with constant $\rho<1$, and the post-processed update is uniformly entry-wise bounded:
\[
    |m_\varphi(g)-g|\le \rho |g|,
    \qquad
    |\varphi(x)|\le R_\varphi\quad\text{for all }x.
\]
The constant $R_\varphi$ is the entry-wise sup-norm of the processed update; for the two maps used in this paper, $R_\varphi=\tau$ for hard clipping $C_\tau$ and $R_\varphi=c/e$ for smooth shrinkage $S_c$ (\Cref{lem:smooth-properties-app}).
Then the post-clipped update \cref{eq:post-update-main} with
\(\eta=\sqrt{2\Delta/(LR_\varphi^2dK)}\) satisfies
\begin{equation}
    \frac1K\sum_{k=0}^{K-1}\E\|G_k\|_F^2
    \le
    \frac{R_\varphi}{1-\rho}\sqrt{\frac{2L\Delta d}{K}}.
    \label{eq:post-relative-rate-app}
\end{equation}
\end{lemma}

\begin{proof}
Define the conditional processed mean 
\[
    M_k^\varphi:=\E_k\varphi(G_k+\Xi_k)=m_\varphi(G_k),
\]
applied entry by entry.  The relative-bias assumption  (this follows from \cref{lem:hard-relative-bias-app} and \cref{lem:smooth-relative-bias-app} for $C_{\tau}(x)$ and $S_{c}(x)$) gives
\[
    \|M_k^\varphi-G_k\|_F\le \rho\|G_k\|_F.
\]
Therefore
\[
\begin{aligned}
    \langle G_k,M_k^\varphi\rangle
    &=\|G_k\|_F^2+\langle G_k,M_k^\varphi-G_k\rangle \\
    &\ge \|G_k\|_F^2-\|G_k\|_F\|M_k^\varphi-G_k\|_F \\
    &\ge (1-\rho)\|G_k\|_F^2.
\end{aligned}
\]
This is the alignment step: the processed mean is biased, but the bias is not large enough to destroy its positive inner product with the true gradient.

By \(L\)-smoothness,
\[
    F(X_{k+1})
    \le F(X_k)-\eta\langle G_k,\varphi(G_k+\Xi_k)\rangle
    +\frac{L\eta^2}{2}\|\varphi(G_k+\Xi_k)\|_F^2.
\]
Taking conditional expectation on step $k$ and using \(\|\varphi(G_k+\Xi_k)\|_F^2\le R_\varphi^2d\) gives
\[
    \E_kF(X_{k+1})
    \le F(X_k)-\eta(1-\rho)\|G_k\|_F^2
    +\frac{L\eta^2R_\varphi^2d}{2}.
\]
Summing over \(k=0,\ldots,K-1\), using \(F(X_K)\ge\inf_XF(X)\), and dividing by \(K\eta(1-\rho)\),
\[
    \frac1K\sum_{k=0}^{K-1}\E\|G_k\|_F^2
    \le
    \frac{\Delta}{K\eta(1-\rho)}
    +\frac{L\eta R_\varphi^2d}{2(1-\rho)}.
\]
With \(\eta=\sqrt{2\Delta/(LR_\varphi^2dK)}\), the two deterministic terms balance and their sum equals the right-hand side of \cref{eq:post-relative-rate-app}.
\end{proof}

\thmposthard*
\begin{proof}
By \Cref{lem:hard-relative-bias-app}, at any iteration $k$, for gradient noise $\xi_k$ with $0\le\sigma_k\le \sigma$ and $0\le\gamma_k \le \gamma$, hard clipping has a uniform relative-bias coefficient $\rho_\tau$.  By \Cref{lem:post-thresholds-app}, the threshold assumed in the theorem gives \(\rho_\tau\le1/2\).  Also \(|C_\tau(x)|\le\tau\), so \Cref{lem:post-relative-descent-app} applies with \(R_\varphi=\tau\) and \(1/(1-\rho_\tau)\le2\).  This proves the displayed average squared-gradient bound.  Finally, Jensen's inequality gives
\[
    \frac1K\sum_{k=0}^{K-1}\E\|G_k\|_F
    \le
    \left(\frac1K\sum_{k=0}^{K-1}\E\|G_k\|_F^2\right)^{1/2},
\]
and the stated lower bound on \(K\) makes the right-hand side at most \(\varepsilon\).
\end{proof}

\thmpostsmooth*
\begin{proof}
By \Cref{lem:smooth-relative-bias-app}, at any iteration $k$, for gradient noise $\xi_k$ with $0\le\sigma_k\le \sigma$ and $0\le\gamma_k \le \gamma$, smooth shrinkage has a uniform relative-bias coefficient \(\rho_c\).  By \Cref{lem:post-thresholds-app}, the threshold choice $c$ in the theorem gives \(\rho_c\le1/2\).  Also \(|S_c(x)|\le c/e\) by \Cref{lem:smooth-properties-app}.  Applying \Cref{lem:post-relative-descent-app} with \(R_\varphi=c/e\) gives the displayed average squared-gradient bound.  Jensen's inequality gives the stated \(K\) condition for the average gradient norm.
\end{proof}

\subsection{Proofs of \cref{thm:pre-hard-main,thm:pre-smooth-main}}
\label{app:pre-proofs}

\begin{lemma}[Matrix-sign descent]
\label{lem:matrix-sign-descent-app}
Let $\bar G_k$ be any matrix and set $E_k:=\bar G_k-G_k$.  For
\[
    X_{k+1}=X_k-\eta\operatorname{msign}(\bar G_k),
\]
one has
\begin{equation}
    F(X_{k+1})
    \le F(X_k)-\eta\|G_k\|_*+2\eta\|E_k\|_*+\frac{L\eta^2r}{2}.
    \label{eq:matrix-sign-descent-app}
\end{equation}
\end{lemma}

\begin{proof}
Let $U_k:=\operatorname{msign}(\bar G_k)$.  By $L$-smoothness,
\[
    F(X_{k+1})
    \le F(X_k)-\eta\langle G_k,U_k\rangle
    +\frac{L\eta^2}{2}\|U_k\|_F^2.
\]
The matrix $U_k$ has at most $r$ nonzero singular values, each equal to one, so $\|U_k\|_F^2\le r$ and $\|U_k\|_{\mathrm{op}}\le1$.  Moreover, by the definition of the matrix sign,
\[
    \langle \bar G_k,U_k\rangle=\|\bar G_k\|_*.
\]
Using $E_k=\bar G_k-G_k$,
\[
\begin{aligned}
    \langle G_k,U_k\rangle
    &=\langle \bar G_k,U_k\rangle-\langle E_k,U_k\rangle \\
    &\ge \|\bar G_k\|_*-\|E_k\|_*\|U_k\|_{\mathrm{op}} \\
    &\ge \|\bar G_k\|_*-\|E_k\|_*.
\end{aligned}
\]
The reverse triangle inequality gives $\|\bar G_k\|_*\ge\|G_k\|_*-\|E_k\|_*$.  Hence
\[
    \langle G_k,U_k\rangle\ge\|G_k\|_*-2\|E_k\|_*.
\]
Substituting this into the smoothness inequality proves the lemma.
\end{proof}

\begin{lemma}[Pre-clipped Muon under relative bias and finite processed variance]
\label{lem:pre-relative-descent-app}
Assume \Cref{ass:smooth-main,ass:bounded-main,ass:noise-main}.  Suppose $\varphi$ satisfies, for all $|g|\le B$, we can a uniform $\rho$ such that
\[
    |m_\varphi(g)-g|\le\rho |g|,
    \qquad
    \operatorname{Var}(\varphi(g+\xi))\le v.
\]
Run \cref{eq:pre-update-main} with $\eta=\sqrt{2\Delta/(LrT)}$.  If $2\sqrt r\rho<1$, then
\begin{equation}
    \frac1T\sum_{k=0}^{T-1}\E\|G_k\|_*
    \le
    \frac{\sqrt{2Lr\Delta/T}+2Dr^{3/2}\sqrt{v/N}}{1-2\sqrt r\rho}.
    \label{eq:pre-relative-rate-app}
\end{equation}
\end{lemma}

\begin{proof}
Let $M_k^\varphi:=\E_k\bar G_k^\varphi=m_\varphi(G_k)$.  Use the exact decomposition
\[
    \bar G_k^\varphi-G_k=(\bar G_k^\varphi-M_k^\varphi)+(M_k^\varphi-G_k).
\]
We bound the two terms separately.

First consider the zero-mean stochastic part.  Each entry of $\bar G_k^\varphi-M_k^\varphi$ is an average of $N$ independent processed samples, so its conditional variance is at most $v/N$.  Hence
\[
\begin{aligned}
    \E_k\|\bar G_k^\varphi-M_k^\varphi\|_*
    &\le \sqrt r\,\E_k\|\bar G_k^\varphi-M_k^\varphi\|_F \\
    &\le \sqrt r\left(\E_k\|\bar G_k^\varphi-M_k^\varphi\|_F^2\right)^{1/2} \\
    &\le \sqrt r\sqrt{d\frac vN}.
\end{aligned}
\]
Since $d=mn=rq$ and $D=\sqrt{q/r}$, we have $\sqrt r\sqrt d=Dr^{3/2}$.  Thus
\[
    \E_k\|\bar G_k^\varphi-M_k^\varphi\|_*
    \le Dr^{3/2}\sqrt{\frac vN}.
\]
This step uses independence across the sample index; it does not require independence across matrix entries.

Now consider the deterministic conditional bias.  By the relative-bias assumption (this follows from \cref{lem:hard-relative-bias-app} and \cref{lem:smooth-relative-bias-app} for $C_{\tau}(x)$ and $S_{c}(x)$),
\[
    \|M_k^\varphi-G_k\|_F\le \rho\|G_k\|_F.
\]
Therefore
\[
    \|M_k^\varphi-G_k\|_*
    \le \sqrt r\|M_k^\varphi-G_k\|_F
    \le \sqrt r\rho\|G_k\|_F
    \le \sqrt r\rho\|G_k\|_*.
\]
Combining the two bounds gives
\begin{equation}
    \E_k\|\bar G_k^\varphi-G_k\|_*
    \le Dr^{3/2}\sqrt{\frac vN}+\sqrt r\rho\|G_k\|_*.
    \label{eq:pre-error-bound-app}
\end{equation}

Apply \Cref{lem:matrix-sign-descent-app} with $\bar G_k=\bar G_k^\varphi$, take conditional expectation, and use \cref{eq:pre-error-bound-app}:
\[
    \E_kF(X_{k+1})
    \le F(X_k)-\eta(1-2\sqrt r\rho)\|G_k\|_*
    +2\eta Dr^{3/2}\sqrt{\frac vN}
    +\frac{L\eta^2r}{2}.
\]
Summing over $k=0,\ldots,T-1$, using $F(X_T)\ge\inf_XF(X)$, and dividing by $\eta T$ yields
\[
    (1-2\sqrt r\rho)A_T
    \le
    \frac{\Delta}{\eta T}+\frac{L\eta r}{2}+2Dr^{3/2}\sqrt{\frac vN},
\]
where
\[
    A_T:=\frac1T\sum_{k=0}^{T-1}\E\|G_k\|_*.
\]
The choice $\eta=\sqrt{2\Delta/(LrT)}$ balances the first two deterministic terms and gives
\[
    \frac{\Delta}{\eta T}+\frac{L\eta r}{2}=\sqrt{\frac{2Lr\Delta}{T}}.
\]
Dividing by $1-2\sqrt r\rho$ proves the lemma.
\end{proof}

\thmprehard*
\begin{proof}
At any iteration $k$, for gradient noise $\xi_k$ with $0\le\sigma_k\le \sigma$ and $0\le\gamma_k \le \gamma$, hard clipping satisfies the relative-bias bound of \Cref{lem:hard-relative-bias-app} and the variance bound of \Cref{lem:hard-variance-app}.  By \Cref{lem:pre-thresholds-app}, the threshold assumed in the theorem gives \(\rho_\tau\le1/(4\sqrt r)\), so \(1-2\sqrt r\rho_\tau\ge1/2\).  Applying \Cref{lem:pre-relative-descent-app} with \(v=v_\tau\) gives
\[
    \frac1T\sum_{k=0}^{T-1}\E\|G_k\|_*
    \le
    2\sqrt{\frac{2Lr\Delta}{T}}
    +4Dr^{3/2}\sqrt{\frac{v_\tau}{N}},
\]
which is the displayed theorem bound.  The lower bounds on \(T\) and \(N\) make the two terms at most \(\varepsilon/2\) each.
\end{proof}

\thmpresmooth*
\begin{proof}
At any iteration $k$, for gradient noise $\xi_k$ with $0\le\sigma_k\le \sigma$ and $0\le\gamma_k \le \gamma$, Smooth shrinkage satisfies the relative-bias bound of \Cref{lem:smooth-relative-bias-app} and the variance bound of \Cref{lem:smooth-variance-app}.  By \Cref{lem:pre-thresholds-app}, the threshold assumed in the theorem gives \(\rho_c\le1/(4\sqrt r)\), so \(1-2\sqrt r\rho_c\ge1/2\).  Applying \Cref{lem:pre-relative-descent-app} with \(v=v_c\) gives the displayed theorem bound.  The lower bounds on \(T\) and \(N\) make the two terms at most \(\varepsilon/2\) each.
\end{proof}

\newpage
\section{Experiment Details for \texorpdfstring{\Cref{sec:experiments}}{Section 6}}
\label{app:experiments}

\subsection{Gaussian Random Feature Regression}
\label{app:synthetic}
\textbf{Hardware.} The experiment was conducted on a CPU machine with 128 AMD EPYC 7532 32-Core Processors.

\textbf{Problem setup.}
We minimize $L(W) = \frac{1}{2n}\|WA - Y\|_F^2$ where $W \in \R^{d_{\mathrm{out}} \times d_{\mathrm{h}}}$ and $A \in \R^{d_{h} \times n}$, feature matrix $A$ follows i.i.d.\ standard Gaussian, and $Y = W^\sharp A$ with $W^\sharp \in \RR^{d_{\mathrm{out}} \times d_h}$ also i.i.d.\ standard Gaussian.
We use $d_{\mathrm{out}} = d_h =32$, $n=128$ and  $d_{\mathrm{out}} = d_h =64$, $n=256$ two settings.

\textbf{Noise model.}
At each step, the true gradient $\nabla L(W_k) = \tfrac{1}{n}(W_k A - Y)A^T$
is corrupted by noise $E_k$ drawn from the contamination model of
\cref{def:contamination} with heavy-tailed component $H = t_1$
(Cauchy noise). We write
$\tilde G_k \coloneqq \nabla L(W_k) + E_k$ for the observed gradient.
The contamination fraction $\alpha$ is selected over
$\{0,\, 10^{-3},\, 10^{-2},\, 5\!\times\!10^{-2},\, 10^{-1},\, 0.5,\, 0.8,\, 1.0\}$. Each configuration is repeated over $10$ random seeds (independent draws of
$A$, $W^\sharp$, and noise $E$). We report the median loss curve.

\textbf{Methods.}
We compare six methods: vanilla GD and spectral GD (with spectral normalization implemented by SVD decomposition), each combined with one
of three clipping maps. The clipping is applied at a different stage
depending on the method: for vanilla GD, the map acts directly on the
noisy gradient (\emph{post-clipping}). For spectral GD, it acts on the
noisy gradient before spectral normalization (\emph{pre-clipping}). The
three clipping maps are: (a) no clipping, (b) entry-wise hard clipping at
threshold $C_{\tau}(x)$, and (c) smooth shrinkage $S_c(x)$. For (b) and (c),
$\tau$ is set to the $q$-th quantile of $\{|\tilde G_{k,ij}|\}_{i,j}$ at
each step, with $q \in \{0.90,\, 0.95,\, 0.99,\, 0.995,\, 0.999,\, 0.9995,\, 0.99999\}$.
For each method we sweep the joint grid of learning rate
$\eta \in \{0.001,\, 0.005,\, 0.01,\, 0.02,\, 0.05,\, 0.1\}$ and clipping
quantile, and report the best final loss.

\textbf{Results.}
\emph{Post-clipping.} Results are shown in \cref{fig:post_d32} ($d=32$) and
\cref{fig:post_d64} ($d=64$). Vanilla GD fails to converge once
$\alpha \ge 10^{-2}$, whereas both clipping methods converge across all
contamination strengths. In both dimensions, smooth shrinkage outperforms
hard clipping when heavy-tailed noise dominates ($\alpha \ge 0.5$). Both
methods also control the spectrum and help recovery of the rotated subspace
(panel~(a) of \cref{fig:post_d32,fig:post_d64}).

\emph{Pre-clipping.} Results are shown in \cref{fig:pre_d32} ($d=32$) and
\cref{fig:pre_d64} ($d=64$). Spectral GD converges across all contamination
strengths with or without pre-clipping; however, both hard clipping and
smooth shrinkage \emph{accelerate} spectral GD relative to the unclipped
baseline beginning at $\alpha = 10^{-2}$, with the speedup growing as
$\alpha$ increases. As in the post-clipping case, smooth shrinkage
outperforms hard clipping for $\alpha \ge 0.5$. We observe similar
spectral control and subspace recovery in both $d=32$ and $d=64$.

\subsubsection{Noise Misspecification}
\label{sec:noise_misspec}

We evaluate the robustness of both clipping maps when the noise deviates
from the Cauchy-contaminated model assumed in \cref{sec:convergence}. We
consider two families of misspecification and report the convergence
speedup $S = L / L_{\varphi}$, where $L$ and $L_{\varphi}$ denote the final
training loss of the unprocessed and processed method, respectively, under
the same grid search protocol as in \cref{app:synthetic}.

\paragraph{Tail-index misspecification.}
We replace the Cauchy contamination ($\nu = 1$, the specified case) with
Student-$t_\nu$ noise for $\nu \in \{2, 3\}$, keeping the threshold rule
unchanged. \cref{tab:misspec_tail} reports the speedup for contamination
fractions $\alpha = .01/.1/.5/1$. Both maps retain their gains as the tail
becomes lighter, although the magnitude decreases with $\nu$. At high
contamination ($\alpha \geq 0.5$), smooth shrinkage generally matches or
exceeds hard clipping. The benefit of smooth shrinkage is therefore not
specific to the assumed Cauchy tail.

\begin{table}[h]
\centering
\caption{Speedup $S$ under tail-index misspecification ($>1$ indicates improvement). Cells correspond
to contamination fractions $\alpha = .01/.1/.5/1$. The threshold rule is
calibrated for $\nu = 1$ (specified) and held fixed across $\nu$.}
\label{tab:misspec_tail}
\small
\begin{tabular}{lccc}
\toprule
Mode/method & $\nu=1$ (specified) & $\nu=2$ & $\nu=3$ \\
\midrule
Post clipping (hard clipping $C_{\tau}(x)$)   & 63/5400/5890/4830 & 1.5/3.8/5.1/3.2 & 1.2/1.5/2.1/1.3 \\
Post clipping (smooth shrinkage $S_{c}(x)$) & 50/4030/9280/8040 & 1.4/4.7/5.0/3.4 & 1.1/1.9/2.2/1.6 \\
Pre clipping (hard clipping $C_{\tau}(x)$)    & 1.1/2.4/3.3/3.5   & 1.1/1.7/1.8/1.6 & 1.1/1.5/1.4/1.2 \\
Pre clipping (smooth shrinkage $S_{c}(x)$)  & 1.1/1.7/4.3/6.4   & 1.1/1.5/1.8/1.8 & 1.1/1.3/1.5/1.3 \\
\bottomrule
\end{tabular}
\end{table}

\begin{table}[h]
\centering
\caption{Speedup $S$ under BBP-style low-rank noise ($>1$ indicates improvement).  Cells correspond to
anisotropy exponents $a = .25/1/2/5$.}
\label{tab:misspec_lowrank}
\small
\begin{tabular}{lccc}
\toprule
Mode/method & Aniso $U_r$ & Aniso $V_r$ & Both \\
\midrule
Post clipping (hard clipping $C_{\tau}(x)$)   & 1.3/1.5/2.0/2.8 & 1.3/1.6/2.5/3.1 & 1.3/1.8/3.4/10.3 \\
Post clipping (smooth shrinkage $S_{c}(x)$) & 2.0/2.1/2.1/2.5 & 2.1/2.2/2.3/2.8 & 2.1/2.2/3.2/9.9  \\
Pre clipping (hard clipping $C_{\tau}(x)$)   & .85/.72/.71/.93 & .85/.74/.74/.89 & .84/.59/.51/1.3  \\
Pre clipping (smooth shrinkage $S_{c}(x)$)  & .15/.09/.02/.18 & .16/.12/.03/.18 & .15/.09/.03/1.1  \\
\bottomrule
\end{tabular}
\end{table}

\paragraph{Structured low-rank misspecification.}
To probe violations of the entry-wise independence assumption, we adopt
a low-rank contamination model inspired by the BBP phase transition \citep{baik2005phase,benaych2012singular}, i.e., $E = m U_r V_r^{\top}$, where the columns of
$U_r$ (and analogously $V_r$) are drawn with anisotropic energy profile
$\mathrm{diag}(\mathbb{E}[U_r U_r^{\top}]) \propto i^{-a}$, so that larger
anisotropy exponents $a$ concentrate the noise energy on fewer rows.
\cref{tab:misspec_lowrank} reports the speedup for
$a = .25/1/2/5$, applying the anisotropy to $U_r$, $V_r$, or both. Under
this structured noise, both maps improve post-clipping, and smooth
shrinkage is strongest at low anisotropy exponents ($a \leq 1$). In
contrast, pre-clipping is generally unhelpful and can be harmful, a detailed reason to explain this would be an interesting further direction.

\begin{figure}[htbp]
  \centering
  
  \begin{subfigure}{\linewidth}
    \centering
    \includegraphics[width=0.98\linewidth]{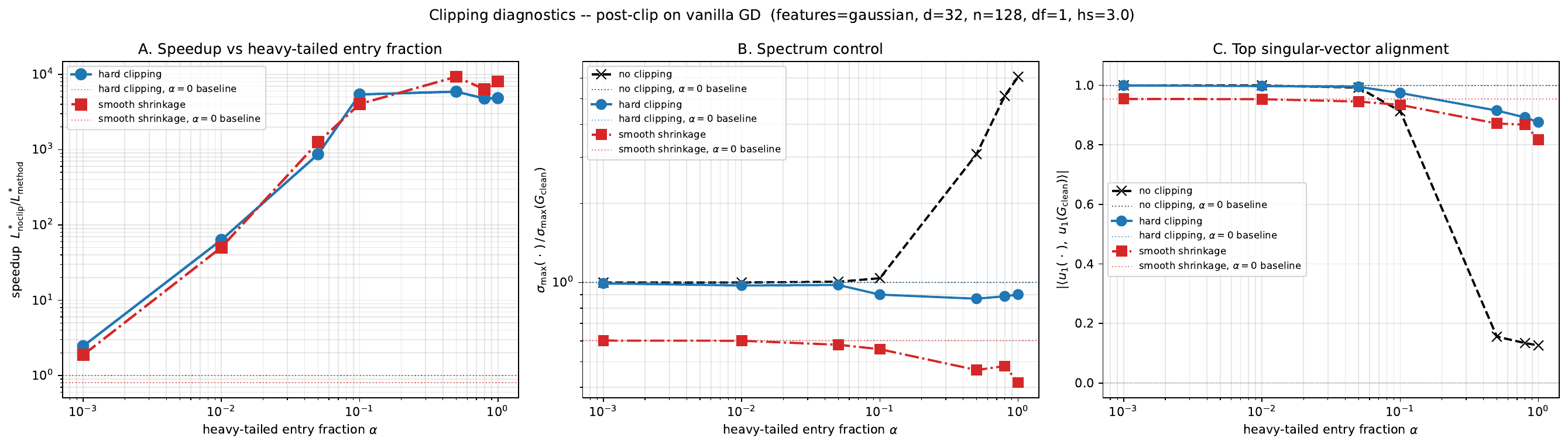}
    \caption{Summary of speedup, spectral control, and subspace recovery across clipping methods and contamination strengths.}

  \end{subfigure}

  \vspace{1cm} 

  \begin{subfigure}{\linewidth}
    \centering
    \includegraphics[width=0.98\linewidth]{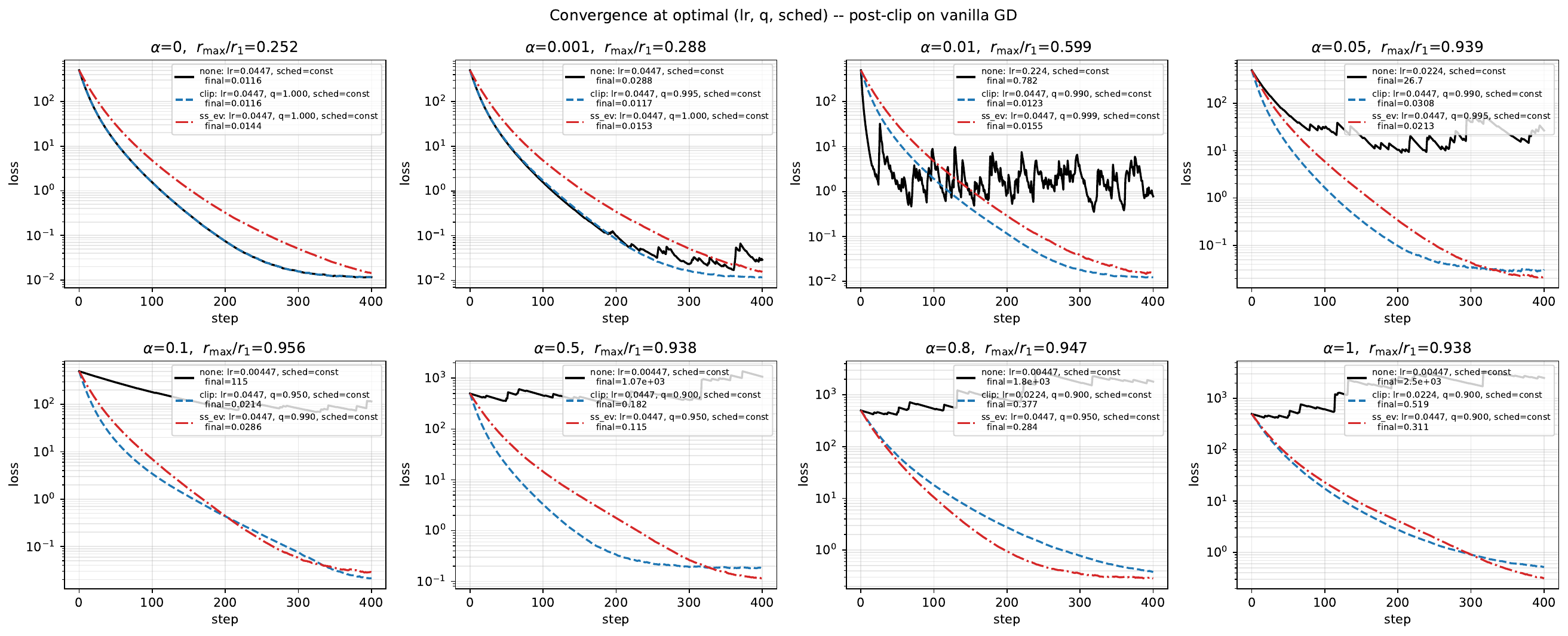}
    \caption{Convergence under different clipping methods and contamination strengths.}

  \end{subfigure}

  \caption{Post Clipping for SGD under Gaussian random feature models ($d=32$)}
  \label{fig:post_d32}
\end{figure}

\begin{figure}[htbp]
  \centering
  
  \begin{subfigure}{\linewidth}
    \centering
    \includegraphics[width=0.98\linewidth]{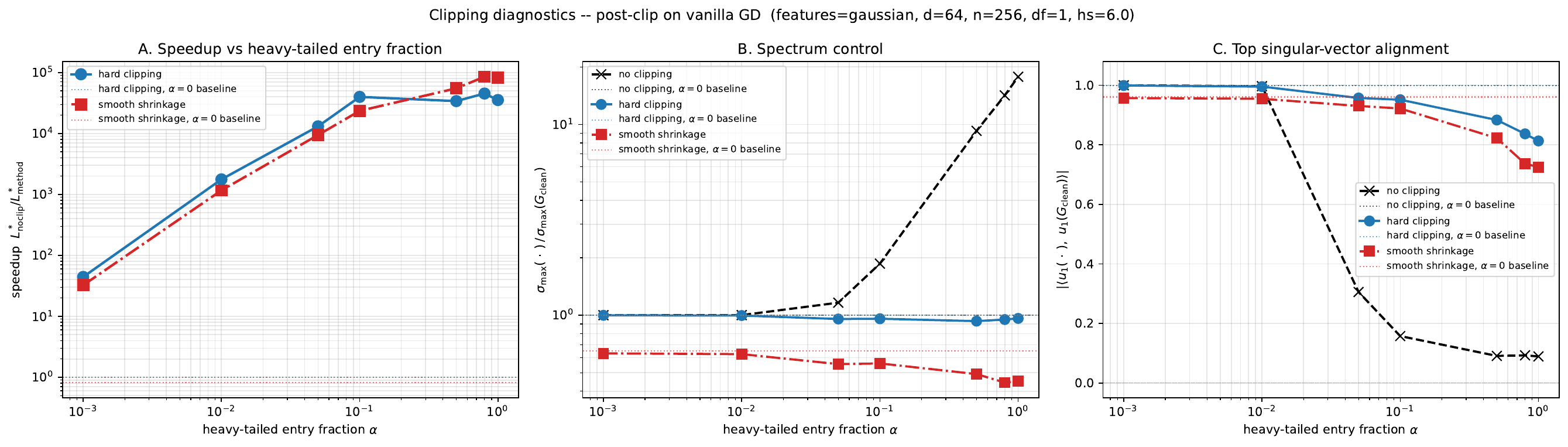}
    \caption{Summary of speedup, spectral control, and subspace recovery across clipping methods and contamination strengths.}
  \end{subfigure}

  \vspace{1cm} 

  \begin{subfigure}{\linewidth}
    \centering
    \includegraphics[width=0.98\linewidth]{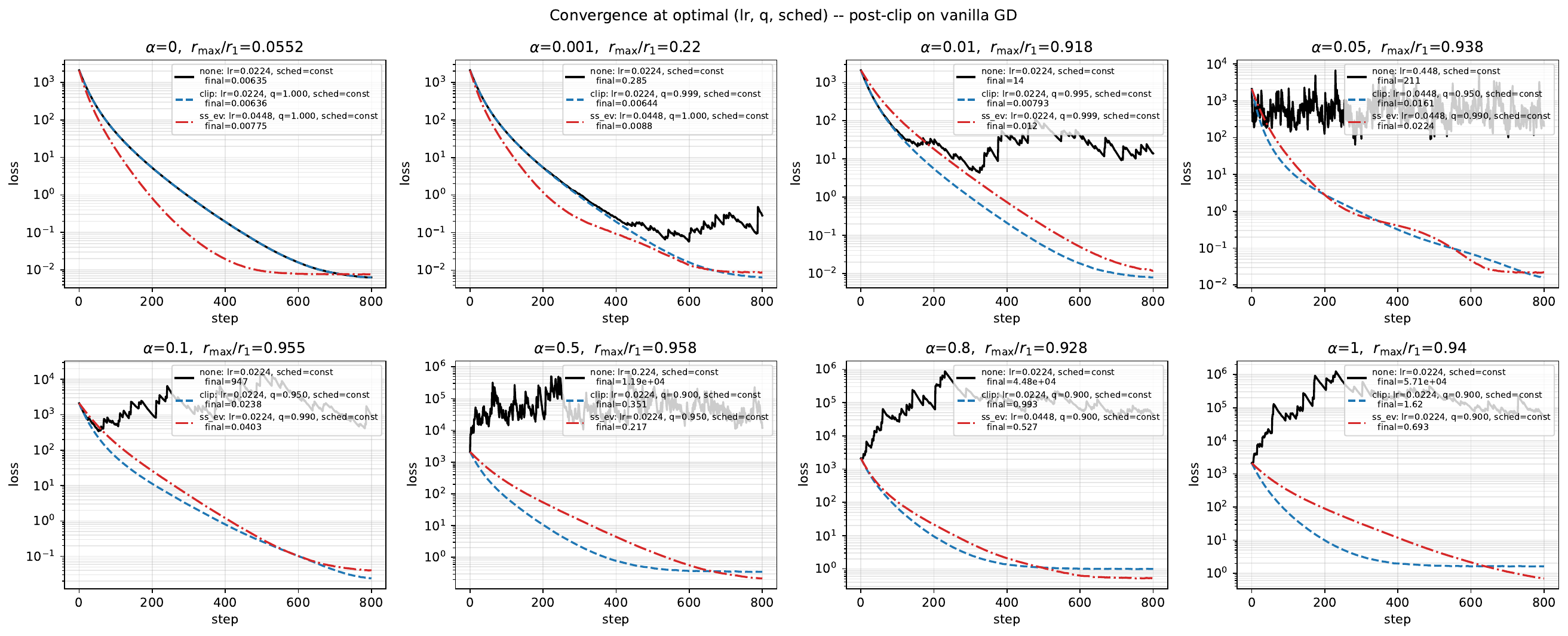}
    \caption{Convergence under different clipping methods and contamination strengths.}
  \end{subfigure}

  \caption{Post Clipping for SGD under Gaussian random feature models ($d=64$)}
  \label{fig:post_d64}
\end{figure}

\begin{figure}[htbp]
  \centering
  
  \begin{subfigure}{\linewidth}
    \centering
    \includegraphics[width=0.98\linewidth]{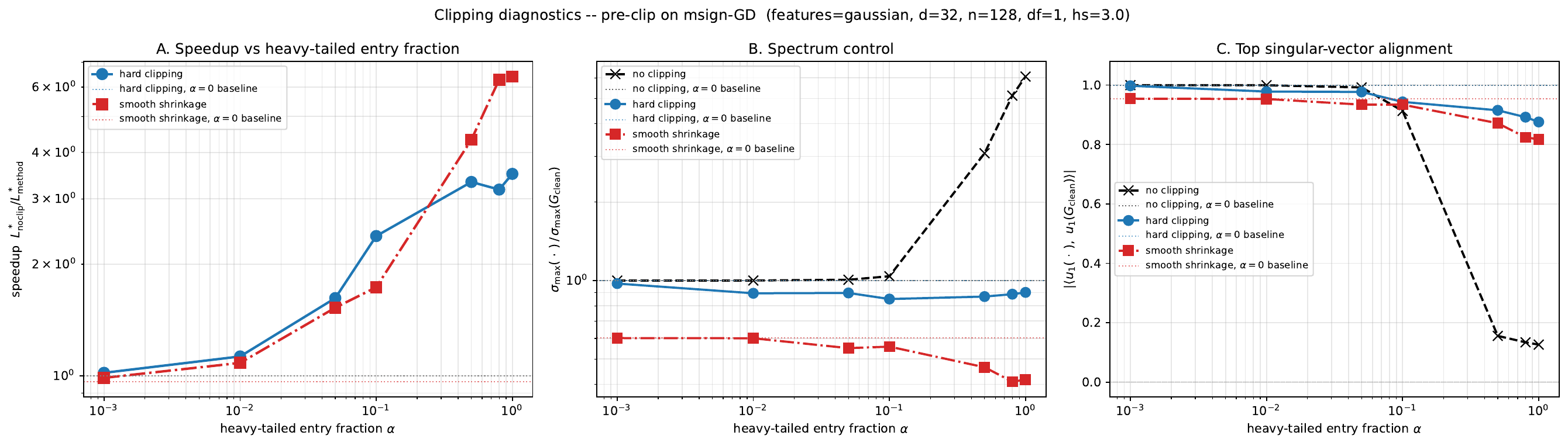}
    \caption{Summary of speedup, spectral control, and subspace recovery across clipping methods and contamination strengths.}
  \end{subfigure}

  \vspace{1cm} 

  \begin{subfigure}{\linewidth}
    \centering
    \includegraphics[width=0.98\linewidth]{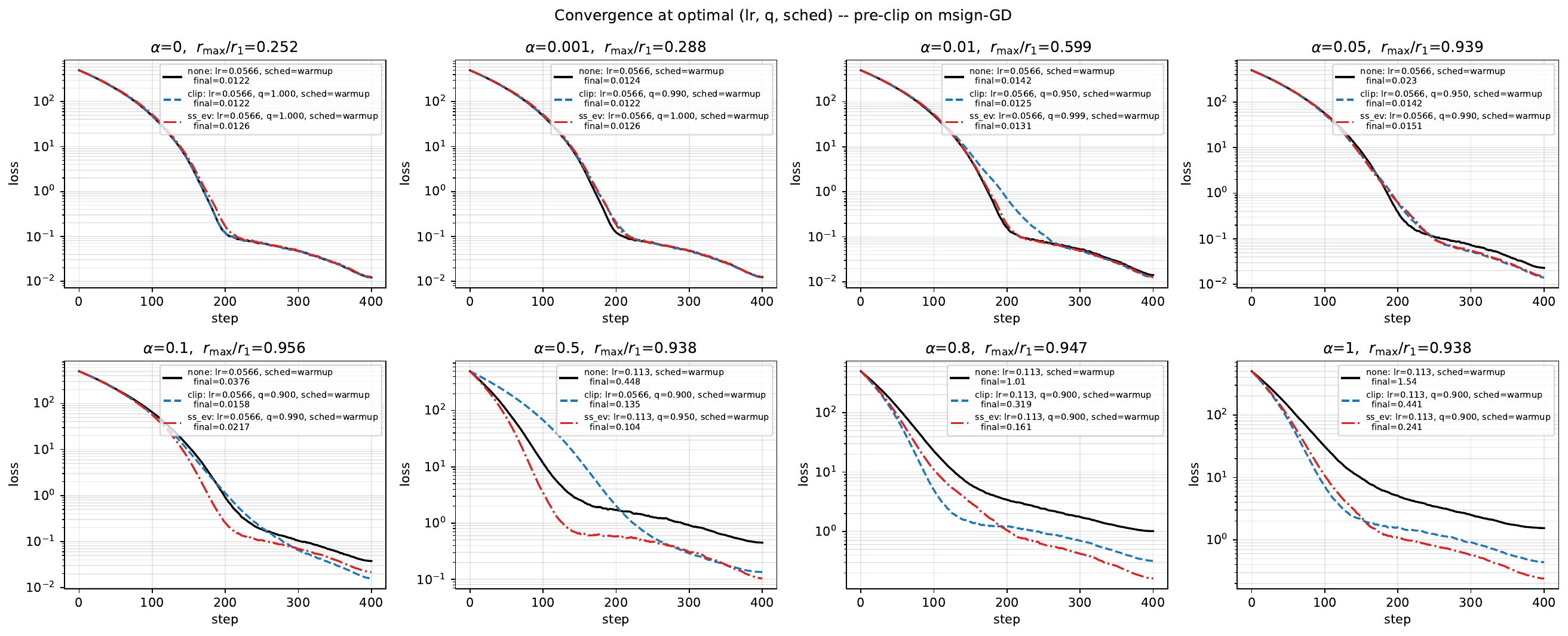}
    \caption{Convergence under different clipping methods and contamination strengths.}
  \end{subfigure}

  \caption{Pre Clipping for spectral GD under Gaussian random feature models ($d=32$)}
  \label{fig:pre_d32}
\end{figure}

\begin{figure}[htbp]
  \centering
  
  \begin{subfigure}{\linewidth}
    \centering
    \includegraphics[width=0.98\linewidth]{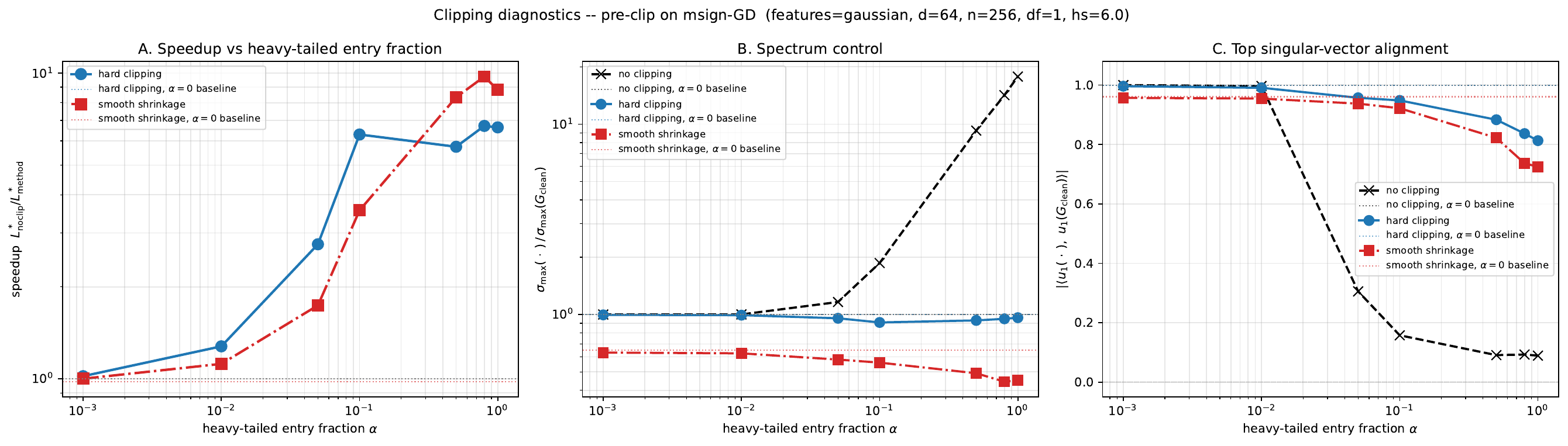}
    \caption{Summary of speedup, spectral control, and subspace recovery across clipping methods and contamination strengths.}
  \end{subfigure}

  \vspace{1cm} 

  \begin{subfigure}{\linewidth}
    \centering
    \includegraphics[width=0.98\linewidth]{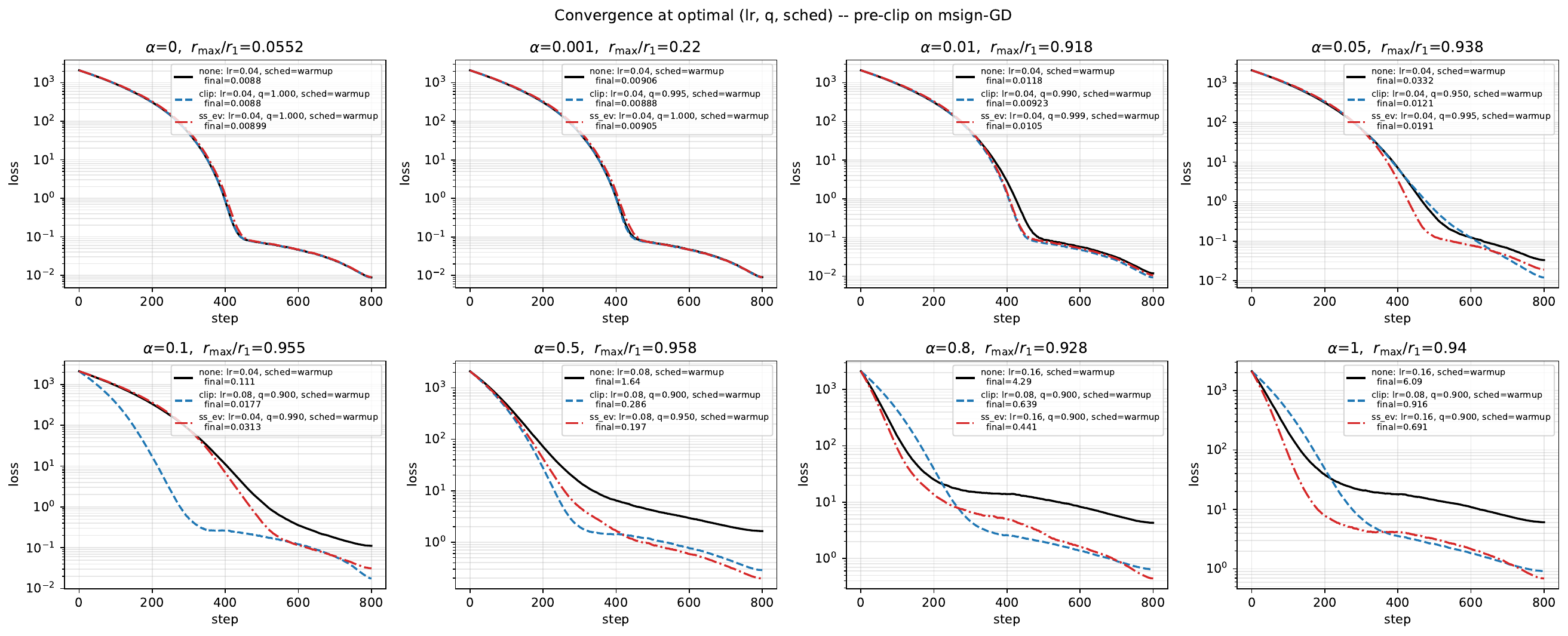}
    \caption{Convergence under different clipping methods and contamination strengths.}
  \end{subfigure}

  \caption{Pre Clipping for spectral GD under Gaussian random feature models ($d=64$)}
  \label{fig:pre_d64}
\end{figure}

\FloatBarrier
\subsection{Language Model Pretraining}
\label{app:lm}

\textbf{Hardware.} The experiment was conducted on a single NVIDIA A100 GPU machine.

\textbf{Model and data.}
We use the May 27, 2025 snapshot of \texttt{modded-nanogpt} \citep{modded_nanogpt_2024} as the basis for our 276M-parameter NanoGPT experiments. This snapshot incorporates several modern architectural enhancements: rotary position embeddings, RMSNorm, a linear-decay learning-rate schedule, and squared-ReLU (ReLU$^2$) activations. It also features an untied output head, three additional value-embedding layers whose lookups are added to the attention values in the early transformer blocks, and a residual-stream U-net that connects early and late blocks via learned skip connections. Each embedding-style block contains approximately 38M parameters ($768 \times 50{,}257$); these four blocks---the untied output head together with the three value-embedding layers---account for 152M parameters, which combined with the 124M-parameter 12-layer transformer backbone yields a total of 276M parameters. We train on the 10B-token subset of FineWeb \citep{penedo2024fineweb}.

\textbf{Methods.} To test the effect of \emph{post-clipping}, we apply two clipping maps $\varphi$, smooth shrinkage and hard clipping, to the AdamW \citep{loshchilov2017decoupled} update $U_k = M_k / (\sqrt{V_k} + \epsilon)$a, i.e., $W_{k+1} = W_k - \eta\,\varphi(U_k)$. We compare them against unclipped AdamW. To test the effect of \emph{pre-clipping}, we apply $\varphi$ to the Nesterov momentum $M_k$ before the matrix-sign step, i.e., $W_{k+1} = W_k - \eta\,\msign(\varphi(M_k))$, where $\msign$ is computed via the Newton--Schulz iteration of \citet{jordan2024muon} and the update is rescaled by $0.2\sqrt{\max(n,m)}$ following \citet{team2025kimi}. We compare these pre-clipping methods against Muon \citep{jordan2024muon} and SPECTRA \citep{jiang2026enhancing}, with the latter applying spectral clipping to the Adam update via Newton--Schulz iterations.

\textbf{Hyperparameters.}
We perform a grid search over learning rate $\times$ clipping threshold $q$ $\times$ first-order momentum coefficient $\beta_1$ for all clipping methods, where the clipping threshold is the $q$-th quantile of $\lvert U_k\rvert$ (post-clipping) or $\lvert M_k\rvert$ (pre-clipping). The unclipped baselines AdamW and Muon are swept over the same learning-rate grid; SPECTRA additionally decays its clipping threshold alongside the learning rate following a WSD \citep{hu2024minicpm} schedule. We sweep the grid of $\{\text{learning\_rate}\} \times \{q\} \times \{ \beta_1\}$ formed by $\{1\!\times\!10^{-4},\, 2\!\times\!10^{-4},\, 6\!\times\!10^{-4},\, 1\!\times\!10^{-3},\, 1.6\!\times\!10^{-3},\, 4\!\times\!10^{-3}\} \times \{.95,.99,.995,999\} \times \{.8,.9,.95,.975,.9875\}$ under one seed. The full sweeping results are presented in \cref{fig:ss_sweep}. We observe that post-clipping (Bottom row of \cref{fig:ss_sweep}) is more sensitive to quantile threshold $q$ and $\beta$, and only two swept $\beta$ values beat the unclipped baseline whereas pre-clipping (Top row of \cref{fig:ss_sweep}) is more robust to both $q$ and $\beta$ . The practical conclusion is therefore that $q\in[.99,.999]$ is a robust search range for both pre-clipping and post-clipping, but post clipping requires joint tuning of $q$ and $\beta_1$ (typically $\beta_1=.8-.9$).

Based on the full hyperparameter sweeping results, we provide a finer look at the learning rate sweep result in \cref{fig:lr_sweep} given the best $q$ and $\beta_1$. The complete best hyperparameter settings to reproduce \cref{fig:nano_gpt_result} are listed in \cref{tab:hyper_nanogpt_post} and \cref{tab:hyper_nanogpt_pre} for post- and pre-clipping, respectively.

\begin{figure}[htp]
    \centering
    \includegraphics[width=0.97\linewidth]{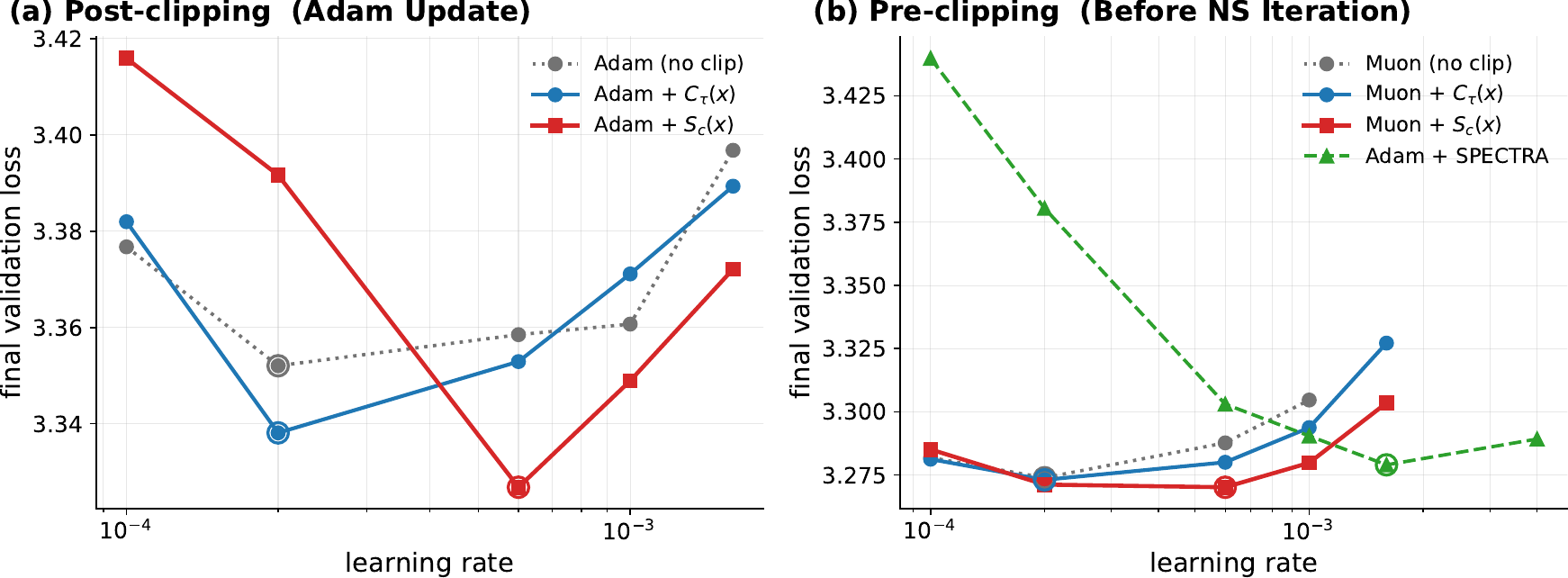}
    \caption{Learning-rate sweeps for post- and pre-clipping methods in terms of the final validation loss on NanoGPT.}
    \label{fig:lr_sweep}
\end{figure}

\begin{table}[htp]
\centering
\caption{Hyperparameter configurations for the optimizers evaluated for \emph{post clipping} in \cref{fig:nano_gpt_result}. $\beta_1$ and $\beta_2$ for AdamW are set to  $0.9$ and $0.99$ by convention.}
\label{tab:hyper_nanogpt_post}
\begin{tabularx}{\textwidth}{l XXX} 
\toprule
\textbf{Hyperparameter} & \textbf{AdamW} & \textbf{AdamW with $C_{\tau}(x)$}  & \textbf{AdamW with $S_{c}(x)$}  \\ 
\midrule
Learning Rate ($\eta$) & 0.0002 & 0.0002 & 0.0006  \\
Clipping Quantile $q$ & -- & 99.9\% & 99.5\% \\ 
Weight Decay           & 0.1 & 0.1 & 0.1\\
Batch Size (tokens)             & 49152   & 49152 & 49152     \\
$\beta_1$ (First Moment)  & 0.9 & 0.9 & 0.8  \\
$\beta_2$ (Second Moment) & 0.99 & 0.99 &0.99   \\
$\epsilon$ (Stability)    & $10^{-8}$ & $10^{-8}$ & $10^{-8}$  \\
\bottomrule
\end{tabularx}
\end{table}

\begin{table}[htp]
\centering
\caption{Hyperparameter configurations for the optimizers evaluated for \emph{pre clipping} in \cref{fig:nano_gpt_result}. $\beta_1$ for Muon is set to $0.95$ by convention.}
\label{tab:hyper_nanogpt_pre}
\begin{tabularx}{\textwidth}{l XXXX} 
\toprule
\textbf{Hyperparameter} & \textbf{Muon} & SPECTRA & \textbf{Muon  ($C_{\tau}(x)$)}  & \textbf{Muon ($S_{c}(x)$)}  \\ 
\midrule
Learning Rate ($\eta$) & 0.0002 & 0.0016 & 0.0002 & 0.0006  \\
Clipping Quantile $q$ & -- & -- &  99.5\% & 99\% \\ 
Weight Decay         & 0.1 & 0.1 & 0.1 & 0.1\\
Batch Size (tokens)     & 49152    & 49152   & 49152 & 49152     \\
$\beta_1$ (First Moment)  & 0.95 & 0.9 & 0.95 & 0.95  \\
$\beta_2$ (Second Moment) & -- & 0.99 & --  & --  \\
$\epsilon$ (Stability)    & -- & $10^{-8}$ & -- & --  \\
\bottomrule
\end{tabularx}
\end{table}

\textbf{Noise structure analysis.}
We investigate the structure of stochastic gradient noise. Let $G \in \mathbb{R}^{m \times n}$ denote the true gradient signal, approximated by averaging per-sample gradients over a large batch of $1024$ sequences with $4096$ tokens each. Let $U_G \in \mathbb{R}^{m \times r}$ and $V_G \in \mathbb{R}^{n \times r}$ collect the top-$r$ left and right singular vectors of $G$, defining its signal subspace. Let $g$ denote a single-sample stochastic gradient and the noise to be $E := g - G$.

Our goal is to measure how much the top entry of the noise matrix $E$ perturbs the spectrum of $G$ at different training stages (\cref{thm:taylor}). We summarize this through the localization ratio $\widehat{R}(E)$ in \cref{def:localization}. We report measurements of stochastic noise at three training stages in \cref{fig:r_hat}. Early in training (step $1000$), nearly all layers' stochastic noise has $\widehat{R}(E) \ll 1$, indicating that the noise is aligned with the signal subspace. As training progresses, $\widehat{R}(E)$ generally rises and exceeds $1$, with the effect strengthening as $k$ grows: the spike-like structure of $E$ becomes increasingly pronounced when $(u_k, v_k)$ are taken from subleading singular directions of $G$.


\subsection{Spectrum Control on Real Gradients}
\label{sec:spectrum_control_details}

\paragraph{Setup.}
We use the NanoGPT model and pretraining setup of \cref{app:lm}
and evaluate at checkpoints of steps 1k, 5k, and 25k. At each checkpoint,
the true gradient $G$ of each weight matrix is approximated by averaging
per-sample gradients over a large batch of 1024 sequences of 4096 tokens
each. We then draw 50 fresh mini-batch gradients $\widetilde{G}$ at the
training batch size and apply each of the four clipping maps: exact spectral clipping, hard clipping, smooth shrinkage, and no clipping, to every draw.

\paragraph{Maps and thresholds.}
The entry-wise threshold for hard clipping and smooth shrinkage is the
$0.995$ quantile of $|\widetilde{G}|$, matching the quantile rule used in
training (\cref{app:lm}). Exact spectral clipping applies
$\hat{\sigma}_i = \min(\sigma_i, c)$ to the exact SVD of $\widetilde{G}$
with $c = 10$, the stable-phase value of the WSD clipping schedule used by
SPECTRA \citep{jiang2026enhancing} in our experiments.

\paragraph{Metric.}
For each processed gradient $\widehat{G}$, the top-$r$ spectrum error is
$e_r = \|\hat{\sigma}_r - \sigma_r\|_2^2 / \|\sigma_r\|_2^2$, where
$\sigma_r \in \R^{r}$ collects the top-$r$ singular values of the large-batch
reference $G$ and $\hat{\sigma}_r$ those of $\widehat{G}$. We report the complete result from
$e_1$ to  $e_8$ in \cref{fig:spec_att} and \cref{fig:spec_mlp} for both attention and MLP layers. We present the compact result of $e_1$ through taking the median over the 44 attention and 24 MLP weight
matrices in \cref{fig:spectrum_control}.

\paragraph{Discussion.}
The results in  \cref{fig:spec_att} and \cref{fig:spec_mlp} localize where each map is
effective. Early in training (step 1k), when the localization ratio becomes $\widehat R(E) \ll 1$ (see \cref{fig:r_hat}), smooth shrinkage attains the lowest
error of all maps, including exact spectral clipping. As training
progresses to the end (step 25k) and the localization ratio becomes $\widehat R(E) \gg 1$, exact
spectral clipping becomes the strongest option, while smooth shrinkage
continues to substantially outperform hard clipping at identical cost.
These observations match the localization analysis of \cref{sec:localization}:
entry-wise maps control the spectrum precisely when the perturbation is
entry-wise \textit{localized}, i.e., $\widehat R(E) \ll 1$ or $\widehat R(E) \gg 1$,.

\newpage

\begin{figure}[htp]
    \centering
    \includegraphics[width=0.98\linewidth]{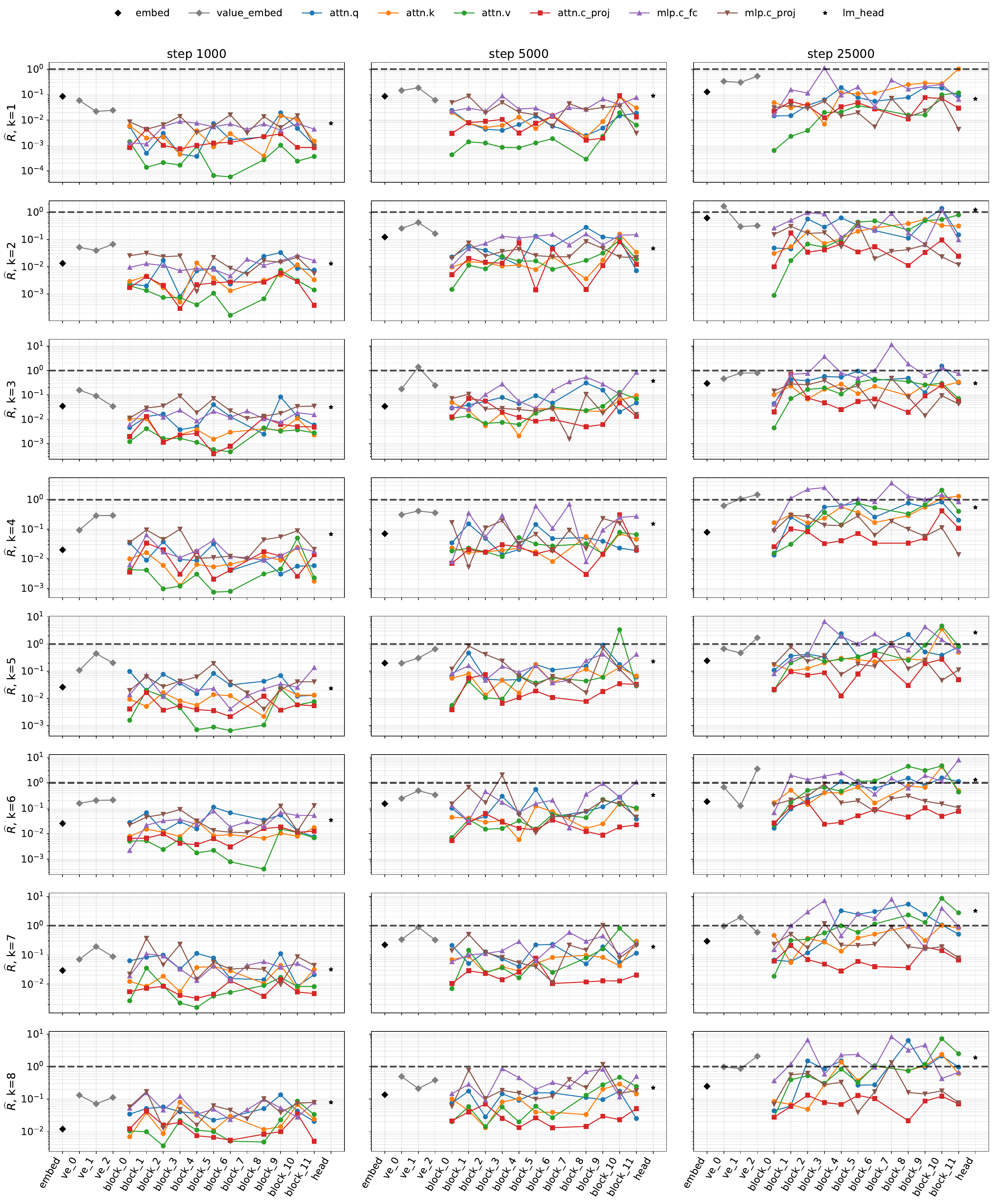}
    \caption{Localization ratio $\widehat{R}(E_{\mathrm{real}})$ of real stochastic noise plotted across layers at three training stages. Each row corresponds to a different singular direction of the signal $G$, from the leading direction (top row) to the $8$th (bottom). The thick dashed line at $y=1$ marks the Gaussian null baseline, at which the noise is delocalized and matches the random-matrix prediction. As training progresses, $\widehat{R}(E)$ generally rises and exceeds $1$, with the effect strengthening as $k$ grows.}

    \label{fig:r_hat}
\end{figure}

\begin{figure}
    \centering
    \includegraphics[width=0.95\linewidth]{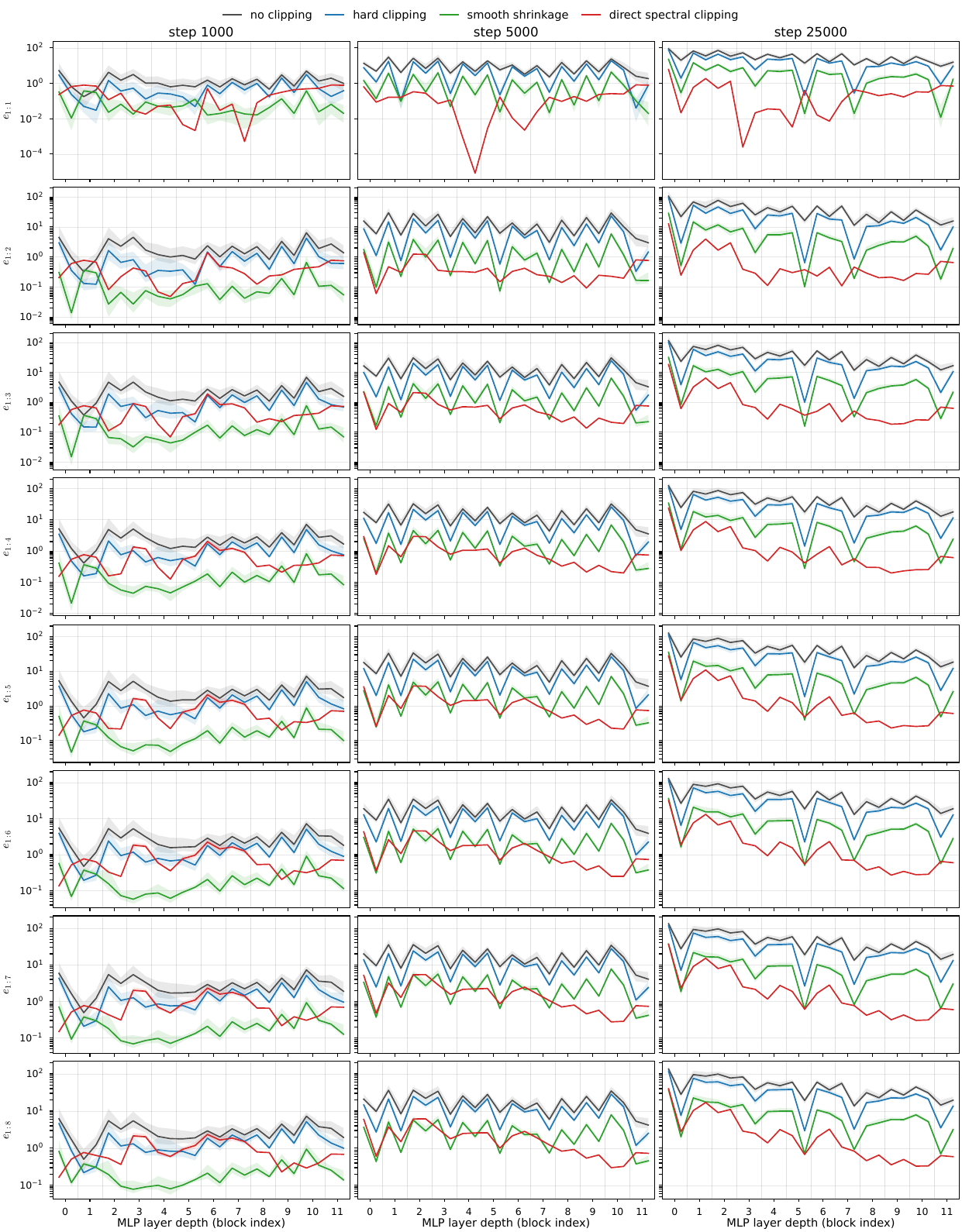}
    \caption{Spectral Control on real gradients of MLP layers of NanoGPT}
    \label{fig:spec_mlp}
\end{figure}
\begin{figure}
    \centering
    \includegraphics[width=0.95\linewidth]{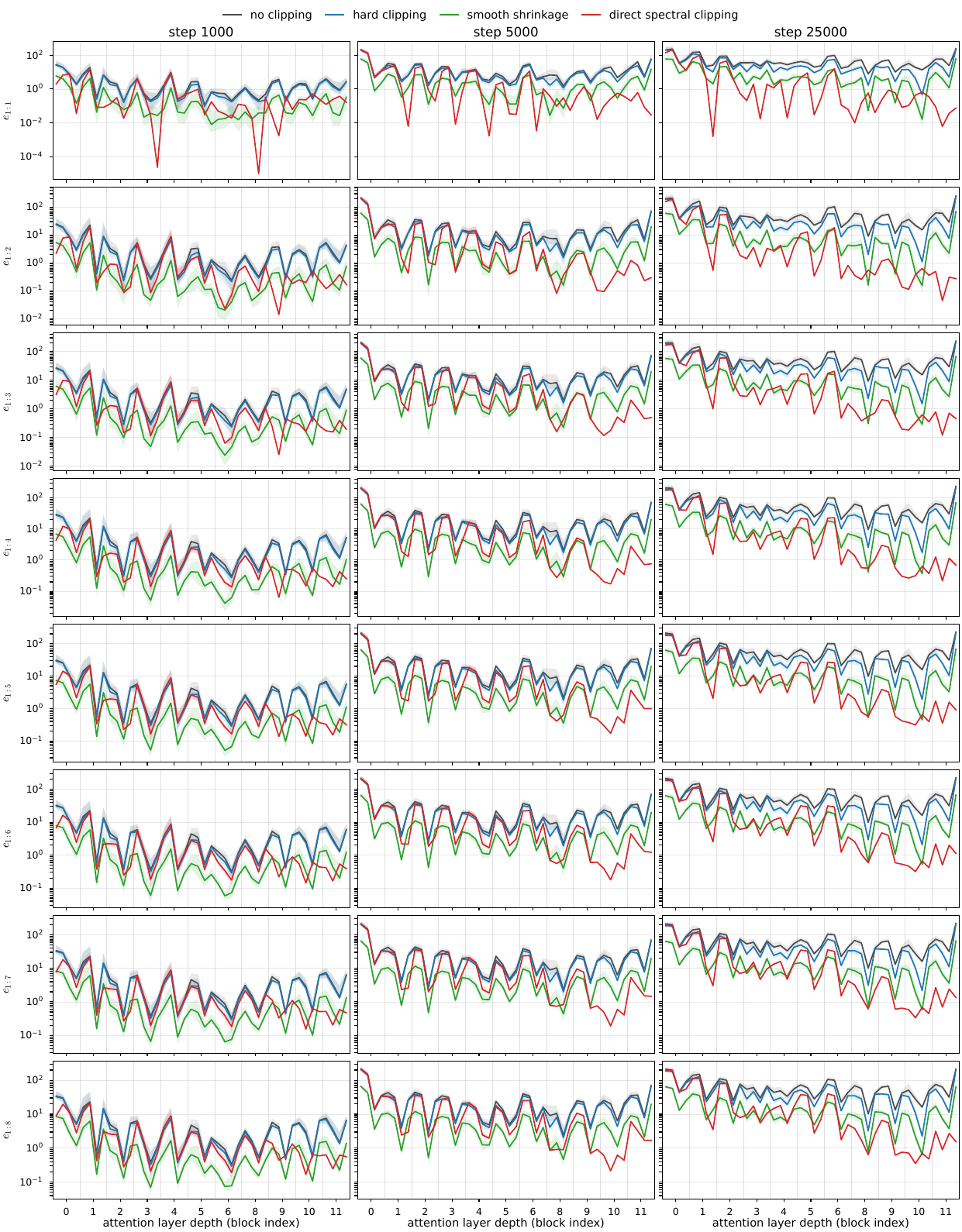}
    \caption{Spectral Control on real gradients of Attention layers of NanoGPT.}
    \label{fig:spec_att}
\end{figure}




\newpage
\clearpage
\section*{NeurIPS Paper Checklist}

%


\begin{enumerate}

\item {\bf Claims}
    \item[] Question: Do the main claims made in the abstract and introduction accurately reflect the paper's contributions and scope?
    \item[] Answer:\answerYes{} 
    \item[] Justification: The three contributions are discussed in the abstract and stated in the bullet point in the introduction \cref{sec:intro}.

\item {\bf Limitations}
    \item[] Question: Does the paper discuss the limitations of the work performed by the authors?
    \item[] Answer: \answerYes{} 
    \item[] Justification: We present the limitation discuss in \cref{app:ss_limitation}

\item {\bf Theory assumptions and proofs}
    \item[] Question: For each theoretical result, does the paper provide the full set of assumptions and a complete (and correct) proof?
    \item[] Answer: \answerYes{} 
    \item[] Justification: We present the complete proof for \cref{sec:noise_model}, \cref{sec:operator}, and \cref{sec:convergence} in \cref{app:sec3}, \cref{app:sec4_bayes}, and \cref{app:convergence-proofs} respectively.

    \item {\bf Experimental result reproducibility}
    \item[] Question: Does the paper fully disclose all the information needed to reproduce the main experimental results of the paper to the extent that it affects the main claims and/or conclusions of the paper (regardless of whether the code and data are provided or not)?
    \item[] Answer: \answerYes{} 
    \item[] Justification: We present the experimental setup in \cref{app:experiments}. We present the detailed NanoGPT hyperparameter sweep result in \cref{fig:lr_sweep} and the hyperparameter settings in \cref{tab:hyper_nanogpt_post} and \cref{tab:hyper_nanogpt_pre}.

\item {\bf Open access to data and code}
    \item[] Question: Does the paper provide open access to the data and code, with sufficient instructions to faithfully reproduce the main experimental results, as described in supplemental material?
    \item[] Answer: \answerYes{} 
    \item[] Justification: The base codebase (modded-nanogpt) and pretraining corpus (FineWeb-Edu 10B subset) are publicly available. We present the code implementation of smooth shrinkage in \cref{app:ss_limitation}.

\item {\bf Experimental setting/details}
    \item[] Question: Does the paper specify all the training and test details (e.g., data splits, hyperparameters, how they were chosen, type of optimizer) necessary to understand the results?
    \item[] Answer: \answerYes{} 
    \item[] Justification: We present the detailed experiment settings in \cref{app:experiments}. The hyperparameters are listed in \cref{tab:hyper_nanogpt_pre} and \cref{tab:hyper_nanogpt_post}.

\item {\bf Experiment statistical significance}
    \item[] Question: Does the paper report error bars suitably and correctly defined or other appropriate information about the statistical significance of the experiments?
    \item[] Answer: \answerYes{} 
    \item[] Justification: For the Gaussian random feature regression, each configuration is repeated over $10$ independent seeds (independent draws of the feature matrix, ground-truth weight, and noise matrix), and median curves are reported.

\item {\bf Experiments compute resources}
    \item[] Question: For each experiment, does the paper provide sufficient information on the computer resources (type of compute workers, memory, time of execution) needed to reproduce the experiments?
    \item[] Answer: \answerYes{} 
    \item[] Justification: We state the hardware for each experiment clearly at the beginning of \cref{app:experiments}.
    
\item {\bf Code of ethics}
    \item[] Question: Does the research conducted in the paper conform, in every respect, with the NeurIPS Code of Ethics \url{https://neurips.cc/public/EthicsGuidelines}?
    \item[] Answer: \answerYes{} 
    \item[] Justification: This work is foundational research on stochastic optimization methods. It involves no human subjects, no sensitive or personally identifiable data, and no artifacts with elevated misuse risk.

\item {\bf Broader impacts}
    \item[] Question: Does the paper discuss both potential positive societal impacts and negative societal impacts of the work performed?
    \item[] Answer: \answerYes{} 
    \item[] Justification: Improvements in optimizer efficiency contribute positively to scientific progress and to increase the token efficiency of training large language models. We present it in \cref{app:ss}.
    
\item {\bf Safeguards}
    \item[] Question: Does the paper describe safeguards that have been put in place for responsible release of data or models that have a high risk for misuse (e.g., pre-trained language models, image generators, or scraped datasets)?
    \item[] Answer: \answerNA{} 
    \item[] Justification: We do not release pretrained models, datasets, or generative artifacts. The released code is an optimizer modification that does not pose a high misuse risk on its own.

\item {\bf Licenses for existing assets}
    \item[] Question: Are the creators or original owners of assets (e.g., code, data, models), used in the paper, properly credited and are the license and terms of use explicitly mentioned and properly respected?
    \item[] Answer: \answerYes{} 
    \item[] Justification: We use the modded-nanogpt codebase \citep{modded_nanogpt_2024} (MIT License) and the FineWeb-Edu dataset \citep{penedo2024fineweb} (ODC-By 1.0). Both are cited in \cref{sec:experiments}.

\item {\bf New assets}
    \item[] Question: Are new assets introduced in the paper well documented and is the documentation provided alongside the assets?
    \item[] Answer:\answerYes{} 
    \item[] Justification: We provide the documentation of the code in \cref{app:ss_limitation}.

\item {\bf Crowdsourcing and research with human subjects}
    \item[] Question: For crowdsourcing experiments and research with human subjects, does the paper include the full text of instructions given to participants and screenshots, if applicable, as well as details about compensation (if any)? 
    \item[] Answer: \answerNA{}
    \item[] Justification: This paper does not involve crowdsourcing or research with human subjects.

\item {\bf Institutional review board (IRB) approvals or equivalent for research with human subjects}
    \item[] Question: Does the paper describe potential risks incurred by study participants, whether such risks were disclosed to the subjects, and whether Institutional Review Board (IRB) approvals (or an equivalent approval/review based on the requirements of your country or institution) were obtained?
    \item[] Answer: \answerNA{} 
    \item[] Justification: This paper does not involve human subjects and therefore did not require IRB approval.

\item {\bf Declaration of LLM usage}
    \item[] Question: Does the paper describe the usage of LLMs if it is an important, original, or non-standard component of the core methods in this research? Note that if the LLM is used only for writing, editing, or formatting purposes and does \emph{not} impact the core methodology, scientific rigor, or originality of the research, declaration is not required.
    \item[] Answer:  \answerNA{} 
    \item[] Justification: LLMs are not used as a component of the proposed methodology and are limited to language polish and theorem verfication.

\end{enumerate}

\end{document}